\documentclass[manuscript,12pt]{colt2022} 

\usepackage[utf8]{inputenc} 
\usepackage[T1]{fontenc}    
\usepackage{hyperref}       
\usepackage{url}            
\usepackage{booktabs}       
\usepackage{amsfonts}       
\usepackage{nicefrac}       
\usepackage{microtype}      

\usepackage[toc,page,header]{appendix}
\usepackage{minitoc}
\renewcommand \thepart{}
\renewcommand \partname{}

\usepackage[bf]{caption} 
\usepackage{amsmath}
\usepackage{empheq}

\usepackage{mathtools}
\usepackage{bm}
\usepackage{algorithm}
\usepackage{algorithmic}
\usepackage{wrapfig,soul}
\usepackage{enumitem}
\usepackage{xcolor}
\usepackage{grffile}
\usepackage{thm-restate}
\usepackage{natbib}
\usepackage{amsmath}

\usepackage[most]{tcolorbox}
\usepackage{pgf}
\graphicspath{{./figures/}}
\newcommand{\emphblockoption}{drop shadow,
    colframe=black!60,
    colback=black!10,
    coltitle=white!, 
    left=.2pt,
    right=.2pt,
    boxrule=1pt,
    arc=1pt}
\tcbset{emphblock/.style={code={\pgfkeysalsofrom{\emphblockoption}}}}

\newtheorem{Lemma}{Lemma}

\newtheorem{Theorem}{Theorem}
\newtheorem{Assumption}{Assumption}
\newtheorem{Remark}{Remark}

\newtheorem{Corollary}{Corollary}

\newcommand{\E}{\mathbb{E}}

\makeatletter
\newcommand\footnoteref[1]{\protected@xdef\@thefnmark{\ref{#1}}\@footnotemark}
\makeatother

\title[A Single-Timescale Analysis for Stochastic Approximation]{A Single-Timescale Analysis For Stochastic Approximation\\ With Multiple Coupled Sequences}
\usepackage{times}

\coltauthor{%
 \Name{Han Shen} \Email{shenh5@rpi.edu}\\
 \addr Rensselaer Polytechnic Institute
 \AND
 \Name{Tianyi Chen} \Email{chent18@rpi.edu}\\
 \addr Rensselaer Polytechnic Institute
}

\begin{document}

\maketitle

\begin{abstract}
Stochastic approximation (SA) with multiple coupled sequences has found broad applications in machine learning such as bilevel learning and reinforcement learning (RL). 
In this paper, we study the finite-time convergence of nonlinear SA with multiple coupled sequences. 
Different from existing multi-timescale analysis, we seek for scenarios where a fine-grained analysis can provide the tight performance guarantee for multi-sequence single-timescale SA (STSA). 
At the heart of our analysis is the smoothness property of the fixed points in multi-sequence SA that holds in many applications. When all sequences have strongly monotone increments, we establish the iteration complexity of $\mathcal{O}(\epsilon^{-1})$ to achieve $\epsilon$-accuracy, which improves the existing $\mathcal{O}(\epsilon^{-1.5})$ complexity for two coupled sequences.
When all but the main sequence have strongly monotone increments, we establish the iteration complexity of $\mathcal{O}(\epsilon^{-2})$. 
The merit of our results lies in that applying them to stochastic bilevel and compositional optimization problems, as well as RL problems leads to either relaxed assumptions or improvements over their existing performance guarantees.
\end{abstract}

\section{Introduction}
 Stochastic approximation (SA) is an iterative procedure used to find the zero of a function when only the noisy estimate of the function is observed. Specifically, with the mapping $v:\mathbb{R}^{d}\mapsto \mathbb{R}^{d}$, the single-sequence SA seeks to solve for $v(x)=0$ with the following iterative update:
 \begin{align}
    x_{k+1}=x_k + \alpha_k (v(x_k)+\xi_k),
 \end{align}
 where $\alpha_k$ is the step size and $\xi_k$ is a random variable.
 Since its introduction in \citep{robbins1951stochastic}, the single-sequence SA has received great interests because of its broad range of applications to areas including stochastic optimization and reinforcement learning (RL) \citep{bottou2018optimization,sutton2009fast}. The asymptotic convergence of single-sequence SA can be established by the ordinary differential equation method; see e.g., \citep{borkar2009stochastic}. To gain more insights into the performance difference of various stochastic optimization algorithms, the finite-time convergence of SA has been widely studied in recent years; see e.g., \citep{nemirovski2009robust,moulines2011nonasympSA,karimi2019nonasymptotic,srikant2019finitetime,mou2020linear,durmus2021stability}.

While most of the SA studies focus on the single-sequence case, the double-sequence SA was introduced in \citep{borkar1999stochastic}, which has been extensively applied to the stochastic optimization and RL methods involving a double-sequence stochastic update structure \citep{sutton2009fast, KondaAC, dalal2020tale}. With mappings $v:\mathbb{R}^{d_0}\times\mathbb{R}^{d_1}\mapsto \mathbb{R}^{d_0}$ and $h:\mathbb{R}^{d_0}\times\mathbb{R}^{d_1} \mapsto \mathbb{R}^{d_1}$, the double-sequence SA seeks to solve $v(x,y)=h(x,y)=0$ with the following update:
\begin{subequations}\label{eq:double-sequence SA}
 \begin{align}
    x_{k+1}&=x_k + \alpha_k (v(x_k,y_k)+\xi_k), \\
    y_{k+1}&=y_k + \beta_k (h(x_k,y_k)+\psi_k),
 \end{align}
 \end{subequations}
 where $\alpha_k,\beta_k$ are the step sizes, and $\xi_k,\psi_k$ are random variables. In \eqref{eq:double-sequence SA}, the update of $x_k$ and that of $y_k$ depend on each other and thus the sequences are \emph{coupled}.
Due to this coupling, the double-sequence SA is more challenging to analyze than its single-sequence counterpart. 

\textbf{Prior art on double-sequence SA.}
Many recent analyses of the double-sequence SA focus on the linear case where $v(x,y)$ and $h(x,y)$ are linear mappings; see e.g., \citep{konda2004convergence,dalal2018finite,gupta2019finitetime,kaledin2020finite}. The key idea there is to use the so-called two-time-scale (TTS) step sizes: \emph{One sequence is updated in the faster time scale while the other is updated in the slower time scale; that is $\lim_{k\rightarrow \infty}\alpha_k/\beta_k=0$.} By doing so, the two sequences are shown to decouple asymptotically, which allows us to leverage the analysis of the single-sequence SA. In particular, \citep{kaledin2020finite} proves an iteration complexity of $\mathcal{O}(\epsilon^{-1})$ to achieve $\epsilon$-accuracy for the TTS linear SA, which is shown to be tight. With similar choice of the step sizes, the TTS nonlinear SA was analyzed in \citep{mokkadem2006convergence,doan2021nonlinear}. In \citep{mokkadem2006convergence}, the finite-time convergence rate of TTS nonlinear SA was established under an assumption that the two sequences converge asymptotically. Later, \citep{doan2021nonlinear} alleviates this assumption and shows that TTS nonlinear SA achieves an iteration complexity of $\mathcal{O}(\epsilon^{-1.5})$. However, this iteration complexity is larger than $\mathcal{O}(\epsilon^{-1})$ of the TTS linear SA.

The gap between the complexities of nonlinear and linear SA motivates an interesting question:
\begin{center}
    \textbf{Q1: Is it possible to prove a faster rate for the nonlinear SA with two coupled sequences?}
\end{center}

 \begin{wrapfigure}{R}{0.4\textwidth}
     \vspace*{-.6cm}
    \hspace{-.1cm}
    \includegraphics[width=.98\linewidth]{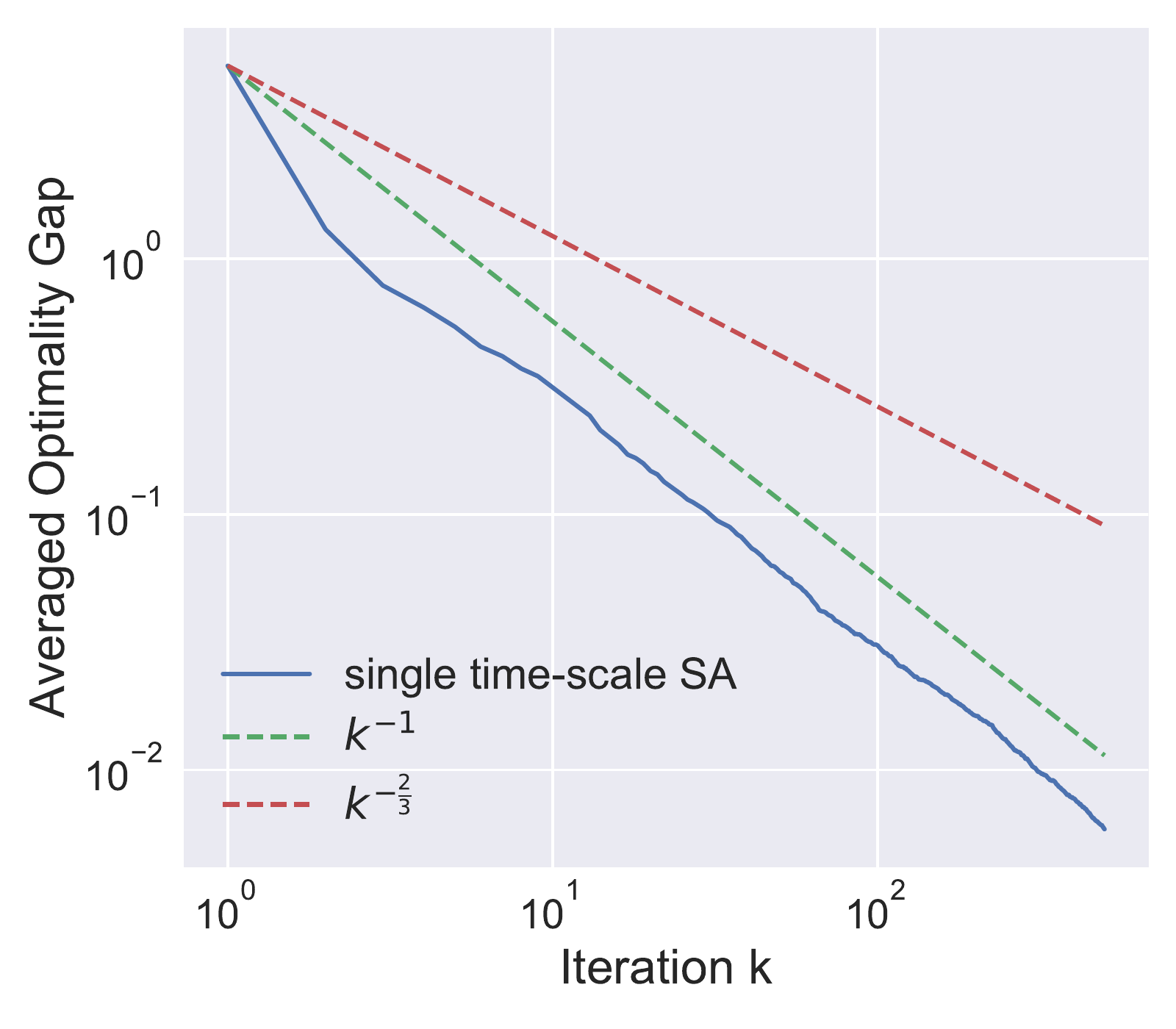}
        \vspace{-.4cm}
   \captionsetup{format=plain}    
    \caption{Solving \eqref{eq:toy problem} via double-sequence nonlinear SA \eqref{eq:double-sequence SA}. The single time-scale nonlinear SA converges with rate $\mathcal{O}(k^{-1})$, which is faster than the theoretical $\mathcal{O}(k^{-\frac{2}{3}})$ rate in \citep{doan2021nonlinear}.}
    \label{fig:toy example}
    \vspace*{-1.6cm}
\end{wrapfigure}

\noindent\textbf{Experiment.}
We first conduct an experiment to examine the possibility. Figure \ref{fig:toy example} shows the performance of using the double-sequence SA \eqref{eq:double-sequence SA} to solve the following problem
\begin{align}\label{eq:toy problem}
   \max_{x\in \mathbb{R}} &~~~-\frac{1}{2}\Big(x^2+\frac{1}{1+e^{-y^*(x)}}\Big)\nonumber\\
    &{\rm s.t.}~~y^*(x) = \arg \min_{y \in \mathbb{R}} \frac{1}{2}(y-x)^2. 
\end{align}

We use the double-sequence SA  \eqref{eq:double-sequence SA} to solve \eqref{eq:toy problem}, where
\begin{align}\label{eq:toy update}
    v(x,y) = -x -  \frac{e^{-y}}{(1+e^{-y})^2},~~~h(x,y) = x-y
\end{align}
and $\zeta_k$, $\xi_k$ are independent Gaussian random variables with zero mean and standard deviations of $0.15$. It is easy to check that \eqref{eq:toy update} satisfies the assumptions in the existing TTS-SA analysis \citep{doan2021nonlinear}. Therefore, we can use the two time-scale step sizes and achieve the iteration complexity of $\mathcal{O}(\epsilon^{-1.5})$. However, as suggested by Figure \ref{fig:toy example}, the iterates still converge with step sizes in a single time-scale ($\alpha_k=\Theta(\frac{1}{k})$, $\beta_k=\Theta(\frac{1}{k})$). In this case, the iteration complexity is $\mathcal{O}(\epsilon^{-1})$, which is the same as that of double-sequence linear SA \citep{kaledin2020finite}. This suggests that existing analysis of double-sequence SA might not be tight, at least for the class of updates similar to \eqref{eq:toy update}. Indeed, as we will show later, the iterates generated by \eqref{eq:toy update} will converge with the iteration complexity of $\mathcal{O}(\epsilon^{-1})$.

\begin{table}[t]
 \small
 \centering
	\begin{tabular}{c||c|c||c | c|c| c||c|c|c}
		\hline
		\hline
		&\multicolumn{2}{|c||}{General result} & \multicolumn{4}{|c||}{Application to SBO} & \multicolumn{3}{|c}{Application to multi-level  SCO}\\
		\hline
 	   &\!\!\!\!{Ours}\!\!\!\!  & \!\!\!\!{TTS SA} \!\!\!\!&\!\!\!\! {Ours} \!\!\!\!&\!\!\!\!\!\! {TTSA}\!\!\!\!\!\! &\!\!\!\!\!\!\! {ALSET}\!\!\!\! &\!\!\!\! {ALSET-AC} \!\!\!\!&\!\!\!\! {Ours} \!\!\!\!&\!\!\!\! {$\alpha$-TSCGD} \!\!\!\!&\!\!\!\! {SG-MRL}\!\! \!\!\\
 	   \hline
 	   \!\!\!\!{SM}\!\!\!\! & \!\!\!\!{$\mathcal{O}(\epsilon^{-1})$}\!\!\!\!&\!\!\!\!\! {$\mathcal{O}(\epsilon^{-1.5})$} \!\!\!\!\!&\!\!\!\! {$\tilde{\mathcal{O}}(\epsilon^{-1})$} & \!\!\!{$\tilde{\mathcal{O}}(\epsilon^{-1.5})$} \!\!\!\!&\!\!\!\! {\char`\~} \!\!\!\!&\!\!\!\! {\char`\~} \!\!\!\!&\!\!\!\! {$\mathcal{O}(\epsilon^{-1})$} \!\!\!\!&\!\!\!\! {$\mathcal{O}(\epsilon^{-\frac{N+5}{4}})$} \!\!\!\!&\!\!\!\! {\char`\~}\\
 	   \hline
 	   \!\!\!\!{N-SM } \!\!\!\!& {$\mathcal{O}(\epsilon^{-2})$} & {\char`\~} &\!\!\!  {$\tilde{\mathcal{O}}(\epsilon^{-2})$} \!\!\! &\!\!\!  {$\tilde{\mathcal{O}}(\epsilon^{-2.5})$} \!\!\! &\!\!\!  {$\tilde{\mathcal{O}}(\epsilon^{-2})$}\!\!\!  & \!\!\! {$\mathcal{O}(\epsilon^{-2})$}\!\!\!  &\!\! {$\mathcal{O}(\epsilon^{-2})$} & \!\!{$\mathcal{O}(\epsilon^{-\frac{N+8}{4}})$} &\!\! {$\mathcal{O}(\epsilon^{-4})$}\\
 	   \hline
 	   \!\!\!\!{Merit} \!\!\!\!&\!\!\!\! {\char`\~} \!\!\!\!&\!\!\!\! {Rate $\uparrow$} \!\!\!\!&\!\!\!\! {\char`\~} \!\!\!\!&\!\!\!\! {Rate $\uparrow$} \!\!\!\!&\!\!\!\! {Relax} \!\!\!\!&\!\!\!\! { Relax} \!\!\!\!&\!\!\!\! {\char`\~} \!\!\!\!&\!\!\!\! {Rate $\uparrow$} & {Rate $\uparrow$}\\
 	   \hline
 	   \hline
		\end{tabular}
		\vspace{5pt}
  \caption{Comparisons with TTS SA \citep{doan2021nonlinear}, TTSA \citep{hong2020ac}, ALSET and ALSET-AC \citep{chen2021tighter}, $\alpha$-TSCGD \citep{yang2018multilevel} and SG-MRL \citep{fallah2021convergence}. Strongly-monotone (SM) and non-strongly-monotone (N-SM) respectively represents the case where the main sequence has strongly-monotone and non-strongly-monotone increments. Rows of SM/N-SM are for the complexity and the row of Merit is for the improvements of this work over the existing work (``Rate $\uparrow$'' stands for faster rate; ``Relax'' for relaxed assumptions).}		\label{table:comp1}
  \vspace{-0.2cm}
\end{table}

Furthermore, existing works on TTS SA mainly focus on the double-sequence case. While in cases such as the multi-level stochastic optimization; see e.g., \citep{yang2018multilevel,sato2021gradient}, more than two sequences are involved. This necessitates the use of the multi-sequence SA. 
Specifically, with mappings $v:\mathbb{R}^{d_0}\times \mathbb{R}^{d_1} \dots \times \mathbb{R}^{d_N} \mapsto \mathbb{R}^{d_0}$, $h^n:\mathbb{R}^{d_{n-1}}\times \mathbb{R}^{d_n} \mapsto \mathbb{R}^{d_n}$, we consider the following multi-sequence Single-Timescale SA (STSA) update in this work:
\begin{tcolorbox}[emphblock]
\vspace{-0.4cm}
\begin{subequations}\label{eq:multiSA}
\begin{align}
\label{eq:updatey2}
\hspace{-2cm}{\rm\bf (STSA)}~~~~~~~~~~~~    y_{k+1}^n &= y_k^n + \beta_{k,n}\big(h^n(y_k^{n-1},y_k^n) + \psi_k^n\big), ~~ n =1,2,...,N\\
    \label{eq:updatex2}
    x_{k+1} &= x_k + \alpha_k \big(v(x_k,y_k^1,y_k^2,\dots,y_k^N)+\xi_k \big)
    \end{align}
\end{subequations}
\end{tcolorbox}
\hspace{-0.7cm}
where $\alpha_k, \beta_{k,1},...,\beta_{k,N}$ are the step sizes, and $\xi_k, \psi_k^1, \dots, \psi_k^N$ are random variables. For conciseness, we have used $y_k^0\coloneqq x_k$ here. 
We call the update \eqref{eq:multiSA} a \emph{single-timescale} update since we will show the two sequences are updated in the same time scale; that is $\lim_{k\rightarrow \infty}\alpha_k/\beta_k=c$ for some constant $c>0$. 
Our goal is to find the unique fixed-points $x^*,y^{1,*},\dots,y^{N,*}$ such that
\begin{gather}\label{eq:goal}
    v(x^*,y^{1,*},\dots,y^{N,*})=h^1(x^*,y^{1,*})=\dots=h^N(y^{N-1,*},y^{N,*})=0.
\end{gather}
Observing that in \eqref{eq:multiSA}, for every $n$, the sequence of $y_k^n$ is coupled with that of $y_k^{n-1}$ and is ultimately coupled with the main sequence $x_k$. Meanwhile, the update of $x_k$ also depends on $\{y_k^n\}_{n=1}^N$. Since all sequences in \eqref{eq:multiSA} are coupled, \eqref{eq:multiSA} is more challenging to analyze than the double-sequence SA.

\textbf{Prior art related to multi-sequence SA.} The multilevel optimization problem \citep{sato2021gradient} and the multilevel SCO problem \citep{balasubramanian2021stochastic,xiao2022projection,zhang2021multilevel,ruszczynski2021stochastic,zhang2020optimal} are closely related to the multi-sequence SA. To tackle the multi-level structure, these recent advances have modified the vanilla multi-sequence SA update to achieve the state-of-the-art complexity and thus their updates are not exactly in the form of \eqref{eq:multiSA}. In particular, the recent work \cite{sato2021gradient} mainly focuses on the deterministic multilevel optimization, where it approximates the desired increment of $x$ by a nested algorithm, and thus its update is not in the form of \eqref{eq:multiSA}. 
In contrast, we focus on the multi-sequence STSA update in \eqref{eq:multiSA}. 
To the best of our knowledge, the only analysis for the exact update \eqref{eq:multiSA} is \citep{yang2018multilevel}, where the TTS technique is generalized to multi-time-scale. 
However, in \citep{yang2018multilevel}, the iteration complexity will get worse as the number of sequences $N$ increases.

This gives rise to another interesting question: 
\begin{center}
    \textbf{Q2: Is it possible to establish convergence rate of the update \eqref{eq:multiSA} that is independent of $N$?}
\end{center}
In this work, we give affirmative answers to both Questions \textbf{Q1} and \textbf{Q2}. 

\textbf{Our contributions.}
By exploiting the \emph{smooth assumption} of the fixed points that can be satisfied in many applications, we show that the vanilla nonlinear STSA can run in a \emph{single time scale}! We further prove that the order of the convergence rate is \emph{independent of} the number of sequences $N$.  
Intuitively, this is possible because when the fixed point $y^{n,*}$ is smooth in $x$, the $y_k^n$-update \emph{converges fast enough} such that its fixed-point residual after one update is at the same order as the drift of $y^{n,*}$.

In the context of prior art, our  contributions can be summarized as follows (see Table \ref{table:comp1}).

\textbf{C1) Single-timescale analysis for multi-sequence SA.} 
Different from existing two-timescale analysis \citep{mokkadem2006convergence,borkar1997actor}, we establish a unifying \textbf{S}ingle-\textbf{T}imescale analysis for \textbf{SA} with multiple coupled sequences that we term \textbf{STSA}. 
When all the sequences have strongly-monotone increments, we improve the $\mathcal{O}(\epsilon^{-1.5})$ iteration complexity for multi-sequence TTS-SA in \citep{doan2021nonlinear} to $\mathcal{O}(\epsilon^{-1})$. When all but the main sequence have strongly monotone increments, we provide the $\mathcal{O}(\epsilon^{-2})$ iteration complexity.

\textbf{C2) STSA for stochastic bilevel optimization (SBO).} When applying our generic results to the SBO problem with double-sequence SA, for  strongly-concave objective functions, we improve the best-known sample complexity $\tilde{\mathcal{O}}(\epsilon^{-1.5})$  of TTSA in \citep{hong2020ac} to $\tilde{\mathcal{O}}(\epsilon^{-1})$. For the non-concave objective function, we achieve the same sample complexity $\mathcal{O}(\epsilon^{-2})$ of ALSET while relaxing the bounded upper-level gradient assumption made in \citep{chen2021tighter}.

\textbf{C3) STSA for stochastic compositional optimization (SCO).} When applying our results to the multi-level SCO problems, we improve the level-dependent sample complexities $\mathcal{O}(\epsilon^{-\frac{N+5}{4}})$ and $\mathcal{O}(\epsilon^{-\frac{N+8}{4}})$ of multi-sequence SA based $\alpha$-TSCGD method in \citep{yang2018multilevel} to the level-independent iteration complexities $\tilde{\mathcal{O}}(\epsilon^{-1})$ and $\mathcal{O}(\epsilon^{-2})$, under the strongly-concave and non-concave objective functions, respectively.

\textbf{C4) STSA for policy optimization in RL problems.} Moreover, applying our results to the actor-critic method achieves the same $\mathcal{O}(\epsilon^{-2})$ sample complexity of ALSET-AC in \citep{chen2021tighter} while relaxing the unverifiable assumption on the stationary distribution of Markov chains; applying our results to the meta policy gradient  improves the $\mathcal{O}(\epsilon^{-4})$ sample complexity of SG-MRL in \citep{fallah2021convergence} to $\mathcal{O}(\epsilon^{-2})$. 


\textbf{Roadmap.} We organize the results as follows. With the multi-sequence SA update introduced in this section, we first present our main results in Section \ref{section:result}. In Section \ref{section:bilevel}, we apply our results to the SBO problem which also encompasses the actor-critic method. In Section \ref{section:sco}, we apply our result to the SCO problem which also encompasses the model-agnostic meta policy gradient method.

\section{Main Results: Convergence of Single-timescale Multi-sequence SA}\label{section:result}
Before introducing the main results, we will first make some standard assumptions.
Throughout the discussion, we define $[N]\coloneqq \{1,2,...,N\}$, $[K]\coloneqq \{1,2,...,K\}$ and $y^0 \coloneqq x$ for conciseness.
\begin{Assumption}[Smoothness of the fixed points]\label{assumption:y*}
For any $n \in [N]$ and $y^{n-1} \in \mathbb{R}^{d_{n-1}}$, there exists a unique $y^{n,*}(y^{n-1}) \in \mathbb{R}^{d_n}$ such that $h^n(y^{n-1},y^{n,*}(y^{n-1}))=0$. Moreover, there exist constants $L_{y,n}$ and $L_{y',n}$ such that for any $y^{n-1},\Bar{y}^{n-1}\in \mathbb{R}^{d_{n-1}}$, the following inequalities hold
\begin{subequations}
\begin{align}\label{eq:y*lip}
    \|y^{n,*}(y^{n-1}) - y^{n,*}(\Bar{y}^{n-1})\| &\leq L_{y,n} \|y^{n-1}-\Bar{y}^{n-1}\|,  \\
    \label{eq:y*smooth}
    \|\nabla y^{n,*}(y^{n-1}) - \nabla y^{n,*}(\Bar{y}^{n-1})\| &\leq L_{y',n} \|y^{n-1}-\Bar{y}^{n-1}\|.
\end{align}
\end{subequations}
\end{Assumption}
Due to the change of $y_k^{n-1}$ at each iteration, the solution of $h^n(y_k^{n-1},y^n)=0$ with respect to (w.r.t.) $y^n$, that is $y^{n,*}(y_k^{n-1})$, is drifting over consecutive iterations. Given $y_k^{n-1}$, since only one-step of $y_k^n$ update is performed at each iteration, one can only hope to establish convergence of $y_k^n$ if the drift of its optimal solution is controlled in some sense. Assumption \ref{assumption:y*} ensures both the zeroth-order and first-order drifts are controlled in the same scale of the change of $y_k^{n-1}$. This assumption is satisfied in linear SA \citep{kaledin2020finite} and other applications which will be shown later.

Define $v(x) \coloneqq v\big(x,y^{1,*}(x), y^{2,*}(y^{1,*}(x)),\dots,y^{N,*}(\dots y^{2,*}(y^{1,*}(x))\dots)\big)$. With $y^{1:N}$ as a concise notation for $(y^1,...,y^N )$, we make the following assumption.
\begin{Assumption}[Lipschitz continuity of increments]\label{assumption:vglip}
For any $n \in [N]$, $x,\Bar{x} \in \mathbb{R}^{d_0}$ and $y^n, \Bar{y}^n \in \mathbb{R}^{d_n}$, there exist constants $L_v$, $L_{v,y}$ and $L_{h,n}$ such that the following inequalities hold
\begin{subequations}
\begin{align}
\label{eq:v lip}    \|v(x) - v(\Bar{x})\|&\!\leq\! L_v \|x-\Bar{x}\|,\\
    \|v(x,y^{1:N})\!-\!v(x,\Bar{y}^{1:N})\| &\!\leq\! L_{v,y} \!\sum_{n=1}^N\!\|y^n \!-\!\Bar{y}^n\|,\! \\
    \|h^n(y^{n-1},y^n)\!-\!h^n(y^{n-1},\Bar{y}^n)\| &\!\leq\! L_{h,n} \|y^n -\Bar{y}^n\|.
\end{align}
\end{subequations}
\end{Assumption}

Define $\mathcal{F}_k$ as the $\sigma$-algebra generated by the random variables in $\{x_i,y_i^{1:N}\}_{i=1}^{k}$ and $\mathcal{F}_k^n$ as the $\sigma$-algebra generated by $\{x_i,y_i^{1:N}\}_{i=1}^{k}\cup\{y_{k+1}^n\}$. We make the following assumption on the noises.
\begin{Assumption}[Bias and variance]\label{assumption:noise}
There exist constants $\{c_n,\sigma_n\}_{n=0}^N$ such that $\forall k, n$, $\|\E[\xi_k | \mathcal{F}_k^1]\|^2 \leq c_0^2\alpha_k$ and $\|\E[\psi_k^n | \mathcal{F}_k^{n+1}]\|^2\leq c_n^2\beta_{k,n}$; $\E[\|\xi_k\|^2 | \mathcal{F}_k^1] \leq \sigma_0^2$ and $\E[\|\psi_k^n\| | \mathcal{F}_k^{n+1}] \leq \sigma_n^2$.
\end{Assumption}
Here we define $\mathcal{F}_k^{N+1}\coloneqq \mathcal{F}_k$. Assumption \ref{assumption:noise} is a generalized version of the standard zero-mean noise with bounded moments assumption \citep{konda2004convergence}. Given $\mathcal{F}_k$, if $\psi_k^n$ and $\xi_k$ are independent of each other and have zero-mean, Assumption \ref{assumption:noise} holds with $c_n=0$ for any $n$.
\begin{Assumption}[Monotonicity of $h$]\label{assumption:strongly g}
For $ n \in [N]$, $h^n(y^{n-1},y^n)$ is one-point strongly monotone on $y^{n,*}(y^{n-1})$ given any $y^{n-1}$; that is, there exists constant $\lambda_n>0$ such that (cf. {\small$h^n(y^{n-1}, y^{n,*})=0$})
\begin{equation}
    \big\langle y^n\!-\!y^{n,*}(y^{n-1}),h^n(y^{n-1},y^n) \big\rangle \leq -\lambda_n \|y^n\!-\!y^{n,*}(y^{n-1})\|^2.
\end{equation}
\end{Assumption}
Assumption \ref{assumption:strongly g} is implied by the standard regularity assumptions in the previous works on TTS linear SA \citep{konda2004convergence,kaledin2020finite}, and has also been exploited in the TTS nonlinear SA works; see e.g., \citep{mokkadem2006convergence,doan2021nonlinear}.

\subsection{The strongly-monotone case}
We first consider the case when the main sequence $x_k$ has strongly-monotone increment.
\begin{Assumption}[Monotonicity of $v$]\label{assumption:strongly v}
Suppose $v(x)$ is one-point strongly monotone on $x^*$; that is, there exists a positive constant $\lambda_0$ such that (cf. $v(x^*)=0$)
\begin{equation}
    \big\langle x-x^*,v(x) \big\rangle \leq -\lambda_0 \|x-x^*\|^2.
\end{equation}
\end{Assumption}
Same as Assumption \ref{assumption:strongly g}, Assumption \ref{assumption:strongly v} is standard in the previous works on TTS SA \citep{mokkadem2006convergence,doan2021nonlinear}. This assumption is a regularity assumption in the case of TTS linear SA; see e.g., \cite[Assumption 2.3]{konda2004convergence}. Or in the case of bilevel optimization which will be discussed later, this assumption is satisfied when the objective function is strongly-concave.

Due to space limitation, we directly present the  result below and defer the proof to Appendix \ref{section:strongly monotone appendix}.
\begin{tcolorbox}[emphblock]
\vspace{-0.1cm}
\begin{Theorem}\label{theorem:strongly monotone}
Consider the STSA sequences generated by  \eqref{eq:multiSA} for $k=[K]$. Suppose \textbf{Assumptions \ref{assumption:y*}--\ref{assumption:strongly v}} hold. Selecting step sizes $\alpha_k = \Theta(\frac{1}{k})$ and $\beta_{k,n}=\Theta(\frac{1}{k})$, then it holds for any $k$ that
\begin{equation}
    \E\|x_k-x^*\|^2 + \sum_{n=1}^N\E\|y_k^n-y^{n,*}(y_k^{n-1})\|^2 = \mathcal{O}\Big(\frac{1}{k} \Big)
\end{equation}
where $\mathcal{O}(\cdot)$ hides constants in the polynomial of $N$. 
Moreover, for any $n \in [N]$, we have
\begin{equation}
   \lim_{k\xrightarrow[]{}\infty}\|x_k- x^*\|^2=0~~~\text{almost surely (a.s.)},~~\lim_{k\xrightarrow[]{}\infty}\|y_k^n- y^{n,*}(y_k^{n-1})\|^2 = 0~~~\text{a.s.}
\end{equation}
\end{Theorem}
\end{tcolorbox}
 It is worth noting that with \eqref{eq:y*lip}, Theorem \ref{theorem:strongly monotone} also implies the same convergence result for the error metric $\|x_k-x^*\|^2+\sum_{n=1}^N\|y_k^n-y^{n,*}\|^2$, the formal justification of which is deferred to the proof of Theorem \ref{theorem:strongly monotone}. In addition, the order of convergence rate in Theorem \ref{theorem:strongly monotone} is independent of $N$, which is in contrast to the convergence rate that gets worse as $N$ increases \citep{doan2021nonlinear,yang2018multilevel}.

\begin{Remark}[Comparison with prior art in multi-sequence SA]
\emph{Theorem \ref{theorem:strongly monotone} fills the gap between the convergence rate of double-sequence linear SA and that of double-sequence nonlinear SA by improving over the $\mathcal{O}(k^{-\frac{2}{3}})$ rate shown in \citep{doan2021nonlinear}. While this is done under the additional assumption \eqref{eq:y*smooth}, as will be shown later, this assumption is satisfied in various applications.
Theorem \ref{theorem:strongly monotone} also generalizes the $\mathcal{O}(\frac{1}{k})$ convergence rate in the double-sequence linear SA analysis (e.g., \citep{kaledin2020finite}) to the multi-sequence nonlinear SA case.}
\end{Remark}

\subsection{The non-strongly-monotone case}
Some applications of multi-sequence nonlinear SA such as the actor-critic method \citep{KondaAC}, Assumption \ref{assumption:strongly v} does not hold. This motivates us to consider a more general setting in this subsection where $v(x)$ is non-strongly-monotone.

Throughout this subsection, we make the following assumption.
\vspace{-0.15cm}
\begin{Assumption}\label{assumption:integral}
Suppose there exists a mapping $F:\mathbb{R}^{d_0}\mapsto \mathbb{R}$ such that $\nabla F(x)=v(x)$. The sequence of $\{x_k\}$ is contained in an open set over which $F(x)$ is upper bounded; e.g., $F(x) \leq C_F$.
\end{Assumption}
Assumption \ref{assumption:integral} is standard in SA; see e.g., \citep{karimi2019nonasymptotic}. 
As will be shown later, $F(x)$ can be chosen as the objective function when applying SA to maximization problems. 

The following theorem gives the general finite-time convergence result of the nonlinear SA when the main sequence has the non-strongly-monotone increment. The proof is deferred to Appendix \ref{section:non strongly monotone appendix}.
\begin{tcolorbox}[emphblock]
\vspace{-0.1cm}
\begin{Theorem}\label{theorem:non-strongly-monotone-2}
    Consider the STSA sequences generated by \eqref{eq:multiSA} for $k=[K]$. Suppose \textbf{Assumptions \ref{assumption:y*}--\ref{assumption:strongly g} and \ref{assumption:integral}} hold. Selecting $\alpha_k = \Theta(\frac{1}{\sqrt{K }})$, $\beta_{k,n}=\Theta(\frac{1}{\sqrt{K}})$, then it holds that
\begin{align}\label{eq:theoremresult3}
    &\frac{1}{K}\sum_{k=1}^K  \Big(\E\|\nabla F(x_k)\|^2 + \sum_{n=1}^N \E\|y_k^n - y^{n,*}(y_k^{n-1})\|^2\Big) = \mathcal{O}\Big(\frac{1}{\sqrt{K}} \Big)
\end{align}
where $\mathcal{O}(\cdot)$ hides problem dependent constants of a polynomial of $N$, and we have used $y_k^0=x_k$.
\end{Theorem}
\end{tcolorbox}
Theorem \ref{theorem:non-strongly-monotone-2} implies a finite-time convergence rate of $\mathcal{O}(K^{-2})$, which is independent of the number of sequences $N$.
The error metric $\|\nabla F(x_k)\|$ used in Theorem \ref{theorem:non-strongly-monotone-2} is of interest since it is a general measure of the convergence of $x_k$ widely adopted in many applications of SA, 
especially when the increment of $x_k$ is not strongly-monotone.
Moreover, although we have assumed the existence and uniqueness of $x^*$ in \eqref{eq:goal}, the proof of Theorem \ref{theorem:non-strongly-monotone-2} does not utilize this fact and thus the theorem applies to the more general case where $x^*$ is not unique or even does not exist.

Given the generic results in Theorems \ref{theorem:strongly monotone} and \ref{theorem:non-strongly-monotone-2}, next we will showcase that many existing SA-type algorithms for stochastic optimization and RL problems indeed satisfy either \textbf{Assumptions \ref{assumption:y*}--\ref{assumption:strongly v}} or \textbf{Assumptions \ref{assumption:y*}--\ref{assumption:strongly g} and \ref{assumption:integral}}, and therefore, their convergence results can be automatically deduced from the generic STSA results in Theorems \ref{theorem:strongly monotone} and \ref{theorem:non-strongly-monotone-2}.

\section{Application to the stochastic bilevel optimization method}\label{section:bilevel}
In this section, we apply our generic STSA results to the stochastic bilevel optimization problems and its applications to the actor-critic algorithm in RL.

With mappings $f:\mathbb{R}^{d_0}\times \mathbb{R}^{d_1}\mapsto \mathbb{R}$ and $g:\mathbb{R}^{d_0}\times \mathbb{R}^{d_1}\mapsto \mathbb{R}$, consider the following formulation of the bilevel optimization problem:
\begin{subequations}\label{eq:bilevel}
\begin{align}
    \max_{x \in \mathbb{R}^{d_0}} F(x) &\coloneqq f(x,y^*(x)) \coloneqq \E_{\zeta}\big[f(x,y^*(x);\zeta)\big]\\
    {\rm s.t.}~y^*(x) &\coloneqq \arg\min_{y \in \mathbb{R}^{d_1}} g(x,y) \coloneqq \E_{\varphi}\big[g(x,y;\varphi)\big]
\end{align}
\end{subequations}
where $\zeta$ and $\varphi$ are two random variables.

\subsection{Reduction from the generic results}
A popular approach to solving \eqref{eq:bilevel} is the gradient-based method \citep{ghadimi2018approximation,hong2020ac,chen2021tighter}. 
Under some conditions that will be specified later, it has been shown in \citep{ghadimi2018approximation} that the gradient of $F(x)$ takes the following form:
\begin{align}\label{eq:nabla F(x)}
    \nabla F(x) = \nabla_x f(x,y^*(x)) 
    - \nabla_{xy}^2 g(x,y^*(x)) [\nabla_{yy}^2 g(x,y^*(x))]^{-1} \nabla_y f(x,y^*(x)).
\end{align}
Note that here $\nabla_x f(x,y^*(x)) = \frac{\partial f(x,y)}{\partial x}|_{y=y^*(x)}$, $\nabla_{xy}^2 g(x,y^*(x)) = \frac{\partial g(x,y)}{\partial x \partial y}|_{y=y^*(x)}$ and the same thing applies to $\nabla_y f(x,y^*(x))$ and $\nabla_{yy}^2 g(x,y^*(x))$. It is clear that computing \eqref{eq:nabla F(x)} requires $y^*(x)$, which is often unknown in practice. Instead,  one can iteratively update $y_k$ to approach $y^*(x_k)$ while using $y_k$ in place of $y^*(x_k)$ during the computation of \eqref{eq:nabla F(x)} \citep{hong2020ac,chen2021tighter}. This leads to an update same as that in \eqref{eq:multiSA} with $N\!=\!1$, where the generic mapping is defined as
\begin{subequations}\label{eq:bilevel update}
\begin{align}
    h(x,y) &= -\nabla_y g(x,y), \\
    \psi_k &\!=\! -h(x_k,y_k) \!-\! \nabla_y g(x_k,y_k;\varphi_k), \\
    v(x,y) &= \nabla_x f(x,y) -\nabla_{xy}^2 g(x,y)[\nabla_{yy}g(x,y)]^{-1} \nabla_y f(x,y), \\
    \xi_k &= -v(x_k,y_k) +  \nabla_x f(x_k,y_k;\zeta_k) -\nabla_{xy}^2 g(x_k,y_k;\varphi_k') H^{yy}_{k} \nabla_y f(x_k,y_k;\zeta_k).
\end{align}
\end{subequations}
Since here we only have two sequences, that is $N\!=\!1$, we omit the index $n$ to simplify notations. In \eqref{eq:bilevel update}, $\zeta_k$ is a random variable with the same distribution as that of $\zeta$, and $\varphi_k$, $\varphi_k'$ have the same distribution as that of $\varphi$. Here $H_k^{yy}$ is a stochastic approximation of the Hessian inverse $[\nabla_{yy}g(x_k,y_k)]^{-1}$. Given $x_k$, when $y_k$ reaches the optimal solution $y^*(x_k)$, it follows from \eqref{eq:nabla F(x)} that $v(x_k,y^*(x_k)) = \nabla F(x_k)$.

As being discussed below Assumption \ref{assumption:y*}, the lower-level optimal solution $y^*(x_k)$ is drifting at each iteration. Under the Lipschitz continuity assumption of $y^*(x)$, the drifting $\|y^*(x_{k+1})-y^*(x_k)\|$ scales with $\|x_{k+1}-x_k\|$ which ultimately scales with $\|\nabla F(x_k)\|$. To control the drift scale, former analysis heavily relies on the condition that $\|\nabla F(x_k)\|$ is bounded for any $k$. In SBO, this means to either make a strong assumption on the Lipschitz continuity of $f(x,y)$ w.r.t. $(x,y)$, which leads to the Lipschitz continuity of $F(x)$ and the boundedness of $\|\nabla F(x_k)\|$ \citep{chen2021tighter}; or to introduce projection in \eqref{eq:bilevel update} to confine $x_k$ in a compact set \citep{hong2020ac}, all of which greatly narrow the range of application. We will show that neither of these the conditions is not needed by applying our generic results to SBO.
\begin{Lemma}[Verifying assumptions of STSA]\label{lemma:bilevel}
Consider the following conditions
\begin{enumerate}[label=(\alph*)]
\setlength\itemsep{-0.15em}
    \item\label{item:strong convex} For any $x \in \mathbb{R}^{d_1}$, $g(x,y)$ is strongly convex w.r.t. $y$ with modulus $\lambda_1 > 0$.
    \item\label{item:g lip 2} There exist constants $L_{xy}, l_{xy}, l_{yy}$ such that $\nabla_y g(x,y)$ is $L_{xy}$-Lipschitz continuous w.r.t. x; $\nabla_y g(x,y)$ is $L_h$-Lipschitz continuous w.r.t. $y$. $\nabla_{xy}g(x,y)$, $\nabla_{yy} g(x,y)$ are respectively $l_{xy}$-Lipschitz and $l_{yy}$-Lipschitz continuous w.r.t. $(x,y)$.
    \item\label{item:f lip} There exist constants $l_{fx}, l_{fy}, l_{fy}', l_y$ such that $\nabla_x f(x,y)$ and $\nabla_y f(x,y)$ are respectively $l_{fx}$ and $l_{fy}$ Lipschitz continuous w.r.t. $y$; $\nabla_y f(x,y)$ is $l_{fy}'$-Lipschitz continuous w.r.t. $x$; $f(x,y)$ is $l_y$-Lipschitz continuous w.r.t. $y$.
     \item\label{item:F RSI} $F(x)$ satisfies the restricted secant inequality: There exists a constant $\lambda_0 > 0$ such that $\langle \nabla F(x), x-x^*\rangle \leq -\lambda_0 \|x-x^*\|^2$, where $x^*\coloneqq \arg\max_{x\in\mathbb{R}^{d_1}}F(x)$.
     \item\label{item:noise bilevel} For any $k$, there exist constants $c_0,c_1$ such that $\|\E[\xi_k | \mathcal{F}_k^1]\|^2 \leq c_0^2\alpha_k$ and $\|\E[\psi_k | \mathcal{F}_k]\|^2 \leq c_1^2\beta_k$. Additionally, there exist constants $\sigma_0,\sigma_1$ such that $\E[\|\xi_k\|^2 | \mathcal{F}_k^1] \leq \sigma_0^2$ and $\E[\|\psi_k\|^2 | \mathcal{F}_k] \leq \sigma_1^2$.
     \item\label{item:F} There exists a constant $C_F$ such that $F(x) \leq C_F$.
\end{enumerate}
We use $a\xrightarrow[]{}b$ to indicate that $a$ is a sufficient condition of $b$. Then we have
\begin{gather}
    \ref{item:strong convex}~\textit{and}~\ref{item:g lip 2} \xrightarrow[]{} \textit{ Assumption \ref{assumption:y*}};~~\ref{item:strong convex}-\ref{item:f lip} \xrightarrow[]{} \textit{ Assumption \ref{assumption:vglip}};~~\ref{item:noise bilevel} \xrightarrow[]{} \textit{Assumption \ref{assumption:noise}};\nonumber\\
    \ref{item:strong convex}\xrightarrow[]{}\textit{ Assumption \ref{assumption:strongly g}};~~
    \ref{item:F RSI}\xrightarrow[]{}\textit{ Assumption \ref{assumption:strongly v}};~~
    \ref{item:F}\xrightarrow[]{}\textit{ Assumption \ref{assumption:integral}}.
\end{gather}
\end{Lemma}
 The conditions listed above are standard in the literature \citep{ghadimi2018approximation,hong2020ac,chen2021tighter}.
It is worth noting that Lemma \ref{lemma:bilevel} does not need the $L_{xy}$-Lipschitz continuity condition of $f(x,y)$ w.r.t. $(x,y)$. This Lipschitz condition, along with the $L_y$-Lipschitz continuity of $y^*(x)$ implied by the  conditions in Lemma \ref{lemma:bilevel}, further leads to the Lipschitz continuity of $F(x)$:
\begin{align}
    |F(x)-F(x')| \leq L_{xy}(\|x-x'\|+\|y^*(x)-y^*(x')\|)\leq L_{xy}(L_y+1)\|x-x'\|.
\end{align}
Although it is rather restrictive, this condition has been used in the previous work when $F(x)$ is not strongly-concave/convex.
Lastly, condition $\ref{item:noise bilevel}$ can be guaranteed by using \citep[Algorithm 3]{ghadimi2018approximation} to obtain a good enough $H_k^{yy}$, which requires extra $\Omega(\log  \epsilon^{-1})$ samples per iteration.
With Lemma \ref{lemma:bilevel}, we have the following corollary regarding the convergence of \eqref{eq:bilevel update}.
\begin{Corollary}[STSA for SBO]\label{corollary:bilevel}
Consider the STSA sequences with the update in \eqref{eq:bilevel update}. Under Conditions \ref{item:strong convex}--\ref{item:noise bilevel}, Theorem \ref{theorem:strongly monotone} holds; that is, with $\alpha_k=\Theta(\frac{1}{k})$ and $\beta_k=\Theta(\frac{1}{k})$ we have
\begin{subequations}
\begin{align}
     &\E\|x_k-x^*\|^2 + \E\|y_k-y^*(x_k)\|^2 \!=\! \mathcal{O}\Big(\frac{1}{k} \Big),\\
    &\lim_{k\xrightarrow[]{}\infty}\|x_k- x^*\|^2=0~~~{\rm and}~~~
    \lim_{k\xrightarrow[]{}\infty}\|y_k- y^*(x_k)\|^2 = 0~~a.s.
\end{align}
\end{subequations}
Under Conditions \ref{item:strong convex}--\ref{item:f lip}, \ref{item:noise bilevel} and \ref{item:F}, Theorem \ref{theorem:non-strongly-monotone-2} holds; i.e., with {\small$\alpha_k\!=\!\Theta(\frac{1}{\sqrt{K}})$, $\beta_k\!=\!\Theta(\frac{1}{\sqrt{K}})$}, we have
\begin{align}
     &\frac{1}{K}\sum_{k=1}^K \Big(\E\|\nabla F(x_k)\|^2 + \E\|y_k - y^*(x_k)\|^2\Big) = \mathcal{O}\Big(\frac{1}{\sqrt{K}} \Big).
\end{align}
\end{Corollary}
\begin{Remark}[Comparison with prior art in single-loop SBO] \emph{If $F(x)$ is strongly concave, Corollary \ref{corollary:bilevel} implies the sample complexity of $\mathcal{O}(\epsilon^{-1}\!\log \epsilon^{-1})$, which improves over the best-known sample complexity $\mathcal{O}(\epsilon^{-1.5}\log \epsilon^{-1})$ in \citep{hong2020ac}.  Different from \citep{hong2020ac}, we do not need the projection of $x_k$ to a compact set. When $F(x)$ is non-concave, corollary \ref{corollary:bilevel} suggests a sample complexity of $\mathcal{O}(\epsilon^{-2}\log \epsilon^{-1} )$, which is the same as the state-of-art complexity established in \citep{chen2021tighter}. Corollary \ref{corollary:bilevel} improves the result in \citep{chen2021tighter} in two major aspects: 1) By carefully utilizing the smoothness of $y^*(x)$, we avoid directly bounding the term $\|x_{k+1}-x_k\|^2$ in the proof, which was originally bounded using the bounded gradient assumption in \citep{chen2021tighter}. Thus we relax the Lipschitz continuity assumption on $f(x,y)$ on $x$; and, 2) An alternating update is adopted in \citep{chen2021tighter} to ensure stability, while some applications of SBO only allow simultaneous updates. Corollary \ref{corollary:bilevel} applies to those cases and thus has a broader range of application.}
\end{Remark}

\subsection{Application to advantage actor-critic}
The advantage actor-critic (AC) is one of the most celebrated methods in RL. We will apply our new nonlinear SA results to AC by starting with the basic concepts in RL.

RL problems are often modeled as a MDP described by $\mathcal{M}=\{ \mathcal{S}, \mathcal{A}, \mathcal{P}, r, \gamma \}$, where $\mathcal{S}$ is the state space, $\mathcal{A}$ is the action space; $\mathcal{P}(s'|s,a)$ is the probability of transitioning to $s'\in \mathcal{S}$ given $(s,a)\!\in\!\mathcal{S}\!\times\!\mathcal{A}$; $r(s,a)\in[0,1]$ is the reward associated with $(s,a)$; and $\gamma \in (0,1)$ is a discount factor. A policy $\pi$ maps $\mathcal{S}$ to a probability distribution over $\mathcal{A}$, and we use $\pi(a|s)$ to denote the probability of choosing $a$ under $s$.  
Given a policy $\pi$, we define the value functions as
$V_{\pi}(s) \coloneqq  \mathbb{E}_{\pi}\!\Big[ \sum_{t=0}^\infty \gamma^t r(s_t, a_t) \mid s_0\!=\!s\Big]$,
where $\E_\pi$ is taken over the trajectory $(s_0,a_0,s_1,a_1,\ldots)$ generated under policy $\pi$ and transition kernel $\mathcal{P}$. 
With $\rho$ denoting the initial state distribution, the discounted visitation distribution induced by policy $\pi$ is defined via $d_{\pi}(s,a) = (1-\gamma) \sum_{t=0}^\infty \gamma^t \mathbf{Pr}_{\pi}(s_t = s\mid s_0\sim \rho)\pi(a|s)$. 
To overcome the difficulty of
learning a function, we parameterize the policy with $x \in \mathbb{R}^{d_0}$, and solve
\begin{align}\label{eq:Jpi}
    \max_{x \in \mathbb{R}^{d_0}} ~F(x) \coloneqq (1-\gamma)\E_{s \sim \rho} [V_{\pi_x}(s)].
\end{align}
The policy gradient method seeks to solve this optimization problem with gradient ascent \citep{SuttonPG}, which takes the following form:
\begin{align}\label{eq:pg}
    \nabla F(x) \!=\! \E_{s,a \sim d_{\pi_x}, s'\!\sim\!\mathcal{P}}[(r(s,a)\!+\!\gamma V_{\pi_x}(s')\!-\!V_{\pi_x}(s))\nabla \log \pi_x(a|s)].
\end{align}
In order to compute \eqref{eq:pg}, one will need to know the value function $V_{\pi_x}$ which can be difficult to compute in practice. A popular method is to estimate the value function with $\hat{V}_y$ parameterized by $y \in \mathbb{R}^{d_1}$. In this work, we focus on the linear function approximation case, i.e. $\hat{V}_y (s) = \phi(s)^\top y$ where $\phi:\mathcal{S}\mapsto\mathbb{R}^{d_1}$ is the feature vector. We update $y$ with the temporal difference (TD) learning method \citep{sutton1988td} which takes the same form as that of \eqref{eq:updatey2} with $N\!=\!1$, given by
\begin{align}\label{eq:tdupdate}
    h(x,y) &= \E_{s\sim \mu_{\pi_x},a\sim\pi_x,s'\sim\mathcal{P}}[\phi(s)(\gamma \phi(s') -\phi(s))^\top] y + \E_{s\sim \mu_{\pi_x},a\sim \pi_x}[r(s,a) \phi(s)],\nonumber\\
    \psi_k &= -h(x_k,y_k) + \phi(s_k)(\gamma \phi(s_k') -\phi(s_k))^\top y + r(s_k,a_k) \phi(s_k)
\end{align}
where $\mu_{\pi_x}$ is the stationary distribution of the Markov chain induced by policy $\pi_x$. Sample $(s_k,a_k,s_k')$ is returned by some sampling protocol.
Under some regularity conditions, it is known that there exists a unique $y^*(x)$ such that $h(x,y^*(x))=0$ \citep{JBTD}. For simplicity, we focus on the case where $V_{\pi_x}(s)$ is linear-realizable; that is, $V_{\pi_x}(s)=\hat{V}_{y^*(x)} (s)$ for any $s$. 

With the estimated value function $\hat{V}_{y_k}$, we can then perform the policy gradient update by replacing $V_{\pi_\theta}$ with $\hat{V}_{y_k}$ in \eqref{eq:pg}. This update is a special case of \eqref{eq:updatex2} by defining
\begin{subequations}\label{eq:actorupdate}
\begin{align}
    v(x,y) & =  \E_{s,a \sim d_{\pi_x},s'\sim\mathcal{P}}[(r(s,a)\!+\! (\gamma\phi(s')\!-\!\phi(s))^\top y)\nabla \log \pi_x(a|s)], \\
    \xi_k &= -v(x_k,y_k) + r(\Bar{s}_k,\Bar{a}_k) +\gamma (\phi(\Bar{s}_k')-\phi(\Bar{s}_k))^\top y_k\nabla \log \pi_{x_k}(\Bar{a}_k|\Bar{s}_k).
\end{align}
\end{subequations}

With the AC update \eqref{eq:tdupdate} and \eqref{eq:actorupdate} written in the form of STSA, we next verify the assumptions of STSA. With $A_x \coloneqq \E_{s\sim \mu_{\pi_x},a\sim\pi_x,s'\sim\mathcal{P}}[\phi(s)(\gamma \phi(s') -\phi(s))^\top], b_x \coloneqq \E_{s\sim \mu_{\pi_x},a\sim \pi_x}[r(s,a) \phi(s)]$, we have the following lemma.
\begin{Lemma}[Verifying the assumptions of STSA]\label{lemma:ac}
Consider the following conditions
\begin{enumerate}[label=(\alph*)]
\setcounter{enumi}{14}
    \item\label{item:A ac} For any $s\in\mathcal{S}$, $\|\phi(s)\|\leq 1$. For any $x \in \mathbb{R}^{d_0}$, there exists a constant $\lambda_1 > 0$ such that $\langle y-y', A_x (y-y') \rangle \leq -\lambda_1 \|y-y'\|^2$ for any $y,y'\in \mathbb{R}^{d_1}$. The smallest singular value of $A_x$ is lower bounded by $\sigma>0$. 
    \item \label{item:policy ac} There exist constants $L_\pi$,$L_\pi'$ and $C_\pi$ such that for any $s\in\mathcal{S}$ and $a\in\mathcal{A}$ and $x,x'\in \mathbb{R}^{d_0}$, the following inequalities hold: i) $\|\pi_{x}(a|s)-\pi_{x'}(a|s)\|\leq L_\pi \|x-x'\|$. ii) $\|\nabla\log\pi_{x}(a|s)-\nabla\log\pi_{x'}(a|s)\|\leq L'_\pi \|x-x'\|$. iii) $\|\nabla\log\pi_x(a|s)\| \leq C_\pi$.
    \item\label{item:ergodic} For any $x\in\mathbb{R}^{d_0}$, the Markov chain induced by the policy $\pi_x$ and transition kernel $\mathcal{P}$ is ergodic. There exist positive constants $\kappa$ and $\rho<1$ such that
    \begin{align}
        \|\mathbb{P}_{\pi_x}(s_t\in\cdot|s_0=s,a_0=a)-\mu_{\pi_x}(\cdot)\|_{TV} \leq \kappa \rho^t,~\forall (s,a)\in \mathcal{S}\times\mathcal{A}
    \end{align}
    where $\mathbb{P}_{\pi_x}(s_t\in \cdot|s_0,a_0)$ is the probability measure of the $t$th state $s_t$ on the Markov chain induced by policy $\pi_x$ and transition kernel $\mathcal{P}$, given the initial state and action $s_0,a_0$.
    \item\label{item:noise ac} We sample by $s_k,a_k \!\sim\! d_{\pi_x}$, $s_k\!\sim\!\mathcal{P}(\cdot|s_k,a_k)$; $\Bar{s}_k\!\sim\!\mu_x$, $\Bar{a}_k \!\sim\! \pi_x(\cdot|\Bar{s}_k)$, $\Bar{s}_k' \!\sim\!\mathcal{P}(\cdot|\bar{s}_k,\bar{a}_k)$.
\end{enumerate}
Then we have
\begin{gather}
    \textit{\ref{item:A ac}--\ref{item:ergodic}} \xrightarrow[]{} \textit{Assumption \ref{assumption:y*}};~~~\textit{\ref{item:A ac}~and~\ref{item:policy ac}}\xrightarrow[]{} \textit{Assumption \ref{assumption:vglip}~and~\ref{assumption:strongly g}};\nonumber\\
    \textit{Assumption \ref{assumption:integral} holds with bounded reward}.
\end{gather}
Moreover, a slightly modified version of Assumption \ref{assumption:noise} holds under condition \ref{item:A ac}~and~\ref{item:noise ac}:
\begin{align}\label{eq:asm 3 ac}
    &\E[\xi_k|\mathcal{F}_k^{1}]=0,~\E[\psi_k|\mathcal{F}_k]=0\nonumber\\
    &\|\xi_k\|^2 \leq \sigma_0^2+\bar{\sigma}_0^2\|y_k-y^*(x_k)\|^2,\|\psi_k\|^2 \leq \sigma_1^2 + \bar{\sigma}_1^2\|y_k-y^*(x_k)\|^2
\end{align}
where $\sigma_0^2 = 8 C_\pi^2(1+4\sigma^{-2})$, $\bar{\sigma}_0^2 = 32 C_\pi^2$, $\sigma_1^2 = 32 \sigma^{-2}+8$ and $\bar{\sigma}_1^2 = 32$.
\end{Lemma}
The conditions listed above are standard in the literature \citep{wu2020finite,chen2021tighter}.
For the AC update in \eqref{eq:tdupdate} and \eqref{eq:actorupdate}, Assumption \ref{assumption:vglip}--\ref{assumption:strongly g} and \ref{assumption:integral} or their sufficient conditions have been explored in the RL context by previous works \citep{wu2020finite}. However, the smoothness of $y^*(x)$ in Assumption \ref{assumption:y*}, which is the key condition leading to a faster convergence rate, has yet been proven and was directly assumed in \citep{chen2021tighter}. While in the proof of Lemma \ref{lemma:ac}, we provide a formal justification for the smoothness of $y^*(x)$. Lastly, note that Assumption \ref{assumption:noise} is modified to \eqref{eq:asm 3 ac} in AC. As $y_k$ converges to $y^*(x_k)$, \eqref{eq:asm 3 ac} reduces to Assumption \ref{assumption:noise}. As a result, though \eqref{eq:asm 3 ac} is a special property for AC, the convergence analysis is still within the framework of STSA.

 We then directly present the following theorem with the proof deferred to Appendix \ref{section:ac appendix}.
\begin{Theorem}\label{theorem:ac}
Consider the sequences generated by \eqref{eq:tdupdate} and \eqref{eq:actorupdate} for $k=[K]$. Under conditions \ref{item:A ac}--\ref{item:noise ac}, Theorem \ref{theorem:non-strongly-monotone-2} holds; that is, with $\alpha_k\!=\!\Theta(\frac{1}{\sqrt{K}})$ and $\beta_k\!=\!\Theta(\frac{1}{\sqrt{K}})$, we have
\begin{align}
     &\frac{1}{K}\sum_{k=1}^K \Big(\E\|\nabla F(x_k)\|^2 + \E\|y_k - y^*(x_k)\|^2\Big) = \mathcal{O}\Big(\frac{1}{\sqrt{K}} \Big).
\end{align}
\end{Theorem}
In \citep{chen2021tighter}, the projection step is adopted in the $y_k$ update to ensure that $\|y_k\|< \infty,\forall k$. Since the projection radius is unknown in practice, adopting the projection is essentially assuming that $\|y_k\|$ is bounded for any $k$, which is strong. Theorem \ref{theorem:ac} holds without this projection.

\section{Application to the stochastic compositional optimization method}\label{section:sco}
In this section, we apply our generic STSA results to the stochastic compositional optimization problems and its applications to the meta policy gradient algorithm in RL.

Define mappings $f^n : \mathbb{R}^{d_n} \mapsto \mathbb{R}^{d_{n+1}}$ for $n=0,1,...,N$ with $d_{N+1}=1$. The stochastic compositional problem can be formulated as
\begin{align}\label{eq:sc}
    \max_{x\in \mathbb{R}^{d_0}}F(x) &\coloneqq f^{N}(f^{N-1}(\dots f^0(x)\dots) \nonumber\\
    {\rm with}~~~ f^n (x) &\coloneqq \E_{\zeta^n}[f^n (x;\zeta^n)],~ n=0,1,...,N
\end{align}
where $\zeta^0,\zeta^1,\dots,\zeta^N$ are random variables. 
Here we slightly overload the notation and use $f^n(x;\zeta^n)$ to represent the stochastic version of the mapping. 

\subsection{Reduction from the generic STSA results}\label{section:sc}
To solve the problem in \eqref{eq:sc}, a natural scheme is to use the stochastic gradient descent method with the gradient given by
\begin{align}\label{eq:grad sc}
    \nabla F(x) \!=\! \nabla f^0 (x) \nabla f^1 (f^0(x))\cdots \nabla f^N(f^{N-1}(\cdots f^0(x)\cdots))
\end{align}
where we use $\nabla f^n(f^{n-1}(\dots f^0(x)\dots)) = \nabla f^n(x)|_{x=f^{n-1}(\dots f^0(x)\dots)}$.
To obtain a stochastic estimator of $\nabla F(x)$, we will need to obtain the stochastic estimators for $\nabla f^n (f^{n-1}(... f^0(x)...))$ for each $n$. For example, when $n=1$, one need the estimator of $\nabla f^1 (\E_{\zeta^0}[f^0(x;\zeta^0)])$. However, due to the possible non-linearity of $\nabla f^1(\cdot)$, the natural candidate $\nabla f^1 (f^0(x;\zeta^0))$ is not an unbiased estimator of $\nabla f^1 (\E_{\zeta^0}[f^0(x;\zeta^0)])$. To tackle this issue, a popular method is to directly track  $\E_{\zeta^n}[f^n(\cdot;\zeta^n)]$ by variable $y^n, n=0,1,...,N$. The update takes the following form \citep{yang2018multilevel}:
\begin{subequations}\label{eq:scupdate}
\begin{align}
    y_{k+1}^n &= y_k^n - \beta_{k,n} (y_k^n- f^{n-1} (y_k^{n-1};\zeta_k^{n-1})),~n \in [N]\label{eq:scupdate-y}\\
    x_{k+1} & =  x_k + \alpha_k \nabla f^0 (x_k;\hat{\zeta}_k^0) \nabla f^1 (y_k^1;\zeta_k^1)\cdots \nabla f^N(y_k^N;\zeta_k^N)\label{eq:scupdate-x}
\end{align}
\end{subequations}
where $\hat{\zeta_k^0}$ and $\zeta_k^0$ have the same distribution as that of $\zeta^0$, and $\zeta_k^1,\dots,\zeta_k^N$ have the same distributions as that of $\zeta^1,\dots,\zeta^N$ respectively. When every $y_k^n$ reaches its fixed-point $y_k^{n,*}$, i.e. $y_k^n = y_k^{n,*} \coloneqq f^{n-1}(y_k^{n-1})$ for any $n \in [N]$, then it follows from \eqref{eq:grad sc} that $\nabla f^0 (x) \nabla f^1 (y_k^{1,*})\cdots \nabla f^N(y_k^{N,*}) \!=\! \nabla F(x_k)$,
which indicates that the expected update direction of $x_k$ in \eqref{eq:scupdate} is $\nabla F(x_k)$.

Therefore, the update of $y^n$ in \eqref{eq:scupdate-y} is a special case of the STSA update in \eqref{eq:updatey2} by defining
\begin{subequations}\label{eq:sch}
\begin{align}
\label{eq:scha}
    h^n(y^{n-1},y^n) &= f^{n-1}(y^{n-1}) - y^{n}, \\
    \psi_k^n &= -h^n(y_k^{n-1},y_k^n) + f^{n-1}(y_k^{n-1};\zeta_k^{n-1}) - y_k^n.
\end{align}
\end{subequations}
Likewise, the update of $x$ in \eqref{eq:scupdate-x} is a special case of the STSA update in \eqref{eq:updatex2} by defining
\begin{subequations}
\begin{align}
    v(x,y^1,y^2,\dots,y^N) &= \nabla f^0 (x) \nabla f^1 (y^1)... \nabla f^N(y^N), \\
    \xi_k &\!=\! -v(x_k,y_k^1,\dots,y_k^N) + \nabla f^0 (x_k;\hat{\zeta}_k^0) \cdots \nabla f^N(y_k^N;\zeta_k^N).
\end{align}
\end{subequations}
Next we provide a lemma that summarizes the sufficient conditions of Assumption \ref{assumption:y*}--\ref{assumption:integral}. The listed conditions are standard in the literature \citep{yang2018multilevel,chen2021solving}.
\begin{Lemma}[Verifying assumptions of STSA]\label{lemma:sc}
Consider the following conditions
\begin{enumerate}[label=(\alph*)]
\setlength\itemsep{-0.15em}
\setcounter{enumi}{6}
\item \label{item:f sc} Given any $n \in \{0,1,\dots,N\}$, there exist positive constants $L_{y,n}$ and $L_{y',n}$ such that the mapping $f^n(\cdot)$ is $L_{y,n}$-Lipschitz continuous and $L_{y',n}$-smooth.
\item \label{item:noise sc} Given $\mathcal{F}_k$, for any $n \in [N]$: $f^n (y_k^{n-1};\zeta_k^n)$ and $\nabla f^n (y_k^{n-1};\zeta_k^n)$ are respectively the unbiased estimators of $f^n (y_k^{n-1})$ and $\nabla f^n (y_k^{n-1})$ with bounded variance; $ f^0 (x_k;\hat{\zeta}_k^0)$ and $\nabla f^0 (x_k;\hat{\zeta}_k^0)$ are respectively the unbiased estimators of $f^0 (x_k)$ and $\nabla f^0 (x_k)$ with bounded variance.
\item\label{item:ind} At each iteration $k$, $\hat{\zeta}_k^0,\zeta_k^0,\zeta_k^1,\dots,\zeta_k^N$ are conditionally independent of each other given $\mathcal{F}_k$. 
\item \label{item:v monotone sc} Function $F(x)$ satisfies the restricted secant inequality: There exists a constant $\lambda_0 > 0$ such that $\langle \nabla F(x), x-x^*\rangle \leq - \lambda_0 \|x-x^*\|^2$, where $x^*\coloneqq \arg\max_{x\in\mathbb{R}^{d_1}}F(x)$.
\item\label{item:F sc} There exists a constant $C_F$ such that $F(x) \leq C_F$.
\end{enumerate}
We use $a\xrightarrow[]{}b$ to indicate that $a$ is a sufficient condition of $b$. Then we have
\begin{gather}
    \ref{item:f sc} \xrightarrow[]{} \textit{Assumption \ref{assumption:y*}~and~\ref{assumption:vglip}};~~~ \ref{item:noise sc} \textit{~and~} \ref{item:ind}\xrightarrow[]{} \textit{Assumption \ref{assumption:noise}};~~~ \ref{item:v monotone sc} \xrightarrow[]{} \textit{Assumption \textit{\ref{assumption:strongly v}}};\nonumber\\
    \ref{item:F sc} \xrightarrow[]{} \textit{Assumption \ref{assumption:integral}};~~~
    \textit{Assumption \ref{assumption:strongly g} holds naturally for \eqref{eq:scha}}.
\end{gather}
\end{Lemma}
With Lemma \ref{lemma:sc}, we can directly arrive at the following corollary on the convergence of the stochastic compositional optimization method.
\begin{Corollary}[STSA for multi-level SCO]\label{corollary:sc}
Consider the STSA sequences generated by update \eqref{eq:scupdate}. Under Conditions \ref{item:f sc}--\ref{item:v monotone sc}, Theorem \ref{theorem:strongly monotone} holds. Under Conditions \ref{item:f sc}--\ref{item:ind} and \ref{item:F sc}, Theorem \ref{theorem:non-strongly-monotone-2} holds.
\end{Corollary}
\begin{Remark}[Comparison with prior art in single-loop SCO]
\emph{Corollary \ref{corollary:sc} establishes the sample complexity of $\mathcal{O}(\epsilon^{-1})$ for the strongly monotone case and the complexity of $\mathcal{O}(\epsilon^{-2})$ for the non-monotone case, which are both independent of $N$. This improves over the $\mathcal{O}(\epsilon^{-\frac{N+5}{4}})$ complexity for the strongly concave case and the $\mathcal{O}(\epsilon^{-\frac{N+8}{4}})$ complexity for the non-concave case shown in \citep{yang2018multilevel}. There are other works that establish the same complexity as that in Corollary \ref{corollary:sc}, but they require modification to the basic SA update \eqref{eq:sch} to achieve acceleration; see e.g., \citep{chen2021solving,balasubramanian2021stochastic,ruszczynski2021stochastic}.}
\end{Remark}

\subsection{Application to model-agnostic meta policy gradient}
The model-agnostic meta policy gradient (MAMPG) aims to find a good starting policy that can be generalized to new tasks in a few stochastic policy gradient steps \citep{finn2017maml,fallah2021convergence}.  In this section, we will apply the STSA result in Section \ref{section:result} to MAMPG.

Consider a set of MDPs $\{\mathcal{M}_i\}_{i=1}^M$ with $\mathcal{M}_i=\{ \mathcal{S}, \mathcal{A}, \mathcal{P}_i, r_i, \gamma \}$. The MDPs model a set of RL tasks that share the same state and action space while having different transition kernels $\mathcal{P}_i$ and reward functions $r_i$.
To be better aligned with previous works \citep{fallah2021convergence}, instead of the infinite-horizon objective function defined in \eqref{eq:Jpi}, we consider the finite-horizon one.
Suppose the policy $\pi$ is parametrized by $x \in \mathbb{R}^{d_0}$, then the objective function of task $i$ is defined as 
\begin{equation}
    F_i (x) \coloneqq \E_{\zeta\sim\pi_x}\Big[\sum_{t=0}^H \gamma^t r_i(s_t,a_t) \mid\rho_i,\mathcal{P}_i\Big]
\end{equation}
where $H \in \mathbb{N}^+$ is the horizon, and $\E_{\zeta\sim\pi_x}$ is taken over the trajectory $\zeta \coloneqq (s_0,a_0,s_1,a_1,...,s_H,a_H)$ generated under the policy $\pi_x$, the initial distribution $\rho_i$ and the transition kernels $\mathcal{P}_i$.

The goal of MAMPG is to find an initial policy $\pi_x$ that can achieve good performance in new tasks by performing a few policy gradient steps. In the case where $N$ steps of gradient update are performed, the problem of finding an initial policy parameter $x$ can be formulated as
\begin{align}\label{eq:mampg}
    \max_{x\in\mathbb{R}^{d_0}} F(x) \coloneqq \frac{1}{M}\sum_{i=1}^M F_i (\Tilde{x}_i^N(x))~~~ {\rm with}\,\, \Tilde{x}_i^{n+1} = \Tilde{x}_i^{n}+ \eta \nabla F_i ( \Tilde{x}_i^n),~n=0,1,...,N\!-\!1
\end{align}
where $x$ is the shared initial policy parameter, i.e. $\tilde{x}_i^0 = x$ for any task $i$ and $\Tilde{x}_i^N(x)$ is the parameter after running $N$ steps of gradient ascent with respect to $F_i$ starting from $x$.

\textbf{Solving \eqref{eq:mampg} with SCO method.} The MAMPG problem in \eqref{eq:mampg} can be solved by the stochastic compositional optimization method introduced before. In order to get $\nabla F(x)$, one will need $\nabla F_i(\Tilde{x}_i^N(x))$ for each task $i$. Observe that $F_i(\Tilde{x}_i^N(x))$ can be written as a compositional function:
\begin{align}
    &F_i(\Tilde{x}_i^N(x)) = f_i^N(f_i^{N-1}(\dots f^0_i(x)\dots))~~ {\rm with}~~ f_i^n (x) \coloneqq x + \eta \nabla F_i(x),~n=0,...,N\!-\!1
\end{align}
where $f_i^N(x)=F_i(x)$.
In order to approximate $\nabla F_i(\Tilde{x}_i^N(x))$, we can follow the discussion in Section \ref{section:sc} and introduce tracking variables $y_i^n\in\mathbb{R}^{d_0}$ for $n\in[N]$ which are updated as follows
\begin{align}\label{eq:mampgupdate i}
    y_{k+1,i}^n &= y_{k,i}^n - \beta_{k,n} (y_{k,i}^n- f_i^{n-1} (y_k^{n-1};\zeta_{k,i}^{n-1})),~n=0,1,...,N\!-\!1
\end{align}
where we define $f_i^n(\cdot;\zeta)$ as a stochastic approximation of $f_i^n(\cdot)$ with the random trajectory $\zeta$. Then we estimate $\nabla F_i(\Tilde{x}_i^N(x))$ by $\hat{\nabla} F_{i,k} \!\coloneqq\! \nabla f_i^0 (x;\hat{\zeta}_{k,i}^0) \nabla f_i^1 (y_{k,i}^1;\zeta_{k,i}^1)\cdots \nabla f_i^N(y_{k,i}^N;\zeta_{k,i}^N)$.
Once we obtain $\hat{\nabla} F_{i,k}$ for each task, we update the initial policy as $ x_{k+1} = x_k +  \frac{\alpha_k}{M}\sum_{i=1}^M \hat{\nabla} F_{i,k}$.

Let $y^n \in \mathbb{R}^{d_n M}$ concatenate $y_i^n$ for $i\in[M]$, and $f^n(y_k^n;\zeta_k^n)$   concatenate $f_i^n(y_{k,i}^n;\zeta_{k,i}^n)$ for $i\in[M]$. We can write the MAMPG update as a special case of STSA \eqref{eq:multiSA} by
\begin{subequations}\label{eq:mampg update}
\begin{align}\label{eq:mampg h}
   \!\!\!    h^n(y^{n-1},y^n) &= f^{n-1}(y^{n-1}) - y^{n},~~\psi_k^n = -h^n(y_k^{n-1},y_k^n) + f^{n-1}(y_k^{n-1};\zeta_k^{n-1}) - y_k^n,\\
     \!\! \!\!\! v(x,y^1,\ldots,y^N) &\!=\!\frac{1}{M}\sum_{i=1}^M\nabla f_i^0 (x) \cdots\nabla f_i^N(y_i^N), ~ \xi_k \!=\! -v(x_k,y_k^1,\ldots,y_k^N) \!+\! \frac{1}{M}\sum_{i=1}^M \hat{\nabla} F_{i,k}.   \!
\end{align}
\end{subequations}

With the update \eqref{eq:mampg update}, we summarize the sufficient conditions for the assumptions of Theorem \ref{theorem:non-strongly-monotone-2}.

\begin{Lemma}[Verifying assumptions of STSA]\label{lemma:mampg}
Consider the following conditions:
\vspace{-0.05cm}
\begin{enumerate}[label=(\alph*)]
\setlength\itemsep{-0.16em}
\setcounter{enumi}{11}
\item\label{item:policy} There exist constants $L_\pi$,$L_\pi'$, $L_\pi''$ and $C_\pi$ such that for any $(s,a)\!\in\!\mathcal{S}\!\times\!\mathcal{A}$ and $x,x'\!\in\! \mathbb{R}^{d_0}$, we have: i) $\|\pi_{x}(a|s)\!\!-\!\!\pi_{x'}(a|s)\|\!\!\leq\!\! L_\pi \|x\!\!-\!\!x'\|$; ii) $\|\nabla\log\pi_{x}(a|s)\!-\!\nabla\log\pi_{x'}(a|s)\|\!\leq\! L'_\pi \|x\!-\!x'\|$; iii) $\|\nabla^2\log\pi_{x}(a|s)\!-\!\nabla^2\log\pi_{x'}(a|s)\|\!\leq\! L''_\pi \|x\!-\!x'\|$ and
iv) $\|\nabla\log\pi_x(a|s)\| \!\leq\! C_\pi$.
\item \label{item:noise mampg} Given $\mathcal{F}_k$, we have for any $n \in [N]$ and $i \in [M]$: $f_i^n (y_{k,i}^{n-1};\zeta_{k,i}^n)$ and $\nabla f_i^n (y_{k,i}^n;\zeta_{k,i}^n)$ are respectively the unbiased estimators of $f_i^n (y_{k,i}^{n-1})$ and $\nabla f_i^n (y_{k,i}^n)$ with bounded variance. Likewise, $f_i^0 (x_k;\zeta_{k,i}^0)$ and $\nabla f_i^0 (x_k;\hat{\zeta}_{k,i}^0)$ are respectively unbiased estimators of $f_i^0 (x_k)$ and $\nabla f_i^0 (x_k)$ with bounded variance.
\item\label{item:traj} Given $\mathcal{F}_k$, $\hat{\zeta}_{k,i}^0,\zeta_{k,i}^0,\zeta_{k,i}^1,\dots,\zeta_{k,i}^N$ are conditionally independent of each other for $i\in [M]$.
\end{enumerate}
\vspace{-0.05cm}
We use $a\xrightarrow[]{}b$ to indicate that $a$ is a sufficient condition of $b$. Then we have
\vspace{-0.1cm}
\begin{gather}
    \ref{item:policy}\xrightarrow[]{} \textit{Assumption \ref{assumption:y*}~and~ \ref{assumption:vglip}};~~~\ref{item:noise mampg}\textit{~and~}\ref{item:traj}\xrightarrow[]{} \textit{Assumption \ref{assumption:noise}};\nonumber\\
    \textit{Assumption \ref{assumption:strongly g} holds by plugging in \eqref{eq:mampg h}};~ \textit{Assumption \ref{assumption:integral} holds under the bounded reward}.
\end{gather}
\end{Lemma}
Condition \ref{item:policy} is a standard assumption commonly adopted in the literature; see e.g., \citep{fallah2021convergence,agarwal2019optimality}.
It is satisfied with certain popular policy parameterization such as the softmax policy. Conditions \ref{item:noise mampg}~and~\ref{item:traj} can be satisfied with the certain choice of estimators and sampling protocols which we will specify in the appendix.

\begin{Theorem}[Complexity of MAMPG]\label{theorem:mampg}
Consider the STSA sequences generated by the MAMPG update in \eqref{eq:mampg update} for $k=[K]$. We have Theorem \ref{theorem:non-strongly-monotone-2} holds under conditions \ref{item:policy}--\ref{item:traj}.
\end{Theorem}

Theorem \ref{theorem:mampg} implies a sample complexity of $\mathcal{O}(\epsilon^{-2})$ to achieve $\epsilon$-accuracy, which improves over the $\mathcal{O}(\epsilon^{-4})$ sample complexity in \citep{fallah2021convergence}. Moreover, Theorem \ref{theorem:mampg} holds for any $N \geq 1$  in \eqref{eq:mampg}, while the method in \citep{fallah2021convergence} only applies to the case $N=1$.

\vspace{-0.2cm}
\section{Conclusions}
In this work, we consider the general nonlinear SA with multiple coupled sequences, and study its non-asymptotic performance. Different from the dominating two-timescale SA analysis, we are particularly interested in under which conditions, single-timescale analysis can be applied to nonlinear SA with multiple coupled sequences.  
When all the sequences have strongly monotone increments, we establish the iteration complexity of $\mathcal{O}(\epsilon^{-1})$. When the main sequence is not strongly-monotone, we establish the iteration complexity of $\mathcal{O}(\epsilon^{-2})$. We then apply our generic SA analysis to stochastic bilevel and compositional optimization and improve their existing results. Specifically, we improve the state-of-the-art convergence rate of: 1) the SBO method and its application to the AC method; and, 2) the multi-level SCO method and its application to the MAMPG method.

\bibliography{abrv,bib}

\begin{thebibliography}{68}
\providecommand{\natexlab}[1]{#1}
\providecommand{\url}[1]{\texttt{#1}}
\expandafter\ifx\csname urlstyle\endcsname\relax
  \providecommand{\doi}[1]{doi: #1}\else
  \providecommand{\doi}{doi: \begingroup \urlstyle{rm}\Url}\fi

\bibitem[Agarwal et~al.(2020)Agarwal, Kakade, Lee, and
  Mahajan]{agarwal2019optimality}
A.~Agarwal, S.~M. Kakade, J.~D. Lee, and G.~Mahajan.
\newblock Optimality and approximation with policy gradient methods in markov
  decision processes.
\newblock In \emph{Proc. of Thirty Third Conference on Learning Theory}, 2020.

\bibitem[Balasubramanian et~al.(2022)Balasubramanian, Ghadimi, and
  Nguyen]{balasubramanian2021stochastic}
K.~Balasubramanian, S.~Ghadimi, and A.~Nguyen.
\newblock Stochastic multi-level composition optimization algorithms with
  level-independent convergence rates.
\newblock \emph{SIAM Journal on Optimization}, 32\penalty0 (2):\penalty0
  519--544, 2022.

\bibitem[Bhandari et~al.(2018)Bhandari, Russo, and Singal]{JBTD}
J.~Bhandari, D.~Russo, and R.~Singal.
\newblock A ﬁnite time analysis of temporal difference learning with linear
  function approximation.
\newblock In \emph{Proc. of Conference on Learning Theory}, 2018.

\bibitem[Borkar(1997)]{borkar1999stochastic}
V.~Borkar.
\newblock Stochastic approximation with two time scales.
\newblock \emph{System control letter}, 29, 1997.

\bibitem[Borkar(2009)]{borkar2009stochastic}
V.~Borkar.
\newblock \emph{Stochastic approximation: a dynamical systems viewpoint}.
\newblock Springer, 2009.

\bibitem[Borkar and Konda(1997)]{borkar1997actor}
V.~Borkar and V.~Konda.
\newblock The actor-critic algorithm as multi-time-scale stochastic
  approximation.
\newblock \emph{Sadhana}, 22\penalty0 (4):\penalty0 525--543, 1997.

\bibitem[Bottou et~al.(2018)Bottou, Curtis, and
  Nocedal]{bottou2018optimization}
L.~Bottou, F.~Curtis, and J.~Nocedal.
\newblock Optimization methods for large-scale machine learning.
\newblock \emph{SIAM Review}, 60\penalty0 (2), 2018.

\bibitem[Chen et~al.(2021{\natexlab{a}})Chen, Sun, and Yin]{chen2021solving}
T.~Chen, Y.~Sun, and W.~Yin.
\newblock Solving stochastic compositional optimization is nearly as easy as
  solving stochastic optimization.
\newblock \emph{IEEE Transactions on Signal Processing}, 69:\penalty0
  4937–4948, 2021{\natexlab{a}}.

\bibitem[Chen et~al.(2021{\natexlab{b}})Chen, Sun, and Yin]{chen2021tighter}
T.~Chen, Y.~Sun, and W.~Yin.
\newblock Tighter analysis of alternating stochastic gradient method for
  stochastic nested problems.
\newblock In \emph{Proc. of Advances in Neural Information Processing Systems},
  2021{\natexlab{b}}.

\bibitem[Chen et~al.(2022)Chen, Sun, Xiao, and Yin]{chen2022single}
T.~Chen, Y.~Sun, Q.~Xiao, and W.~Yin.
\newblock A single-timescale method for stochastic bilevel optimization.
\newblock In \emph{Proc. of Intl. Conf. Artificial Intelligence and
  Statistics}, pages 2466--2488, 2022.

\bibitem[Dagreou et~al.(2022)Dagreou, Ablin, Vaiter, and
  Moreau]{dagreou2022framework}
M.~Dagreou, P.~Ablin, S.~Vaiter, and T.~Moreau.
\newblock A framework for bilevel optimization that enables stochastic and
  global variance reduction algorithms.
\newblock \emph{arXiv preprint arXiv:2201.13409}, 2022.

\bibitem[Dalal et~al.(2018)Dalal, Szorenyi, Thoppe, and
  Mannor]{dalal2018finite}
G.~Dalal, B.~Szorenyi, G.~Thoppe, and S.~Mannor.
\newblock Finite sample analysis of two-timescale stochastic approximation with
  applications to reinforcement learning.
\newblock In \emph{Proc. of Conference on Learning Theory}, 2018.

\bibitem[Dalal et~al.(2020)Dalal, Szorenyi, and Thoppe]{dalal2020tale}
G.~Dalal, B.~Szorenyi, and G.~Thoppe.
\newblock A tale of two-timescale reinforcement learning with the tightest
  finite-time bound.
\newblock In \emph{Proc. of AAAI Conference on Artificial Intelligence}, 2020.

\bibitem[Doan(2021)]{doan2021nonlinear}
T.~Doan.
\newblock Nonlinear two-time-scale stochastic approximation: Convergence and
  finite-time performance.
\newblock \emph{arXiv preprint:2011.01868}, 2021.

\bibitem[Durmus et~al.(2021)Durmus, Moulines, Naumov, Samsonov, and
  Wai]{durmus2021stability}
A.~Durmus, E.~Moulines, A.~Naumov, S.~Samsonov, and H.~Wai.
\newblock On the stability of random matrix product with markovian noise:
  Application to linear stochastic approximation and td learning, 2021.

\bibitem[Fallah et~al.(2021)Fallah, Georgiev, Mokhtari, and
  Ozdaglar]{fallah2021convergence}
A.~Fallah, K.~Georgiev, A.~Mokhtari, and A.~Ozdaglar.
\newblock On the convergence theory of debiased model-agnostic
  meta-reinforcement learning.
\newblock In \emph{Proc. of Advances in Neural Information Processing Systems},
  2021.

\bibitem[Finn et~al.(2017)Finn, Abbeel, and Levine]{finn2017maml}
C.~Finn, P.~Abbeel, and S.~Levine.
\newblock Model-agnostic meta-learning for fast adaptation of deep networks.
\newblock In \emph{Proc. of Intl. Conf. Machine Learning}, 2017.

\bibitem[Franceschi et~al.(2018)Franceschi, Frasconi, Salzo, Grazzi, and
  Pontil]{franceschi2018bilevel}
L.~Franceschi, P.~Frasconi, S.~Salzo, R.~Grazzi, and M.~Pontil.
\newblock Bilevel programming for hyperparameter optimization and
  meta-learning.
\newblock In \emph{Proc. of Intl. Conf. Machine Learning}, 2018.

\bibitem[Fu et~al.(2020)Fu, Yang, and Wang]{fu2020single}
Z.~Fu, Z.~Yang, and Z.~Wang.
\newblock Single-timescale actor-critic provably finds globally optimal policy.
\newblock \emph{arXiv preprint:2008.00483}, 2020.

\bibitem[Gadet(2017)]{gadet2017stochastic}
S.~Gadet.
\newblock Stochastic optimization algorithms.
\newblock
  \url{https://perso.math.univ-toulouse.fr/gadat/files/2012/12/cours_Algo_Stos_M2R.pdf},
  2017.

\bibitem[Ghadimi and Wang(2018)]{ghadimi2018approximation}
S.~Ghadimi and M.~Wang.
\newblock Approximation methods for bilevel programming.
\newblock \emph{arXiv preprint:1802.02246}, 2018.

\bibitem[Ghadimi et~al.(2020)Ghadimi, Ruszczynski, and Wang]{ghadimi2020jopt}
S.~Ghadimi, A.~Ruszczynski, and M.~Wang.
\newblock A single timescale stochastic approximation method for nested
  stochastic optimization.
\newblock \emph{SIAM Journal on Optimization}, 30\penalty0 (1):\penalty0
  960--979, 2020.

\bibitem[Grazzi et~al.(2020)Grazzi, Franceschi, Pontil, and
  Salzo]{grazzi2020iteration}
R.~Grazzi, L.~Franceschi, M.~Pontil, and S.~Salzo.
\newblock On the iteration complexity of hypergradient computation.
\newblock In \emph{Proc. of Intl. Conf. Machine Learning}, pages 3748--3758,
  2020.

\bibitem[Grazzi et~al.(2022)Grazzi, Pontil, and Salzo]{grazzi2022bilevel}
R.~Grazzi, M.~Pontil, and S.~Salzo.
\newblock Bilevel optimization with a lower-level contraction: Optimal sample
  complexity without warm-start.
\newblock \emph{arXiv preprint arXiv:2202.03397}, 2022.

\bibitem[Guo and Yang(2021)]{guo2021randomized}
Z.~Guo and T.~Yang.
\newblock Randomized stochastic variance-reduced methods for stochastic bilevel
  optimization.
\newblock \emph{arXiv preprint arXiv:2105.02266}, 2021.

\bibitem[Gupta et~al.(2019)Gupta, Srikant, and Ying]{gupta2019finitetime}
H.~Gupta, R.~Srikant, and L.~Ying.
\newblock Finite-time performance bounds and adaptive learning rate selection
  for two time-scale reinforcement learning.
\newblock In \emph{Proc. of Advances in Neural Information Processing Systems},
  2019.

\bibitem[Hong et~al.(2020)Hong, Wai, Wang, and Yang]{hong2020ac}
M.~Hong, H.-T. Wai, Z.~Wang, and Z.~Yang.
\newblock A two-timescale framework for bilevel optimization: Complexity
  analysis and application to actor-critic.
\newblock \emph{arXiv preprint:2007.05170}, 2020.

\bibitem[Hu et~al.(2022)Hu, Zhong, and Yang]{hu2022multi}
Q.~Hu, Y.~Zhong, and T.~Yang.
\newblock Multi-block min-max bilevel optimization with applications in
  multi-task deep auc maximization.
\newblock \emph{arXiv preprint arXiv:2206.00260}, 2022.

\bibitem[Ji and Liang(2021)]{ji2021lower}
K.~Ji and Y.~Liang.
\newblock Lower bounds and accelerated algorithms for bilevel optimization.
\newblock \emph{arXiv preprint arXiv:2102.03926}, 2021.

\bibitem[Ji et~al.(2021)Ji, Yang, and Liang]{ji2020provably}
K.~Ji, J.~Yang, and Y.~Liang.
\newblock Provably faster algorithms for bilevel optimization and applications
  to meta-learning.
\newblock In \emph{Proc. of Intl. Conf. Machine Learning}, 2021.

\bibitem[Jiang et~al.(2022)Jiang, Wang, Wang, Zhang, and
  Yang]{jiang2022optimal}
W.~Jiang, B.~Wang, Y.~Wang, L.~Zhang, and T.~Yang.
\newblock Optimal algorithms for stochastic multi-level compositional
  optimization.
\newblock \emph{arXiv preprint arXiv:2202.07530}, 2022.

\bibitem[Kaledin et~al.(2020)Kaledin, Moulines, Naumov, Tadic, and
  Wai]{kaledin2020finite}
M.~Kaledin, E.~Moulines, A.~Naumov, V.~Tadic, and H.~Wai.
\newblock Finite time analysis of linear two-timescale stochastic approximation
  with markovian noise.
\newblock \emph{Proc. of Conference on Learning Theory}, 2020.

\bibitem[Karimi et~al.(2019)Karimi, Miasojedow, Moulines, and
  Wai]{karimi2019nonasymptotic}
B.~Karimi, B.~Miasojedow, E.~Moulines, and H.~Wai.
\newblock Non-asymptotic analysis of biased stochastic approximation scheme.
\newblock In \emph{Proc. of Conference on Learning Theory}, 2019.

\bibitem[Khanduri et~al.(2021)Khanduri, Zeng, Hong, Wai, Wang, and
  Yang]{khanduri2021momentum}
P.~Khanduri, S.~Zeng, M.~Hong, H.~Wai, Z.~Wang, and Z.~Yang.
\newblock A momentum-assisted single-timescale stochastic approximation
  algorithm for bilevel optimization.
\newblock \emph{arXiv preprint arXiv:2102.07367}, 2021.

\bibitem[Konda(2002)]{KondaAC}
V.~Konda.
\newblock \emph{Actor-critic algorithms.}
\newblock PhD thesis, Department of Electrical Engineering and Computer
  Science, Massachusetts Institute of Technology, 2002.

\bibitem[Konda and Borkar(1999)]{konda1999actor}
V.~Konda and V.~Borkar.
\newblock Actor-critic--type learning algorithms for markov decision processes.
\newblock \emph{SIAM Journal on Control and Optimization}, 38\penalty0
  (1):\penalty0 94--123, 1999.

\bibitem[Konda and Tsitsiklis(2004)]{konda2004convergence}
V.~Konda and J.~Tsitsiklis.
\newblock Convergence rate of linear two-time-scale stochastic approximation.
\newblock \emph{The Annals of Applied Probability}, 14\penalty0 (2), 2004.

\bibitem[Kumar et~al.(2019)Kumar, Koppel, and Ribeiro]{kumar2019sample}
H.~Kumar, A.~Koppel, and A.~Ribeiro.
\newblock On the sample complexity of actor-critic method for reinforcement
  learning with function approximation.
\newblock \emph{arXiv preprint:1910.08412}, 2019.

\bibitem[Li et~al.(2022)Li, Gu, and Huang]{Li2022fully}
J.~Li, B.~Gu, and H.~Huang.
\newblock A fully single loop algorithm for bilevel optimization without
  hessian inverse.
\newblock In \emph{Proc. of AAAI Conference on Artificial Intelligence}, 2022.

\bibitem[Liu et~al.(2020)Liu, Mu, Yuan, Zeng, and Zhang]{liu2020icml}
R.~Liu, P.~Mu, X.~Yuan, S.~Zeng, and J.~Zhang.
\newblock A generic first-order algorithmic framework for bi-level programming
  beyond lower-level singleton.
\newblock In \emph{Proc. of Intl. Conf. Machine Learning}, 2020.

\bibitem[Mokkadem and Pelletier(2006)]{mokkadem2006convergence}
A.~Mokkadem and M.~Pelletier.
\newblock Convergence rate and averaging of nonlinear two-time-scale stochastic
  approximation algorithms.
\newblock \emph{The Annals of Applied Probability}, 16\penalty0 (3), 2006.

\bibitem[Moshe and Ludo(1984)]{moshe1984perturbation}
H.~Moshe and V.~Ludo.
\newblock Perturbation bounds for the stationary probabilities of a finite
  markov chain.
\newblock \emph{Advances in Applied Probability}, 16\penalty0 (4):\penalty0
  804--818, 1984.

\bibitem[Mou et~al.(2020)Mou, Li, Wainwright, Bartlett, and
  Jordan]{mou2020linear}
W.~Mou, J.~Li, M.~Wainwright, P.~Bartlett, and M.~Jordan.
\newblock On linear stochastic approximation: Fine-grained polyak-ruppert and
  non-asymptotic concentration, 2020.

\bibitem[Moulines and Bach(2011)]{moulines2011nonasympSA}
E.~Moulines and F.~Bach.
\newblock Non-asymptotic analysis of stochastic approximation algorithms for
  machine learning.
\newblock In \emph{Proc. of Advances in Neural Information Processing Systems},
  2011.

\bibitem[Nemirovski et~al.(2009)Nemirovski, Juditsky, Lan, and
  Shapiro]{nemirovski2009robust}
A.~Nemirovski, A.~Juditsky, G.~Lan, and A.~Shapiro.
\newblock Robust stochastic approximation approach to stochastic programming.
\newblock \emph{SIAM Journal on Optimization}, 19\penalty0 (4):\penalty0
  1574–1609, 2009.

\bibitem[Olshevsky and Gharesifard(2022)]{olshevsky2022small}
A.~Olshevsky and B.~Gharesifard.
\newblock A small gain analysis of single timescale actor critic.
\newblock \emph{arXiv preprint arXiv:2203.02591}, 2022.

\bibitem[Qiu et~al.(2019)Qiu, Yang, Ye, and Wang]{qiu2019finite}
S.~Qiu, Z.~Yang, J.~Ye, and Z.~Wang.
\newblock On the finite-time convergence of actor-critic algorithm.
\newblock In \emph{Optimization Foundations for Reinforcement Learning Workshop
  at Advances in Neural Information Processing Systems}, 2019.

\bibitem[Robbins and Monro(1951)]{robbins1951stochastic}
H.~Robbins and S.~Monro.
\newblock A stochastic approximation method.
\newblock \emph{The Annals of Mathematical Statis- tics}, 22\penalty0
  (3):\penalty0 400--407, 1951.

\bibitem[Ruszczynski(2021)]{ruszczynski2021stochastic}
A.~Ruszczynski.
\newblock A stochastic subgradient method for nonsmooth nonconvex multilevel
  composition optimization.
\newblock \emph{SIAM Journal on Control and Optimization}, 59\penalty0
  (3):\penalty0 2301--2320, 2021.

\bibitem[Sabach and Shtern(2017)]{sabach2017jopt}
S.~Sabach and S.~Shtern.
\newblock A first order method for solving convex bilevel optimization
  problems.
\newblock \emph{SIAM Journal on Optimization}, 27\penalty0 (2):\penalty0
  640--660, 2017.

\bibitem[Sato et~al.(2021)Sato, Tanaka, and Takeda]{sato2021gradient}
R.~Sato, M.~Tanaka, and A.~Takeda.
\newblock A gradient method for multilevel optimization.
\newblock \emph{arXiv preprint:2105.13954}, 2021.

\bibitem[Srikant and Ying(2019)]{srikant2019finitetime}
R.~Srikant and L.~Ying.
\newblock Finite-time error bounds for linear stochastic approximation and td
  learning.
\newblock In \emph{Proc. of Conference on Learning Theory}, 2019.

\bibitem[Stackelberg(1952)]{stackelberg}
H.~Stackelberg.
\newblock \emph{The Theory of Market Economy}.
\newblock Oxford University Press, 1952.

\bibitem[Sutton et~al.(2009)Sutton, Maei, Precup, Bhatnagar, Silver, and
  Szepesv{\'a}ri]{sutton2009fast}
R.~Sutton, H.~Maei, D.~Precup, S.~Bhatnagar, D.~Silver, and E.~Szepesv{\'a}ri,
  C.and~Wiewiora.
\newblock Fast gradient-descent methods for temporal-difference learning with
  linear function approximation.
\newblock In \emph{Proc. of Intl. Conf. Machine Learning}, 2009.

\bibitem[Sutton et~al.(2000)Sutton, McAllester, Singh, and Mansour]{SuttonPG}
R.~S. Sutton, D.~McAllester, S.~Singh, and Y.~Mansour.
\newblock Policy gradient methods for reinforcement learning with function
  approximation.
\newblock In \emph{Proc. of Advances in Neural Information Processing Systems},
  2000.

\bibitem[Sutton(1988)]{sutton1988td}
R.S. Sutton.
\newblock Learning to predict by the methods of temporal differences.
\newblock \emph{Machine Learning}, 3:\penalty0 9--44, 1988.

\bibitem[Wang et~al.(2017{\natexlab{a}})Wang, Fang, and
  Liu]{wang2014stochastic}
M.~Wang, E.~Fang, and H.~Liu.
\newblock Stochastic compositional gradient descent: Algorithms for minimizing
  compositions of expected-value functions.
\newblock \emph{Mathematical Programming}, 161:\penalty0 419–449,
  2017{\natexlab{a}}.

\bibitem[Wang et~al.(2017{\natexlab{b}})Wang, Liu, and Fang]{wang2017jmlr}
M.~Wang, J.~Liu, and E.~Fang.
\newblock Accelerating stochastic composition optimization.
\newblock \emph{Journal of Machine Learning Research}, 18\penalty0
  (1):\penalty0 3721--3743, 2017{\natexlab{b}}.

\bibitem[Wu et~al.(2020)Wu, Zhang, Xu, and Gu]{wu2020finite}
Y.~Wu, W.~Zhang, P.~Xu, and Q.~Gu.
\newblock A finite time analysis of two time-scale actor critic methods.
\newblock In \emph{Proc. of Advances in Neural Information Processing Systems},
  2020.

\bibitem[Xiao et~al.(2022)Xiao, Balasubramanian, and
  Ghadimi]{xiao2022projection}
T.~Xiao, K.~Balasubramanian, and S.~Ghadimi.
\newblock A projection-free algorithm for constrained stochastic multi-level
  composition optimization.
\newblock \emph{arXiv preprint:2202.04296}, 2022.

\bibitem[Xu et~al.(2020)Xu, Wang, and Liang]{xu2020improving}
T.~Xu, Z.~Wang, and Y.~Liang.
\newblock Improving sample complexity bounds for (natural) actor-critic
  algorithms.
\newblock In \emph{Proc. of Advances in Neural Information Processing Systems},
  2020.

\bibitem[Yang et~al.(2021)Yang, Ji, and Liang]{yang2021provably}
J.~Yang, K.~Ji, and Y.~Liang.
\newblock Provably faster algorithms for bilevel optimization.
\newblock \emph{arXiv preprint arXiv:2106.04692}, 2021.

\bibitem[Yang et~al.(2018{\natexlab{a}})Yang, Wang, and
  Fang]{yang2018multilevel}
S.~Yang, M.~Wang, and E.~Fang.
\newblock Multi-level stochastic gradient methods for nested composition
  optimization.
\newblock \emph{SIAM Journal on Optimization}, 29\penalty0 (1),
  2018{\natexlab{a}}.

\bibitem[Yang et~al.(2018{\natexlab{b}})Yang, Zhang, Hong, and
  Ba{\c{s}}ar]{yang2018finite}
Z.~Yang, K.~Zhang, M.~Hong, and T.~Ba{\c{s}}ar.
\newblock A finite sample analysis of the actor-critic algorithm.
\newblock In \emph{Proc. of IEEE Conference on Decision and Control}, pages
  2759--2764, 2018{\natexlab{b}}.

\bibitem[Zhang and Xiao(2021)]{zhang2021multilevel}
J.~Zhang and L.~Xiao.
\newblock Multilevel composite stochastic optimization via nested variance
  reduction.
\newblock \emph{SIAM Journal on Optimization}, 31\penalty0 (2):\penalty0
  1131--1157, 2021.

\bibitem[Zhang et~al.(2019)Zhang, Koppel, Zhu, and Başar]{zhang2019global}
K.~Zhang, A.~Koppel, H.~Zhu, and T.~Başar.
\newblock Global convergence of policy gradient methods to (almost) locally
  optimal policies.
\newblock \emph{SIAM Journal on Control and Optimization}, 58\penalty0
  (6):\penalty0 3586–3612, 2019.

\bibitem[Zhang and Lan(2020)]{zhang2020optimal}
Z.~Zhang and G.~Lan.
\newblock Optimal algorithms for convex nested stochastic composite
  optimization.
\newblock \emph{arXiv preprint arXiv:2011.10076}, 2020.

\bibitem[Zou et~al.(2019)Zou, Xu, and Liang]{zou2019sarsa}
S.~Zou, T.~Xu, and Y.~Liang.
\newblock Finite-sample analysis for {SARSA} with linear function
  approximation.
\newblock In \emph{Proc. of Advances in Neural Information Processing Systems},
  2019.

\end{thebibliography}

\appendix

\section{Additional related works}
In this section, we review additional prior work on the applications of multi-sequence SA.

\textbf{Gradient-based stochastic bilevel optimization.} 
One recent application of multi-sequence SA is on the stochastic gradient-based bilevel optimization. 
The bilevel optimization was first introduced in \citep{stackelberg}. Recently, the gradient-based bilevel methods have gained growing popularity \citep{sabach2017jopt,franceschi2018bilevel,grazzi2020iteration,liu2020icml}. The finite-time convergence of the double-loop bilevel optimization methods has been studied in some previous works; see e.g., \citep{ghadimi2018approximation,ji2020provably,ji2021lower}. Later, \citep{hong2020ac} proved the finite-time convergence rate for the single-loop two time-scale bilevel optimization method, which was then improved by \citep{chen2021tighter} to the optimal rate with additional assumptions and a more refined analysis. There are also other works that incorporate momentum to accelerate the convergence rate; see e.g., \citep{chen2022single,khanduri2021momentum,guo2021randomized,yang2021provably}. After our initial conference submission, we have also noticed some concurrent works that are relevant to this work; e.g., \citep{dagreou2022framework,grazzi2022bilevel,Li2022fully,hu2022multi}. Specifically, \citep{dagreou2022framework} proposed a SBO method with the variance-reduction technique and achieved optimal rate. And \citep{grazzi2022bilevel} proposed a SBO method that achieves the optimal rate without warm-start. The algorithms in \citep{dagreou2022framework,grazzi2022bilevel} are not a case of the plain-vanilla SA update discussed in this work and thus its analysis is not applicable to our problem. 
In addition, the multi-block min-max bilevel algorithm has been studied in \citep{hu2022multi}. 
Lastly, \citep{Li2022fully} proposed a single-loop SBO method without Hessian inverse, but it required the bounded-gradient assumption not needed in this work.

\textbf{Gradient-based stochastic compositional optimization.}
Stochastic compositional optimization algorithm is also a recent application of multi-sequence SA.
The two time-scale stochastic compositional optimization method has been proposed in \citep{wang2014stochastic,wang2017jmlr}. Due to the two time-scale step sizes choice, the convergence rate of \citep{wang2014stochastic,wang2017jmlr} is slower than that of the SGD. In order to improve the convergence, \citep{ghadimi2020jopt,chen2021solving,ruszczynski2021stochastic,balasubramanian2021stochastic} have modified the basic update in \citep{wang2014stochastic,yang2018multilevel} and successfully established the convergence rate same as that of SGD. Concurrent to this work, \citep{jiang2022optimal} proposed a variance-reduced SCO method that achieved the optimal rate under variance-reduction, but the present work focuses on establishing an optimal rate for the SA update without having diminishing variance. Due to the difference in the update scheme, their analysis is not directly applicable to our SCO case.

\textbf{Actor-critic method with linear function approximation.} After its first introduction in \citep{konda1999actor}, the finite-sample guarantee for the AC algorithm has been established in \citep{yang2018finite,kumar2019sample,fu2020single} with i.i.d. sampling. In \citep{qiu2019finite}, the finite-time convergence rate has been established for the nested-loop AC under the Markovian setting, which was later improved improved by \citep{xu2020improving}. On the other hand, the finite-time convergence of two-timescale AC has been studied in \citep{wu2020finite} under Markovian sampling and \citep{hong2020ac,chen2021tighter} under i.i.d. sampling.
Concurrent to our work, a recent work \citep{olshevsky2022small} has used the small-gain theorem to analyze the single-loop AC algorithm and achieved the sample complexity with the  best-known dependence on the condition number. However, \citep{olshevsky2022small} still requires projection in the critic update.

\section{Proof of Theorem \ref{theorem:strongly monotone}}\label{section:strongly monotone appendix}
\subsection{Analysis of the lower-level sequences}

  For brevity, we define the shorthand notations $y_k^{n,*} \coloneqq y^{n,*}(y^{n-1}_k)$ with $y_k^{1,*} \coloneqq y^{1,*}(x_k)$. Also, we write $\E[\cdot|\mathcal{F}_k]$ as $\E_k[\cdot]$ for brevity. 

\noindent\textbf{One-step contraction of lower-level sequences.} With $y_k^0 = x_k$, it holds for any $n \in [N]$ that
\begin{align}\label{eq:yk+1-yk*start}
    \E_k\|y_{k+1}^n\!-\!y_k^{n,*}\|^2 
    &\!=\! \|y_k^n-y_k^{n,*}\|^2 \!+\! 2\beta_{k,n} \E_k\langle y_k^n\!-\!y_k^{n,*},h^n(y_k^{n-1},y_k^n)\!+\!\psi_k^n\rangle \!+\! \E_k\|y_{k+1}^n\!-\!y_k^n\|^2\!.\!
\end{align}
The second term in \eqref{eq:yk+1-yk*start} can be bounded as
\begin{align}\label{eq:yk+1-yk*start I2}
    \E_k\langle y_k^n-y_k^{n,*},h^n(y_k^{n-1},y_k^n)+\psi_k^n\rangle
    &= \langle y_k^n-y_k^{n,*},h^n(y_k^{n-1},y_k^n))  \rangle+\langle y_k^n-y_k^{n,*},\E_k[\psi_k^n]\rangle \nonumber\\
    &\leq -\lambda_{n} \|y_k^n-y_k^{n,*}\|^2 + \|y_k^n-y_k^{n,*}\|\|\E_k[\psi_k^n]\| \nonumber\\
    &\leq -\lambda_{n} \|y_k^n-y_k^{n,*}\|^2 + \frac{\lambda_{n}}{4}\|y_k^n-y_k^{n,*}\|^2 + \frac{1}{\lambda_{n}}\|\E_k[\psi_k^n]\|^2 \nonumber\\
    &\leq - \frac{3\lambda_{n}}{4}\|y_k^n-y_k^{n,*}\|^2 + \frac{ c_n^2}{\lambda_{n}} \beta_{k,n},
\end{align}
where the first inequality follows from the strong monotonicity of $h(y^{n-1},y^n)$ in Assumption \ref{assumption:strongly g}, the second inequality follows from the Young's inequality, and the last inequality follows from the bias of the increment $\psi_k^n$ in Assumption \ref{assumption:noise}.

The third term in \eqref{eq:yk+1-yk*start} can be bounded as
\begin{align}\label{eq:yk+1-yk*start I3}
    \E_k\|y_{k+1}^n-y_k^n\|^2
    &\leq 2\beta_{k,n}^2\big(\|h^n(y_k^{n-1},y_k^n)\|^2 +  \sigma_n^2\big) 
    \leq 2 L_{h,n}^2\beta_{k,n}^2 \|y_k^n - y_k^{n,*}\| + 2\sigma_n^2  \beta_{k,n}^2
\end{align}
where the last inequality follows from Assumption \ref{assumption:vglip} which gives
\begin{align}
   \|h^n(y_k^{n-1},y_k^n)\|=\|h^n(y_k^{n-1},y_k^n)-\underbrace{h^n(y_k^{n-1},y_k^{n,*}(y_k^{n-1}))}_{=0}\|\leq L_{h,n} \|y_k^n-y_k^{n,*}(y_k^{n-1})\|.
\end{align}

Collecting the upper bounds in \eqref{eq:yk+1-yk*start I2} and \eqref{eq:yk+1-yk*start I3} yields
\begin{align}\label{eq:yk+1-yk*}
\E_k\|y_{k+1}^n-y_k^{n,*}\|^2
    &\leq (1-\frac{3}{2}\lambda_{n} \beta_{k,n} + 2 L_{h,n}^2 \beta_{k,n}^2)\|y_k^n-y_k^{n,*}\|^2 + 2(\sigma_n^2 \!+\!   c_n^2 \lambda_{n}^{-1}) \beta_{k,n}^2 \nonumber\\
    &\leq (1-\lambda_{n} \beta_{k,n} )\|y_k^n-y_k^{n,*}\|^2 + 2(\sigma_n^2 \!+\!   c_n^2 \lambda_{n}^{-1}) \beta_{k,n}^2,
\end{align}
where the last inequality is due to the choice of step size that satisfies $2 L_{h,n}^2 \beta_{k,n}^2 \leq \frac{\lambda_{n}}{2} \beta_{k,n}$.

\noindent\textbf{Bounding the drifting optimality gap.} For any $n \geq 1$, we have
\begin{align}\label{eq:yk+1-yk+1*}
    \|y_{k+1}^n-y_{k+1}^{n,*}\|^2 
    &= \|y_{k+1}^n-y_k^{n,*}\|^2 + 2\langle y_k^{n,*}-y_{k+1}^n,y_{k+1}^{n,*}-y_k^{n,*} \rangle + \|y_k^{n,*}-y_{k+1}^{n,*}\|^2.
\end{align}
\textbf{(1) When $n\geq 2$.}
By the mean-value theorem, for some $\hat{y}_{k+1}^{n-1}=a y_k^{n-1} + (1-a) y_{k+1}^{n-1}, a \in [0,1]$, the second term in \eqref{eq:yk+1-yk+1*} can be rewritten as
\begin{align}\label{eq:drift n>=2}
    \langle y_k^{n,*}-y_{k+1}^n,y_{k+1}^{n,*}-y_k^{n,*} \rangle 
    &= \langle y_k^{n,*}-y_{k+1}^n,\nabla y^{n,*}(\hat{y}_{k+1}^{n-1})^\top (y_{k+1}^{n-1}-y_k^{n-1}) \rangle \nonumber\\
    &= \langle y_k^{n,*}-y_{k+1}^n,\beta_{k,n-1}\nabla y^{n,*}(\hat{y}_{k+1}^{n-1})^\top h^{n-1}(y_k^{n-2},y_k^{n-1}) \rangle \nonumber\\
    &\quad+ \langle y_k^{n,*}-y_{k+1}^n,\beta_{k,n-1}\nabla y^{n,*}(\hat{y}_{k+1}^{n-1})^\top \psi_k^{n-1} \rangle.
\end{align}
The first term in the right-hand side (RHS) of \eqref{eq:drift n>=2} can be bounded as
\begin{align}\label{eq:driftI1 n>=2}
    &\langle y_k^{n,*}-y_{k+1}^n,\beta_{k,n-1}\nabla y^{n,*}(\hat{y}_{k+1}^{n-1})^\top h^{n-1}(y_k^{n-2},y_k^{n-1}) \rangle \nonumber\\
    &\leq L_{y,n}\beta_{k,n-1}\| y_k^{n,*}-y_{k+1}^n\| \| h^{n-1}(y_k^{n-2},y_k^{n-1})\| \nonumber\\
    &\leq L_{y,n}L_{h,n-1}\beta_{k,n-1}\| y_k^{n,*}-y_{k+1}^n\|\|y_k^{n-1}-y_k^{n-1,*}\| \nonumber\\
    &\leq \frac{2 L_{y,n}^2 L_{h,n-1}^2}{\lambda_{n-1}}\beta_{k,n-1}\| y_k^{n,*}-y_{k+1}^n\|^2 + \frac{\lambda_{n-1}}{8} \beta_{k,n-1}\|y_k^{n-1}-y_k^{n-1,*}\|^2
\end{align}
where the second inequality follows from
\begin{align}\label{eq:hn-1}
    \| h^{n-1}(y_k^{n-2},y_k^{n-1})\| 
    &= \| h^{n-1}(y_k^{n-2},y_k^{n-1})- h^{n-1}(y_k^{n-2},y_k^{n-1,*})\| \nonumber\\
    &\leq L_{h,n-1} \|y_k^{n-1}-y_k^{n-1,*}\|.
\end{align}
The second term in the RHS of \eqref{eq:drift n>=2} can be further decomposed into
\begin{align}\label{eq:driftI2start n>=2}
    &\langle y_k^{n,*}-y_{k+1}^n,\beta_{k,n-1}\nabla y^{n,*}(\hat{y}_{k+1}^{n-1})^\top \psi_k^{n-1} \rangle \nonumber\\
    &= \langle y_k^{n,*}-y_{k+1}^n,\beta_{k,n-1}(\nabla y^{n,*}(\hat{y}_{k+1}^{n-1})-\nabla y^{n,*}(y_k^{n-1}))^\top \psi_k^{n-1} \rangle \nonumber\\
    &\quad+ \langle y_k^{n,*}-y_{k+1}^n,\beta_{k,n-1}\nabla y^{n,*}(y_k^{n-1})^\top \psi_k^{n-1} \rangle.
\end{align}

Taking expectation on the first term in the RHS of \eqref{eq:driftI2start n>=2} leads to
\begin{align} \label{eq:driftI2 n>=2}
    &\E_k\langle y_k^{n,*}-y_{k+1}^n,\beta_{k,n-1}\big(\nabla y^{n,*}(\hat{y}_{k+1}^{n-1})-\nabla y^{n,*}(y_k^{n-1})\big)^\top \psi_k^{n-1} \rangle  \nonumber\\
    &\leq L_{y',n}\beta_{k,n-1}\E_k\big[\|y_k^{n,*}-y_{k+1}^n\|\|\hat{y}_{k+1}^{n-1}-y_k^{n-1}\|\| \psi_k^{n-1}\| \big] \nonumber\\
    &\stackrel{(a)}{\leq} L_{y',n}\beta_{k,n-1}\E_k\big[\|y_k^{n,*}-y_{k+1}^n\|\|y_{k+1}^{n-1}-y_k^{n-1}\|\| \psi_k^{n-1}\| \big] \nonumber\\
    &\stackrel{(b)}{\leq} L_{y',n}\beta_{k,n-1}^2 \Big(\E_k\big[\|y_k^{n,*}-y_{k+1}^n\|\|h^{n-1}(y_k^{n-2},y_k^{n-1})\|\| \psi_k^{n-1}\|\big] +  \E_k\big[\|y_k^{n,*}-y_{k+1}^n\| \|\psi_k^{n-1}\|^2\big] \Big) \nonumber\\
    &= L_{y',n}\beta_{k,n-1}^2 \Big(\E_k\big[\|y_k^{n,*}-y_{k+1}^n\|\|h^{n-1}(y_k^{n-2},y_k^{n-1})\|\E[\| \psi_k^{n-1}\||\mathcal{F}_k^{n}]\big] +  \E_k\big[\|y_k^{n,*}-y_{k+1}^n\| \|\psi_k^{n-1}\|^2\big] \Big) \nonumber\\
    &\stackrel{(c)}{\leq}  L_{y',n}\beta_{k,n-1}^2 \Big(\sigma_{n-1}\E_k\big[\|y_k^{n,*}-y_{k+1}^n\|\|h^{n-1}(y_k^{n-2},y_k^{n-1})\|\big] +  \sigma_{n-1}^2\E_k\|y_k^{n,*}-y_{k+1}^n\| \Big) \nonumber\\
    &\!\leq\! L_{y',n}\beta_{k,n\!-\!1}^2 \Big(\sigma_{n\!-\!1}\E_k\big[\|y_k^{n,*}\!-\!y_{k+1}^n\|\|h^{n\!-\!1}(y_k^{n\!-\!2},y_k^{n\!-\!1})\|\big] +  \frac{\sigma_{n\!-\!1}^2}{2}\E_k\|y_k^{n,*}\!-\!y_{k+1}^n\|^2 + \frac{\sigma_{n\!-\!1}^2}{2} \Big)\!\!\! \nonumber\\
    &\stackrel{(d)}{\leq} L_{y',n}\sigma_{n\!-\!1}\beta_{k,n\!-\!1}^2 \Big(\frac{L_{h,n\!-\!1} \!+\!\sigma_{n\!-\!1}}{2}\E_k\|y_{k+1}^n\!-\!y_k^{n,*}\|^2 + \frac{ L_{h,n\!-\!1}}{2}\|y_k^{n\!-\!1}\!-\!y_k^{n\!-\!1,*}\|^2 + \frac{\sigma_{n\!-\!1}}{2} \Big),
\end{align}
where (a) is due to
\begin{align}
    \|\hat{y}_{k+1}^{n-1}-y_k^{n-1}\| = (1-a)\|y_k^{n-1}-y_{k+1}^{n-1}\| \leq \|y_k^{n-1}-y_{k+1}^{n-1}\|,
\end{align}
then (b) is due to
\begin{align}
    \|y_{k+1}^{n-1}-y_k^{n-1}\| \leq \beta_{k}^{n-1} \big(\|h^{n-1}(y_k^{n-2},y_k^{n-1})\|+\|\psi_k^{n-1}\| \big)
\end{align}
and (c) follows from Assumption \ref{assumption:noise} and Jensen's inequality:
\begin{align}\label{eq:standarddev}
    \E[\|\psi_k^n\|] = \E[\sqrt{\|\psi_k^n\|^2}] \leq \sqrt{\E\|\psi_k^n\|^2} \leq \sigma_{n},
\end{align}
the (d) follows from \eqref{eq:hn-1} and one-step Young's inequality:
\begin{align}
    \|y_k^{n,*}\!-\!y_{k+1}^n\|\|h^{n\!-\!1}(y_k^{n\!-\!2},y_k^{n\!-\!1})\| &\stackrel{\eqref{eq:hn-1}}{\leq} L_{h,n-1}\|y_k^{n,*}\!-\!y_{k+1}^n\|\|y_k^{n-1}-y_k^{n-1,*}\| \nonumber\\
    &\leq \frac{L_{h,n-1}}{2}\|y_k^{n,*}\!-\!y_{k+1}^n\|^2  + \frac{L_{h,n-1}}{2}\|y_k^{n-1}-y_k^{n-1,*}\|^2.
\end{align}

The second term in \eqref{eq:driftI2start n>=2} can be bounded as
\begin{align}\label{eq:driftI3 n>=2}
\E_k\langle y_k^{n,*}-y_{k+1}^n,\beta_{k-1,n}\nabla y^{n,*}(y_k^{n-1})^\top \psi_k^{n-1} \rangle 
    &= \E_k\big[\langle y_k^{n,*}-y_{k+1}^n,\beta_{k,n-1}\nabla y^{n,*}(x_k)^\top \E[\psi_k^{n-1} | \mathcal{F}_k^n] \rangle  \big] \nonumber\\
    &\leq L_{y,n}\beta_{k,n-1} \E_k [ \|y_k^{n,*}-y_{k+1}^n\|\|\E[\psi_k^{n-1} | \mathcal{F}_k^n]\| ] \nonumber\\
    &\stackrel{(a)}{\leq} \frac{L_{y,n} c_{n-1}}{2} \beta_{k,n-1} \big(\E_k\|y_k^{n,*}-y_{k+1}^n\|^2 + \beta_{k,n-1} \big)
\end{align}
where (a) follows from Assumption \ref{assumption:noise}.

Collecting and substituting the upper bounds in \eqref{eq:driftI1 n>=2}, \eqref{eq:driftI2 n>=2} and \eqref{eq:driftI3 n>=2} into \eqref{eq:drift n>=2} yields
\begin{align}\label{eq:yk+1n-yk+1n*I2 n>=2}
    &\E_k\langle y_k^{n,*}-y_{k+1}^n,y_{k+1}^{n,*}-y_k^{n,*} \rangle \nonumber\\
    &\leq \big(\big(\frac{L_{y,n} c_{n-1}}{2} + \frac{2 L_{y,n}^2 L_{h,n-1}^2}{\lambda_{n-1}} \big) \beta_{k,n-1} + L_{y',n}\sigma_{n-1}\frac{L_{h,n-1} \!+\!\sigma_{n-1}}{2}\beta_{k,n-1}^2\big)\E_k\|y_{k+1}^n-y_k^{n,*}\|^2 \nonumber\\
    &\quad\!+\!\big(\frac{\lambda_{n-1}}{8}\beta_{k,n-1}\!+\!\frac{L_{y',n}\sigma_{n-1} L_{h,n-1}}{2}\beta_{k,n-1}^2\big)\|y_k^{n-1}\!-\!y_k^{n-1,*}\|^2 \!+\! \frac{L_{y',n}\sigma_{n-1}^2\!+\!L_{y,n}c_{n-1}}{2}\beta_{k,n-1}^2.
\end{align}
The last term in \eqref{eq:yk+1-yk+1*} can be bounded as
\begin{align}\label{eq:yk+1-yk+1*last n>=2}
    \E_k\|y_k^{n,*}-y_{k+1}^{n,*}\|^2 &\leq L_{y,n}^2 \beta_{k,n-1}^2 \E_k\|h^{n-1}(y_k^{n-2}, y_k^{n-1})+\psi_k^{n-1}\|^2 \nonumber\\
    &\leq  2 L_{y,n}^2 \beta_{k,n-1}^2 \|h^{n-1}(y_k^{n-2}, y_k^{n-1})\|^2 + 2 L_{y,n}^2 \sigma_{n-1}^2 \beta_{k,n-1}^2 \nonumber\\
    &\stackrel{\eqref{eq:hn-1}}{\leq}  2 L_{y,n}^2 L_{h,n-1}^2 \beta_{k,n-1}^2 \|y_k^{n-1}-y_k^{n-1,*}\|^2 + 2 L_{y,n}^2\sigma_{n-1}^2 \beta_{k,n-1}^2.
\end{align}
Substituting the upper bounds in \eqref{eq:yk+1n-yk+1n*I2 n>=2} and \eqref{eq:yk+1-yk+1*last n>=2} into \eqref{eq:yk+1-yk+1*} yields (for $2\leq\! n \leq N$)
\begin{align}\label{eq:yk+1-yk+1*final n>=2}
    &\E_k\|y_{k+1}^n-y_{k+1}^{n,*}\|^2 \nonumber\\
    &\leq \big(1+\big(L_{y,n} c_{n-1} + \frac{4 L_{y,n}^2 L_{h,n-1}^2}{\lambda_{n-1}} \big) \beta_{k,n-1} + L_{y',n}\sigma_{n-1} (L_{h,n-1} \!+\!\sigma_{n-1})\beta_{k,n-1}^2\big)\E_k\|y_{k+1}^n-y_k^{n,*}\|^2 \nonumber\\
    &\quad\!+\!\frac{\lambda_{n-1}}{2}\beta_{k,n-1} \|y_k^{n-1}\!-\!y_k^{n-1,*}\|^2 \!+\! (L_{y',n}\sigma_{n-1}^2 \!+\! L_{y,n}c_{n-1} \!+\! 2 L_{y,n}^2\sigma_{n-1}^2)\beta_{k,n-1}^2 
\end{align}
where we have used the following condition of the step size to simplify the inequality:
\begin{align}
     (L_{y',n}\sigma_{n-1} L_{h,n-1}+2  L_{y,n}^2 L_{h,n-1}^2)\beta_{k,n-1}^2 \leq \frac{\lambda_{n-1}}{4}\beta_{k,n-1},  ~~~2 \leq n \leq N.
\end{align}
\textbf{(2) When $n=1$.}
The update of $y_k^1$ is correlated with its upper level variable $x_k$ instead of $y_k^{n-1}$ when $n\!\geq\!2$. And since the update of $x_k$ depends on all variables while the update of $y_k^{n-1}~(n\!\geq\!2)$ only depends on $y_k^{n-2}$, the analysis of $y_k^1$ is different from that of $y_k^n~(n\!>\!2)$. The difference therefore lies in analyzing \eqref{eq:yk+1-yk+1*}, which captures the dependence of lower level variable to its upper level variable.

By the mean-value theorem, for some $\hat{x}_{k+1}=a x_k + (1-a) x_{k+1}, a \in [0,1]$, the second term in \eqref{eq:yk+1-yk+1*} can be rewritten as
\begin{align}\label{eq:drift}
    \langle y_k^{1,*}-y_{k+1}^1,y_{k+1}^{1,*}-y_k^{1,*} \rangle 
    &= \langle y_k^{1,*}-y_{k+1}^1,\nabla y^{1,*}(\hat{x}_{k+1})^\top (x_{k+1}-x_k) \rangle \nonumber\\
    &= \langle y_k^{1,*}-y_{k+1}^1,\alpha_k\nabla y^{1,*}(\hat{x}_{k+1})^\top v(x_k,y_k^{1:N}) \rangle \nonumber\\
    &\quad+ \langle y_k^{1,*}-y_{k+1}^1,\alpha_k\nabla y^{1,*}(\hat{x}_{k+1})^\top \xi_k \rangle.
\end{align}

The first term in the RHS of \eqref{eq:drift} can be bounded as
\begin{align}
\label{eq:driftI1 1} 
    &\langle y_k^{1,*}-y_{k+1}^1,\alpha_k\nabla y^{1,*}(\hat{x}_{k+1})^\top v(x_k,y_k^{1:N}) \rangle \nonumber\\
    &\leq L_{y,1} \alpha_k\|y_k^{1,*}-y_{k+1}^1\|\| v(x_k,y_k^{1:N})\| \\
    &\stackrel{(a)}{\leq} L_{y,1}  \alpha_k \big(L_{v,y}\|y_k^{1,*}-y_{k+1}^1\|\sum_{n=1}^N L_y (n)\|y_k^n - y_k^{n,*}\| + L_v\|y_k^{1,*}-y_{k+1}^1\|\|x_k - x^*\| \big) \nonumber\\
    &\stackrel{(b)}{\leq} L_{y,1}  \alpha_k \big(\frac{L_{v,y}}{2}\|y_k^{1,*}\!-\!y_{k\!+\!1}^1\|^2 \!+\! \frac{L_{v,y}N}{2} \sum_{n=1}^N L^2_y (n)\|y_k^n \!-\! y_k^{n,*}\|^2 \!+\! \frac{L_{y,1} L_v^2}{\lambda_0}\|y_k^{1,*}\!-\!y_{k\!+\!1}^1\|^2 \!+\! \frac{\lambda_0}{4 L_{y,1}}\|x_k \!-\! x^*\|^2 \big) \nonumber\\
    \label{eq:driftI1 2}
    &\!=\! (\frac{L_{y,1} L_{v,y}}{2} \!+\! \frac{L_{y,1}^2 L_v^2}{\lambda_0}) \alpha_k \|y_{k\!+\!1}^1\!-\!y_k^{1,*}\|^2 \!+\! \frac{L_{y,1} L_{v,y} N}{2}\alpha_k \sum_{n=1}^N L^2_y (n)\|y_k^n \!-\! y_k^{n,*}\|^2 \!+\! \frac{\lambda_0}{4}\alpha_k\| x_k \!-\! x^*\|^2\!
\end{align}
and (a) follows from
\begin{align}\label{eq:vk}
    \| v(x_k,y_k^{1:N})\| 
    &= \| v(x_k,y_k^{1:N}) - v(x_k)+v(x_k) - \underbrace{v(x^*)}_{=0}\| \nonumber\\
    &\leq \| v(x_k,y_k^{1:N}) - v(x_k)\| + L_v \| x_k -x^*\| \nonumber\\
    &\leq L_{v,y} \sum_{n=1}^N L_y (n)\|y_k^n - y_k^{n,*}\| + L_v\| x_k -x^*\|,
\end{align}
where the first inequality follows from Assumption \ref{assumption:vglip} and the last inequality follows from Lemma \ref{lemma:yk+1n-yn*dots}; 
and (b) follows from Young's inequality:
\begin{subequations}
\begin{align}
    \|y_k^{1,*}-y_{k+1}^1\|\sum_{n=1}^N L_y (n)\|y_k^n - y_k^{n,*}\| &\leq \frac{1}{2}\|y_k^{1,*}-y_{k+1}^1\|^2 + \frac{1}{2}\big(\sum_{n=1}^N L_y (n)\|y_k^n - y_k^{n,*}\|\big)^2 \nonumber\\
    &\leq \frac{1}{2}\|y_k^{1,*}-y_{k+1}^1\|^2 + \sum_{n=1}^N\frac{N }{2} L^2_y (n)\|y_k^n - y_k^{n,*}\|^2,
\end{align}
and
\begin{align}
    L_v\|y_k^{1,*}-y_{k+1}^1\|\|x_k - x^*\| \leq \frac{L_{y,1} L_v^2}{\lambda_0}\|y_k^{1,*}\!-\!y_{k\!+\!1}^1\|^2 \!+\! \frac{\lambda_0}{4 L_{y,1}}\|x_k \!-\! x^*\|^2.
\end{align}
\end{subequations}

The second term in the RHS of \eqref{eq:drift} can be further decomposed as
\begin{align}\label{eq:driftI2start}
    &\E_k\langle y_k^{1,*}-y_{k+1}^1,\alpha_k\nabla y^{1,*}(\hat{x}_{k+1})^\top \xi_k \rangle \nonumber\\
    &= \E_k\langle y_k^{1,*}-y_{k+1}^1,\alpha_k\big(\nabla y^{1,*}(\hat{x}_{k+1})-\nabla y^{1,*}(x_k)\big)^\top \xi_k \rangle + \E_k\langle y_k^{1,*}-y_{k+1}^1,\alpha_k\nabla y^{1,*}(x_k)^\top \xi_k \rangle.
\end{align}
The first term in the RHS of \eqref{eq:driftI2start} can be bounded similarly to \eqref{eq:driftI2 n>=2}, with the upper level update term $\|x_{k+1}-x_k\|$ in place of $\|y_{k+1}^{n-1}-y_k^{n-1}\|~(n\!>\!2)$, that is
\begin{align}
    &\E_k\langle y_k^{1,*}-y_{k+1}^1,\alpha_k\big(\nabla y^{1,*}(\hat{x}_{k+1})-\nabla y^{1,*}(x_k)\big)^\top \xi_k \rangle \nonumber\\
    &=\E_k\langle y_k^{1,*}-y_{k+1}^1,\alpha_k\big(\nabla y^{1,*}(\hat{x}_{k+1})-\nabla y^{1,*}(x_k)\big)^\top \xi_k \rangle   \nonumber\\
    &\leq L_{y',1}\alpha_k\E_k\big[\|y_k^{1,*}-y_{k+1}^1\|\|\hat{x}_{k+1}-x_k\|\| \xi_k\| \big] \nonumber\\
    &\leq L_{y',1}\alpha_k\E_k\big[\|y_k^{1,*}-y_{k+1}^1\|\|x_{k+1}-x_k\|\| \xi_k\| \big] \nonumber\\
    &\leq L_{y',1}\alpha_k^2 \Big(\E_k\big[\|y_k^{1,*}-y_{k+1}^1\|\|v(x_k,y_k^{1:N})\|\| \xi_k\|\big] +  \E_k\big[\|y_k^{1,*}-y_{k+1}^1\| \|\xi_k\|^2\big] \Big)\nonumber\\
    &= L_{y',1}\alpha_k^2 \Big(\E_k\big[\|y_k^{1,*}-y_{k+1}^1\|\|v(x_k,y_k^{1:N})\|\E[\| \xi_k\||\mathcal{F}_k^1]\big] +  \E_k\big[\|y_k^{1,*}-y_{k+1}^1\| \E[\| \xi_k\|^2|\mathcal{F}_k^1]\big] \Big)\nonumber\\
    &\leq  L_{y',1}\alpha_k^2 \Big(\sigma_0\E_k\big[\|y_k^{1,*}-y_{k+1}^1\|\|v(x_k,y_k^{1:N})\|\big] + \sigma_0^2\E_k\|y_k^{1,*}-y_{k+1}^1\| \Big) \nonumber\\
    \label{eq:driftI2-}
    &\leq L_{y',1}\alpha_k^2 \Big(\sigma_0\E_k\big[\|y_k^{1,*}-y_{k+1}^1\|\|v(x_k,y_k^{1:N})\|\big] +  \frac{\sigma_0^2}{2}\E_k\|y_k^{1,*}-y_{k+1}^1\|^2 + \frac{\sigma_0^2}{2} \Big) \\
    \label{eq:driftI2}
    &\leq L_{y',1}\sigma_0\alpha_k^2 \Big(\frac{\sigma_0\!+\! L_{v,y} \!+\!L_v }{2}\E_k\|y_{k\!+\!1}^1-y_k^{1,*}\|^2 \!+\! \frac{ L_{v,y} N}{2}\sum_{n=1}^N \|y_k^n - y_k^{n,*}\|^2\!+\!\frac{L_v}{2}\|x_k - x^*\|^2 \!+\! \frac{\sigma_0}{2}\Big),\!
\end{align}
where the fourth inequality follows from Assumption \ref{assumption:noise} and the last inequality follows from similar derivations of the upper bound of $\|y_k^{1,*}-y_{k+1}^1\|\|v(x_k,y_k^{1:N})\|$ shown in \eqref{eq:driftI1 1}--\eqref{eq:driftI1 2}.

The second term in the RHS of \eqref{eq:driftI2start} can be bounded as
\begin{align}\label{eq:driftI3}
    \E_k\langle y_k^{1,*}-y_{k+1}^1,\alpha_k\nabla y^{1,*}(x_k)^\top \xi_k \rangle
    &= \E_k\big[\langle y_k^{1,*}-y_{k+1}^1,\alpha_k\nabla y^{1,*}(x_k)^\top \E[\xi_k | \mathcal{F}_k^1] \rangle  \big] \nonumber\\
    &\leq L_{y,1}\alpha_k \E_k [ \|y_k^{1,*}-y_{k+1}^1\|\|\E[\xi_k | \mathcal{F}_k^1]\| ] \nonumber\\
    &\leq \frac{L_{y,1} c_0}{2} \alpha_k \big(\E_k\|y_k^{1,*}-y_{k+1}^1\|^2 + \alpha_k \big).
\end{align}

Substituting the upper bounds in \eqref{eq:driftI1 2}, \eqref{eq:driftI2} and \eqref{eq:driftI3} into \eqref{eq:drift} yields
\begin{align}\label{eq:yk+1-yk+1*I2}
     &\E_k\langle y_k^{1,*}-y_{k+1}^1,y_{k+1}^{1,*}-y_k^{1,*} \rangle \nonumber\\
     &\leq \Big(\big(\frac{L_{y,1} L_{v,y}}{2} + \frac{L_{y,1}^2 L_v^2}{\lambda_0} +  \frac{L_{y,1} c_0}{2}\big)\alpha_k +  L_{y',1} \frac{\sigma_0^2+ (L_{v,y}\!+\!L_v) \sigma_0}{2}\alpha_k^2\Big)\E_k\|y_{k+1}^1-y_k^{1,*}\|^2 \nonumber\\
     &\quad+ \big(\frac{L_{y,1} L_{v,y} N}{2}\alpha_k + \frac{L_{y',1}\sigma_0 L_{v,y} N}{2}\alpha_k^2 \big)N \sum_{n=1}^N L^2_y (n)\|y_k^n - y_k^{n,*}\|^2 \nonumber\\
     &\quad+ \big( \frac{\lambda_0}{4}\alpha_k + \frac{L_{y',1}L_v\sigma_0}{2}\alpha_k^2 \big)\|x_k - x^*\|^2 + \frac{L_{y',1} \sigma_0^2 + L_{y,1} c_0}{2}\alpha_k^2.
\end{align}
The last term in \eqref{eq:yk+1-yk+1*} can be bounded as
\begin{align}\label{eq:yk+1-yk+1*last}
    &\E_k\|y_k^{1,*}-y_{k+1}^1\|^2 \nonumber\\
    &\leq L_{y,1}^2 \alpha_k^2 \E_k\|v(x_k,y_k^{1:N})+\xi_k\|^2 \nonumber\\
    &\leq  2 L_{y,1}^2 \alpha_k^2 \|v(x_k,y_k^{1:N})\|^2 + 2 L_{y,1}^2 \sigma_0^2 \alpha_k^2 \nonumber\\
    &\stackrel{\eqref{eq:vk}}{\leq} 4  L_{y,1}^2 \alpha_k^2 \big(L_{v,y}^2 N\sum_{n=1}^N L_y(n)^2\|y_k^n-y_k^{n,*}\|^2 \!+\! L_v^2\|x_k-x^*\|^2\big) \!+\! 2 L_{y,1}^2\sigma_0^2 \alpha_k^2.
\end{align}
Substituting the upper bounds in \eqref{eq:yk+1-yk+1*I2} and \eqref{eq:yk+1-yk+1*last} into \eqref{eq:yk+1-yk+1*} yields
\begin{align}\label{eq:yk+1-yk+1*final}
    &\E_k\|y_{k+1}^1-y_{k+1}^{1,*}\|^2 \nonumber\\
     &\leq \big(1 + L_{y,1} \big(L_{v,y} \!+\!2 L_{y,1}L_v^2 \lambda_0^{-1} \!+\! c_0 \big) \alpha_k + L_{y',1}\sigma_0 (L_{v,y}\!+\!L_v\!+\!\sigma_0 )\alpha_k^2\big)\E_k\|y_{k+1}^1-y_k^{1,*}\|^2 \nonumber\\
      &\quad+ \big(L_{y,1}L_{v,y} N \alpha_k + (L_{y',1}\sigma_0 L_{v,y} + 4 L_{y,1}^2 L_{v,y}^2) N\alpha_k^2\big)\sum_{n=1}^N L^2_y(n)\|y_k^n - y_k^{n,*}\|^2 \nonumber\\
      &\quad+ \big(\frac{\lambda_0}{2}\alpha_k + (L_{y',1}L_v\sigma_0 \!+\! 4 L_{y,1}^2 L_v^2)\alpha_k^2\big)\| x_k - x^*\|^2 + (L_{y',1}\sigma_0^2 \!+\! L_{y,1}c_0 \!+\! 2 L_{y,1}^2\sigma_0^2)\alpha_k^2.
\end{align}
This completes the analysis of lower-level sequences.

\subsection{Analysis of the main sequence}
Recall that we defined the shorthand notations $y_k^{n,*} = y^{n,*}(y^{n-1}_k)$ with $y_k^{1,*} = y^{1,*}(x_k)$; $y_k^{1:N}=(y_k^1,y_k^2,\ldots,y_k^N)$. For convenience, we write $\E[\cdot|\mathcal{F}_k]$ as $\E_k[\cdot]$.
In this section, we will analyze the main sequence and then establish the convergence rate.

 First we have
\begin{align}\label{eq:xk+1-x*}
    &\E_k\|x_{k+1}-x^*\|^2\nonumber\\
    &= \|x_k-x^*\|^2 + 2\alpha_k \E_k\langle x_k-x^*,v(x_k,y_k^{1:N}) + \xi_k \rangle + \E_k\|x_{k+1}-x_k\|^2 \nonumber\\
    &= \|x_k-x^*\|^2 + 2\alpha_k \E_k\langle x_k-x^*,v(x_k,y_k^{1:N})-v(x_k) \rangle + 2\alpha_k \langle x_k-x^*,v(x_k) \rangle \nonumber\\
    &\quad + 2\alpha_k \langle x_k-x^*, \E_k[\xi_k] \rangle+\alpha_k^2\E_k\|v(x_k,y_k^{1:N})+\xi_k\|^2.
\end{align}
By Lemma \ref{lemma:yk+1n-yn*dots}, the second term in \eqref{eq:xk+1-x*} can be bounded as
\begin{align}\label{eq:xk+1-x*I2}
    \langle x_k-x^*,v(x_k,y_k^{1:N}) \!-\! v(x_k) \rangle &\leq L_{v,y} \|x_k-x^*\|\sum_{n=1}^N L_y(n)\|y_k^n-y_k^{n,*}\| \nonumber\\
    &\leq \frac{\lambda_0}{8}\|x_k-x^*\|^2 + \frac{2 L_{v,y}^2 N}{\lambda_0}\sum_{n=1}^N L_y(n)^2\|y_k^n-y_k^{n,*}\|^2.
\end{align}
By the strong monotonicity of $v(x,y^*(x))$ in Assumption \ref{assumption:strongly v}, the third term in \eqref{eq:xk+1-x*} can be bounded as
\begin{align}\label{eq:xk+1-x*I3}
    \langle x_k-x^*,v(x_k) \rangle \leq -\lambda_0 \|x_k-x^*\|^2.
\end{align}
Using Assumption \ref{assumption:noise}, the fourth term in \eqref{eq:xk+1-x*} can be bounded as
\begin{align}\label{eq:xk+1-x*I4}
    \langle x_k-x^*, \E_k[\xi_k] \rangle \leq \frac{\lambda_0}{8}\|x_k - x^*\|^2 + \frac{2 c_0^2}{\lambda_0} \alpha_k.
\end{align}
The last term in \eqref{eq:xk+1-x*} can be bounded as
\begin{align}\label{eq:xk+1-x*last}
    \E_k\|v(x_k,y_k^{1:N})+\xi_k\|^2 
    &\leq 2\|v(x_k,y_k^{1:N})\|^2 + 2\sigma_0^2 \nonumber\\ 
    &\stackrel{\eqref{eq:vk}}{\leq } 4 L_v^2\|x_k-x^*\|^2 +  4 N L_{v,y}^2 \sum_{n=1}^N L_y(n)^2\|y_k^n-y_k^{n,*}\|^2 + 2\sigma_0^2.
\end{align}
Substituting the upper bounds in \eqref{eq:xk+1-x*I2}--\eqref{eq:xk+1-x*last} into \eqref{eq:xk+1-x*} yields
\begin{align}\label{eq:xk+1-x*final}
   \E_k\|x_{k+1}-x^*\|^2 
   &\leq \big(1 \!-\!\frac{3}{2}\lambda_0\alpha_k \!+\! 4 L_v^2\alpha_k^2)\|x_k-x^*\|^2 \!+\! 4\big( \frac{L_{v,y}^2}{\lambda_0}\alpha_k \!+\! L_{v,y}^2\alpha_k^2\big)N \sum_{n=1}^N L^2_y(n)\|y_k^n\!-\!y_k^{n,*}\|^2 \nonumber\\
   &\quad+ 2 \big(\sigma_0^2 + \frac{2 c_0^2}{\lambda_0} \big) \alpha_k^2.
\end{align}
\textbf{Establishing convergence.} For brevity, we fist define the following series
\begin{align}
    &C_0(1) \coloneqq L_{y,1} \big(L_{v,y} \!+\!2 L_{y,1}L_v^2 \lambda_0^{-1} \!+\! c_0 \big),~ C_1 (1) \coloneqq L_{y',1}\sigma_0 (L_v\!+\!L_{v,y}\!+\!\sigma_0 );\nonumber\\
    &C_0(n) \coloneqq L_{y,n} c_{n-1} + \frac{4 L_{y,n}^2 L_{h,n-1}^2}{\lambda_{n-1}},~C_1(n) \coloneqq L_{y',n}\sigma_{n-1} (L_{h,n-1} \!+\!\sigma_{n-1}), 2\leq n \leq N;\nonumber\\
    &C_2(n) \coloneqq (4\frac{L_{v,y}^2}{\lambda_0} +\frac{L_{y,1} L_{v,y} }{2}) N L^2_y(n),~ C_3(n) \coloneqq (L_{v,y}^2 +\frac{L_{y',1}\sigma_0 L_{v,y} }{2}) N L^2_y(n), \forall n.
\end{align}
Define a Lyapunov function $\mathcal{J}_k \coloneqq \|x_k-x^*\|^2 + \sum_{n=1}^N\|y_k^n-y_k^{n,*}\|^2$. Then we have
\begin{align}\label{eq:jk+1-jk-}
    \E_k[\mathcal{J}_{k+1}] - \mathcal{J}_k 
    &= \E_k\|x_{k+1}-x^*\|^2 - \|x_k-x^*\|^2 + \sum_{n=1}^N\|y_{k+1}^n-y_{k+1}^{n,*}\|^2 -\|y_k^n-y_k^{n,*}\|^2.
\end{align}
Substituting \eqref{eq:yk+1-yk+1*final n>=2}, \eqref{eq:yk+1-yk+1*final} and \eqref{eq:xk+1-x*final} into \eqref{eq:jk+1-jk-}, and then applying \eqref{eq:yk+1-yk*} yields
\begin{align}\label{eq:jk+1-jk}
    &\E_k[\mathcal{J}_{k+1}]-\mathcal{J}_k \nonumber\\
    &\leq \big(-\lambda_0\alpha_k +  \big(L_{y',1}L_v\sigma_0 +4  L_{y,1}^2 L_v^2+ 4 L_v^2 \big) \alpha_k^2\big)\|x_k-x^*\|^2 \nonumber\\
    &\quad\!+\!\sum_{n=1}^{N\!-\!1} \big((1\!+\!C_0(n)\beta_{k,n\!-\!1} \!+\!  C_1(n)\beta_{k,n\!-\!1}^2 )(1\!-\!\lambda_n \beta_{k,n}) \!-\!1 \!+\! \frac{\lambda_n}{2}\beta_{k,n} \!+\! C_2(n)\alpha_k \!+\! C_3(n)\alpha_k^2\big)\|y_k^n\!-\!y_k^{n,*}\|^2\nonumber\\
    &\quad\!+\! \big((1\!+\!C_0(N)\beta_{k,N\!-\!1} \!+\!  C_1(N)\beta_{k,N\!-\!1}^2 )(1\!-\!\lambda_N \beta_{k,N}) \!-\!1  \!+\! C_2(N)\alpha_k \!+\! C_3(N)\alpha_k^2\big)\|y_k^N\!-\!y_k^{N,*}\|^2 \nonumber\\
    &\quad+  \Theta(\alpha_k^2) + \Theta\big(\sum_{n=1}^N (1+\beta_{k,n-1}+\beta_{k,n-1}^2)\beta_{k,n}^2\big),
\end{align}
where we define $\beta_{k,0}\coloneqq\alpha_k$ to simplify the result. As a clarification, the second term in the last inequality disappears when $N\leq1$.
Let the step sizes satisfy
\begin{align}
\label{eq:step size a}
     &-\lambda_0\alpha_k +  \big(L_{y',1}L_v\sigma_0 +4  L_{y,1}^2 L_v^2+ 4 L_v^2 \big) \alpha_k^2 \leq -\frac{\lambda_0}{2}\alpha_k, \\
     \label{eq:step size b}
     &(1\!+\!C_0(n)\beta_{k,n\!-\!1} \!+\!  C_1(n)\beta_{k,n\!-\!1}^2 )(1\!-\!\lambda_n \beta_{k,n}) \!-\!1 \!+\! \frac{\lambda_n}{2}\beta_{k,n} \!+\! C_2(n)\alpha_k \!+\! C_3(n)\alpha_k^2 \leq \!-\frac{\lambda_0}{2}\alpha_k, 1\!\leq\! n\!\leq\! N-1, \\
     \label{eq:step size c}
     &(1\!+\!C_0(N)\beta_{k,n\!-\!1} \!+\!  C_1(N)\beta_{k,n\!-\!1}^2 )(1\!-\!\lambda_N \beta_{k,N}) \!-\!1  \!+\! C_2(N)\alpha_k \!+\! C_3(N)\alpha_k^2 \leq \!-\frac{\lambda_0}{2}\alpha_k, 
\end{align}
Note that \eqref{eq:step size a} always admits solution for small enough $\alpha_1$. Given $\beta_{k,N}$, applying Lemma \ref{lemma:step sizes} for $n=N,\ldots,1$ to \eqref{eq:step size c} and \eqref{eq:step size b} implies that there exist solutions for $\beta_{k,n}(\forall n)$.

Then by \eqref{eq:step size a}--\eqref{eq:step size c}, we have from \eqref{eq:jk+1-jk} that
\begin{align}\label{eq:jk+1-jkfinal}
    \E_k[\mathcal{J}_{k+1}] \leq \big(1-\frac{\lambda_0}{2}\alpha_k \big)\mathcal{J}_{k} +  \Theta(\alpha_k^2) + \Theta\big(\sum_{n=1}^N (1+\beta_{k,n-1}+\beta_{k,n-1}^2)\beta_{k,n}^2\big).
\end{align}
Note that \eqref{eq:jk+1-jkfinal} implies a finite-time convergence rate of $\frac{1}{k}$ with the choice of step size. 
Applying Robbins-Siegmund's theorem stated in Lemma \ref{lemma:robbins-siegmund} to \eqref{eq:jk+1-jkfinal} gives $\sum_{k=1}^\infty \alpha_k \mathcal{J}_k < \infty$ and $\lim_{k\xrightarrow[]{}\infty}\mathcal{J}_k < \infty$ almost surely, which along with the fact that $\sum_{k=1}^\infty \alpha_k = \infty$ implies $\lim_{k\xrightarrow[]{}\infty}\mathcal{J}_k=0$, i.e. for any $n\in [N]$
\begin{align}
    \lim_{k\xrightarrow[]{}\infty}\|x_k- x^*\|^2=0, ~~~~~~\lim_{k\xrightarrow[]{}\infty}\|y_k^n - y_k^{n,*}\|^2 = 0, ~a.s.
\end{align}
Finally, as a direct result of Lemma \ref{lemma:error metic equiv}, we can directly obtain the same convergence theorem for the alternative error metric $\|x_k-x^*\|^2+\sum_{n=1}^N\|y_k^n-y^{n,*}\|^2$. This completes the proof.

\section{Proof of Theorem \ref{theorem:non-strongly-monotone-2}}\label{section:non strongly monotone appendix}
\subsection{Analysis of the lower-level sequences}
In this section, we provide a bound of the lower-level optimality gaps. Recall that we defined the shorthand notations $y_k^{n,*} = y^{n,*}(y^{n-1}_k)$ with $y_k^{1,*} = y^{1,*}(x_k)$; $y_k^{1:N}=(y_k^1,y_k^2,\ldots,y_k^N)$. For convenience, we write $\E[\cdot|\mathcal{F}_k]$ as $\E_k[\cdot]$.

It follows from \eqref{eq:yk+1-yk*} that
\begin{align}\label{eq:yk+1-yk*c}
\E_k\|y_{k+1}^n-y_k^{n,*}\|^2
    &\leq (1-\lambda_{n} \beta_{k,n} )\|y_k^n-y_k^{n,*}\|^2 + 2(\sigma_n^2 \!+\!   c_n^2 \lambda_{n}^{-1}) \beta_{k,n}^2.
\end{align}
\textbf{Bounding the drifting optimality gap.} For any $n \geq 1$, we have
\begin{align}\label{eq:yk+1-yk+1*c}
    \|y_{k+1}^n-y_{k+1}^{n,*}\|^2 
    &= \|y_{k+1}^n-y_k^{n,*}\|^2 + 2\langle y_k^{n,*}-y_{k+1}^n,y_{k+1}^{n,*}-y_k^{n,*} \rangle + \|y_k^{n,*}-y_{k+1}^{n,*}\|^2.
\end{align}
\textbf{(1) When $n=1$.}
By the mean-value theorem, for some $\hat{x}_{k+1}=a x_k + (1-a) x_{k+1}, a \in [0,1]$, the second term in \eqref{eq:yk+1-yk+1*c} can be rewritten as
\begin{align}\label{eq:driftc}
    \langle y_k^{1,*}-y_{k+1}^1,y_{k+1}^{1,*}-y_k^{1,*} \rangle 
    &= \langle y_k^{1,*}-y_{k+1}^1,\nabla y^{1,*}(\hat{x}_{k+1})^\top (x_{k+1}-x_k) \rangle \nonumber\\
    &= \langle y_k^{1,*}-y_{k+1}^1,\alpha_k\nabla y^{1,*}(\hat{x}_{k+1})^\top v(x_k,y_k^{1:N}) \rangle \nonumber\\
    &\quad+ \langle y_k^{1,*}-y_{k+1}^1,\alpha_k\nabla y^{1,*}(\hat{x}_{k+1})^\top \xi_k \rangle.
\end{align}
The first term in \eqref{eq:driftc} can be bounded as
\begin{align}
\label{eq:driftcI1 1} 
    &\langle y_k^{1,*}-y_{k+1}^1,\alpha_k\nabla y^{1,*}(\hat{x}_{k+1})^\top v(x_k,y_k^{1:N}) \rangle \nonumber\\
    &\leq L_{y,1} \alpha_k\|y_k^{1,*}-y_{k+1}^1\|\| v(x_k,y_k^{1:N})\| \\
    &\leq L_{y,1} \alpha_k \Big(L_{v,y}\|y_k^{1,*}-y_{k+1}^1\|\sum_{n=1}^N L_y(n)\|y_k^n - y_k^{n,*}\| + \|y_k^{1,*}-y_{k+1}^1\|\| v(x_k)\| \Big) \nonumber\\
    &\leq L_{y,1} \alpha_k \Big(\frac{L_{v,y}}{2}\|y_k^{1,*}-y_{k+1}^1\|^2 + \frac{L_{v,y} N}{2} \sum_{n=1}^N L^2_y(n)\|y_k^n - y_k^{n,*}\|^2 \nonumber\\
    &\quad\quad\quad\quad\quad+ 2L_{y,1}\|y_k^{1,*}-y_{k+1}^1\|^2 + \frac{1}{8L_{y,1}}\| v(x_k)\|^2 \Big) \nonumber\\
    \label{eq:driftcI1 2}
    &= L_{y,1} \big(\frac{L_{v,y}}{2} \!+\! 2 L_{y,1} \big) \alpha_k \|y_{k\!+\!1}^1\!-\!y_k^{1,*}\|^2 \!+\! \frac{L_{y,1}L_{v,y} N}{2} \alpha_k\sum_{n=1}^N L^2_y(n)\|y_k^n \!-\! y_k^{n,*}\|^2 \!+\! \frac{1}{8}\alpha_k\| v(x_k)\|^2,
\end{align}
where the second inequality follows from Lemma \ref{lemma:yk+1n-yn*dots}:
\begin{align}\label{eq:vkc}
    \| v(x_k,y_k^{1:N})\| &\leq\| v(x_k,y_k^{1:N}) - v(x_k)\|+\|v(x_k)\| \nonumber\\
    &\leq L_{v,y} \sum_{n=1}^N L_y(n)\|y_k^n-y_k^{n,*}\| + \|v(x_k)\|.
\end{align}
The second term in \eqref{eq:driftc} can be further decomposed as
\begin{align}\label{eq:driftcI2start}
    &\E_k\langle y_k^{1,*}-y_{k+1}^1,\alpha_k\nabla y^{1,*}(\hat{x}_{k+1})^\top \xi_k \rangle \nonumber\\
    &= \E_k\langle y_k^{1,*}\!-\!y_{k+1}^1,\alpha_k\big(\nabla y^{1,*}(\hat{x}_{k+1})\!-\!\nabla y^{1,*}(x_k)\big)^\top \xi_k \rangle \!+\! \E_k\langle y_k^{1,*}\!-\!y_{k+1}^1,\alpha_k\nabla y^{1,*}(x_k)^\top \xi_k \rangle.\!
\end{align}
The first term in \eqref{eq:driftcI2start} can be bounded as
\begin{align}    \label{eq:driftcI2}
    &\E_k\langle y_k^{1,*}-y_{k+1}^1,\alpha_k\big(\nabla y^{1,*}(\hat{x}_{k+1})-\nabla y^{1,*}(x_k)\big)^\top \xi_k \rangle  \\
    &\stackrel{\eqref{eq:driftI2-}}{\leq}L_{y',1} \sigma_0\alpha_k^2 \Big(\E_k\big[\|y_k^{1,*}-y_{k+1}^1\|\|v(x_k,y_k^{1:N})\|\big] +  \frac{\sigma_0}{2}\E_k\|y_k^{1,*}-y_{k+1}^1\|^2 + \frac{\sigma_0}{2} \Big) \nonumber\\
    &\leq L_{y',1} \sigma_0\alpha_k^2 \Big(\frac{L_{v,y} \!\!+\!\!\sigma_0 \!\!+\!\!1}{2}\E_k\|y_{k\!+\!1}^1\!-\!y_k^{1,*}\|^2 \!+\! \frac{ L_{v,y} N}{2}\sum_{n=1}^N L^2_y(n) \|y_k^n\!-\!y_k^{n,*}\|^2 \!+\! \frac{1}{2}\|v(x_k)\|^2 \!+\! \frac{\sigma_0}{2} \Big)\nonumber
\end{align}
where the last inequality follows from similar derivations of the upper bound of $\|y_k^{1,*}-y_{k+1}^1\|\|v(x_k,y_k^{1:N})\|$ shown in \eqref{eq:driftcI1 1}--\eqref{eq:driftcI1 2}.

The second term in \eqref{eq:driftcI2start} can be bounded as 
\begin{align}\label{eq:driftcI3}
    \E_k\langle y_k^{1,*}-y_{k+1}^1,\alpha_k\nabla y^{1,*}(x_k)^\top \xi_k \rangle
    &\stackrel{\eqref{eq:driftI3}}{\leq} \frac{L_{y,1} c_0}{2} \alpha_k \big(\E_k\|y_k^{1,*}-y_{k+1}^1\|^2 + \alpha_k \big).
\end{align}
Substituting the upper bounds in \eqref{eq:driftcI1 2}, \eqref{eq:driftcI2} and \eqref{eq:driftcI3} into \eqref{eq:driftc} yields
\begin{align}\label{eq:yk+1-yk+1*cI2}
      \langle y_k^{1,*}-y_{k+1}^1,y_{k+1}^{1,*}-y_k^{1,*} \rangle 
      &\leq \big(L_{y,1} \big(\frac{L_{v,y} \!+\! c_0}{2} \!+\! 2 L_{y,1} \big) \alpha_k \!+\! L_{y',1}\sigma_0 \frac{L_{v,y}\!+\!\sigma_0 \!+\! 1}{2}\alpha_k^2\big)\|y_{k+1}^1-y_k^{1,*}\|^2 \nonumber\\
      &\quad+ \frac{1}{2}\big(L_{y,1}L_{v,y} N \alpha_k +L_{y',1}\sigma_0 L_{v,y} N\alpha_k^2\big)\sum_{n=1}^N L^2_y(n)\|y_k^n - y_k^{n,*}\|^2 \nonumber\\
      &\quad+ \big(\frac{1}{8}\alpha_k \!+\! \frac{L_{y',1}\sigma_0}{2}\alpha_k^2\big)\| v(x_k)\|^2 \!+\! \frac{L_{y',1}\sigma_0^2 \!+\! L_{y,1}c_0}{2}\alpha_k^2.
\end{align}
The last term in \eqref{eq:yk+1-yk+1*c} can be bounded as
\begin{align}\label{eq:yk+1-yk+1*clast}
    &\E_k\|y_k^{1,*}-y_{k+1}^{1,*}\|^2 \nonumber\\
    &\leq L_{y,1}^2 \alpha_k^2 \E_k\|v(x_k,y_k^1)+\xi_k\|^2 \leq  2 L_{y,1}^2 \alpha_k^2 \|v(x_k,y_k^1)\|^2 + 2 L_{y,1}^2 \sigma_0^2 \alpha_k^2 \nonumber\\
    &\stackrel{\eqref{eq:vkc}}{\leq} 4 L_{v,y}^2 L_{y,1}^2 N\alpha_k^2 \sum_{n=1}^N L^2_y(n) \|y_k^n-y_k^{n,*}\|^2 + 4 L_{y,1}^2 \alpha_k^2 \|v(x_k)\|^2 + 2 L_{y,1}^2\sigma_0^2 \alpha_k^2.
\end{align}
Substituting the upper bounds in \eqref{eq:yk+1-yk+1*cI2} and \eqref{eq:yk+1-yk+1*clast} into \eqref{eq:yk+1-yk+1*c} yields
\begin{align}\label{eq:yk+1-yk+1*cfinal}
     &\E_k\|y_{k+1}^1-y_{k+1}^{1,*}\|^2 \nonumber\\
     &\leq \big(1 + L_{y,1} \big(L_{v,y} \!+\! c_0 \!+\! 4 L_{y,1} \big) \alpha_k + L_{y',1}\sigma_0 (L_{v,y}\!+\!\sigma_0 \!+\! 1)\alpha_k^2\big)\E_k\|y_{k+1}^1-y_k^{1,*}\|^2 \nonumber\\
      &\quad+ \big(L_{y,1}L_{v,y} N \alpha_k + (L_{y',1}\sigma_0 L_{v,y} + 4 L_{v,y}^2 L_{y,1}^2) N\alpha_k^2\big)\sum_{n=1}^N L^2_y(n)\|y_k^n - y_k^{n,*}\|^2 \nonumber\\
      &\quad+ \big(\frac{1}{4}\alpha_k + (L_{y',1}\sigma_0 \!+\! 4 L_{y,1}^2)\alpha_k^2\big)\| v(x_k)\|^2 + (L_{y',1}\sigma_0^2 \!+\! L_{y,1}c_0 \!+\! 2 L_{y,1}^2\sigma_0^2)\alpha_k^2.
\end{align}

\noindent\textbf{(2) When $n\geq 2$.}  The update of $y_k^n~(n\!\geq\!2)$ has no direct dependence on $x_k$, therefore the analysis is identical to that of Theorem \ref{theorem:strongly monotone}. It directly follows from \eqref{eq:yk+1-yk+1*last n>=2} that
\begin{align}\label{eq:yk+1-yk+1*cfinal n>=2}
    &\E_k\|y_{k+1}^n-y_{k+1}^{n,*}\|^2 \nonumber\\
    &\leq \big(1+\big(L_{y,n} c_{n-1} + \frac{4 L_{y,n}^2 L_{h,n-1}^2}{\lambda_{n-1}} \big) \beta_{k,n-1} + L_{y',n}\sigma_{n-1} (L_{h,n-1} \!+\!\sigma_{n-1})\beta_{k,n-1}^2\big)\E_k\|y_{k+1}^n-y_k^{n,*}\|^2 \nonumber\\
    &\quad+\frac{\lambda_{n-1}}{2}\beta_{k,n-1} \|y_k^{n-1}-y_k^{n-1,*}\|^2 + (L_{y',n}\sigma_{n-1}^2 \!+\! L_{y,n}c_{n-1} \!+\! 2 L_{y,n}^2\sigma_{n-1}^2)\beta_{k,n-1}^2
\end{align}
where we have imposed the following condition on the step size
\begin{align}
     (L_{y',n}\sigma_{n-1} L_{h,n-1}+2  L_{y,n}^2 L_{h,n-1}^2)\beta_{k,n-1}^2 \leq \frac{\lambda_{n-1}}{4}\beta_{k,n-1},  2 \leq n \leq N.
\end{align}
This completes the analysis of the lower-level sequences.

\subsection{Analysis of the main sequence}
In this section, we provide an analysis of the main sequence update, and then establish the finite-time convergence rate. Recall the shorthand notations $y_k^{n,*} = y^{n,*}(y^{n-1}_k)$ with $y_k^{1,*} = y^{1,*}(x_k)$.

 By the $L_v$-smoothness of $F(x)$, we have
\begin{align}\label{eq:Fk+1-Fk}
    &\E_k[F(x_{k+1})]-F(x_k) \nonumber\\
    &\geq \E_k\big\langle v(x_k), x_{k+1}-x_k \big\rangle - \frac{L_v}{2}\E_k\|x_{k+1}-x_k\|^2 \nonumber\\
    &= \E_k\big\langle v(x_k), \alpha_k v(x_k,y_k^{1:N}) \big\rangle + \E_k\big\langle v(x_k), \alpha_k \xi_k \big\rangle - \frac{L_v}{2}\E_k\|x_{k+1}-x_k\|^2.
\end{align}
Define $L_y (n) \coloneqq  \sum_{i=n}^{N} L_{y,i-1} L_{y,i-2}\dots L_{y,n} $ with $L_{y,n-1} L_{y,i-2}\dots L_{y,n} \coloneqq 1$. Using Lemma \ref{lemma:yk+1n-yn*dots}, the first term in \eqref{eq:Fk+1-Fk} can be bounded as
\begin{align}\label{eq:Fk+1-FkI1}
    \big\langle v(x_k), \alpha_k v(x_k,y_k^{1:N}) \big\rangle 
    &= \big\langle v(x_k), \alpha_k( v(x_k,y_k^{1:N}) - v(x_k) ) \big\rangle + \alpha_k\|v(x_k)\|^2 \nonumber\\
    &\geq -L_{v,y} \alpha_k \Big[\big\|v(x_k)\big\| \sum_{n=1}^N L_y(n) \|y_k^n - y_k^{n,*}\|\Big] + \alpha_k\|v(x_k)\|^2 \nonumber\\
    &\geq -\frac{\alpha_k}{4}\|v(x_k)\|^2 - L_{v,y}^2 N \alpha_k \sum_{n=1}^N L^2_y(n) \|y_k^n - y_k^{n,*}\|^2+ \alpha_k\|v(x_k)\|^2 \nonumber\\
    &= \frac{3\alpha_k}{4}\|v(x_k)\|^2 - L_{v,y}^2 N  \alpha_k \sum_{n=1}^N L^2_y(n)\|y_k^n - y_k^{n,*}\|^2.
\end{align}
The second term in \eqref{eq:Fk+1-Fk} can be bounded as
\begin{align}\label{eq:Fk+1-FkI2}
    \E_k\big\langle v(x_k), \alpha_k \xi_k \big\rangle
    &= \big\langle v(x_k), \alpha_k \E_k[\xi_k ] \big\rangle \nonumber\\
    &\geq -\frac{\alpha_k}{4}\| v(x_k)\|^2 - \alpha_k \| \E_k[\xi_k ]\|^2 \nonumber\\
    &\geq -\frac{\alpha_k}{4}\| v(x_k)\|^2 - c_0^2\alpha_k^2.
\end{align}
The last term in \eqref{eq:Fk+1-Fk} can be bounded as
\begin{align}\label{eq:Fk+1-Fklast}
    \E_k\|x_{k+1}-x_k\|^2 &\leq 2\alpha_k^2 \big(\|v(x_k,y_k^{1:N})\|^2 \!+\! \E_k\|\xi_k\|^2\big) \nonumber\\
    &\stackrel{\eqref{eq:vkc}}{\leq} 4 \alpha_k^2 \|v(x_k)\|^2 \!+\! 4 L_{v,y}^2 N\alpha_k^2 \sum_{n=1}^N L^2_y(n)\|y_k^n-y_k^{n,*}\|^2 \!+\! 2 \sigma_0^2 \alpha_k^2.
\end{align}
Substituting the bounds in \eqref{eq:Fk+1-FkI1}, \ref{eq:Fk+1-FkI2} and \eqref{eq:Fk+1-Fklast} into \eqref{eq:Fk+1-Fk} yields
\begin{align}\label{eq:Fk+1-Fkfinal}
    &\E_k[F(x_{k+1})] - F(x_k) \nonumber\\
    &\geq (\frac{\alpha_k}{2} \!-\! 2 L_v\alpha_k^2) \|v(x_k)\|^2 \!-\! N(L_{v,y}^2\alpha_k \!+\! 2 L_v L_{v,y}^2 \alpha_k^2) \sum_{n=1}^N L^2_y(n)\|y_k^n-y_k^{n,*}\|^2 \!-\! (L_v\sigma_0^2\!+\! c_0^2) \alpha_k^2.
\end{align}
\textbf{Establishing convergence.} For brevity, we fist define the following series
\begin{align}
    &C_4(1) \coloneqq  L_{y,1} \big(L_{v,y} \!+\! c_0 \!+\! 4 L_{y,1} \big),~C_5 (1)\coloneqq  L_{y',1}\sigma_0 \big(L_{v,y}\!+\!\sigma_0 \!+\! 1\big);\nonumber\\
    &C_4(n) \coloneqq L_{y,n} c_{n-1} + 4 L_{y,n}^2 L_{h,n-1}^2 \lambda_{n-1}^{-1},~C_5(n) \coloneqq L_{y',n}\sigma_{n-1} (L_{h,n-1} \!+\!\sigma_{n-1}), 2\leq n \leq N;\nonumber\\
    &C_6(n)\coloneqq \big(L_{y,1}L_{v,y}+L_{v,y}^2\big) N L^2_y(n) ,~C_7(n)\coloneqq \big(L_{y',1}\sigma_0 L_{v,y} \!+\! 4 L_{v,y}^2 L_{y,1}^2 \!+\! 2L_v L_{v,y}^2 \big) N L^2_y(n) , \forall n.
\end{align}
Define a Lyapunov function $\mathcal{L}_k \coloneqq -F(x_k) + \sum_{n=1}^N\|y_k^n-y_k^{n,*}\|^2$. Then we have
\begin{align}\label{eq:lk+1-lk-}
    \E_k[\mathcal{L}_{k+1}]-\mathcal{L}_k 
    &= F(x_k) - \E_k[F(x_{k+1})] + \sum_{n=1}^N \E_k\|y_{k+1}^n-y_{k+1}^{n,*}\|^2 - \|y_k^n-y_k^{n,*}\|^2.
\end{align}
Substituting \eqref{eq:yk+1-yk+1*cfinal}, \eqref{eq:yk+1-yk+1*cfinal n>=2} and \eqref{eq:Fk+1-Fkfinal} into \eqref{eq:lk+1-lk-}, and then applying \eqref{eq:yk+1-yk*} yields
\begin{align}\label{eq:lk+1-lk}
    &\E_k[\mathcal{L}_{k+1}]-\mathcal{L}_k \nonumber\\
    &\leq \big(\!-\!\frac{1}{4}\alpha_k + (L_{y',1}\sigma_0 \!+\! 4 L_{y,1}^2 \!+\! 2 L_v)\alpha_k^2\big)\| v(x_k)\|^2 \nonumber\\
    &\quad\!+\! \sum_{n=1}^{N\!-\!1} \big( (1\!+\!C_4(n)\beta_{k,n\!-\!1} \!+\! C_5(n)\beta_{k,n\!-\!1}^2)(1\!-\!\lambda_n \beta_{k,n}) \!-\! 1 \!+\! \frac{\lambda_{n}}{2}\beta_{k,n} \!+\! C_6(n)\alpha_k \!+\! C_7(n)\alpha_k^2 \big)\|y_k^{n} \!-\! y_k^{n,*}\|^2 \nonumber\\
    &\quad\!+\! \big( (1\!+\!C_4(N)\beta_{k,N\!-\!1} \!+\! C_5(N)\beta_{k,N\!-\!1}^2)(1\!-\!\lambda_N \beta_{k,N}) \!-\! 1  \!+\! C_6(N)\alpha_k \!+\! C_7(N)\alpha_k^2 \big)\|y_k^{N} \!-\! y_k^{N,*}\|^2 \nonumber\\
    &\quad+ \Theta(\alpha_k^2) + \Theta\Big(\sum_{n=1}^N (1+\beta_{k,n-1}+\beta_{k,n-1}^2)\beta_{k,n}^2\Big).
\end{align}
As a clarification, the second term in the last inequality is $0$ when $N=1$. We have also used $\beta_{k,0}= \alpha_k$.
Consider the following choice of step sizes
\begin{align}
\label{eq:step size aa}
    &\!-\frac{1}{4}\alpha_k + (L_{y',1}\sigma_0 \!+\! 4 L_{y,1}^2 \!+\! 2 L_v)\alpha_k^2 \leq -\frac{1}{8}\alpha_k, \\
    \label{eq:step size bb}
    &(1 \!+\! C_1(n)\beta_{k,n-1} \!+\! C_2(n)\beta_{k,n-1}^2)(1\!-\!\lambda_n \beta_{k,n}) \!-\! 1 \!+\! \frac{\lambda_{n}}{2}\beta_{k,n} \!+\! C_3(n)\alpha_k \!+\! C_4(n)\alpha_k^2 \leq -\lambda_n \alpha_k, n \!\leq\! N\!-\!1, \\
    \label{eq:step size cc}
    &(1\!+\!C_1(N)\beta_{k,N-1} \!+\! C_2(N)\beta_{k,N-1}^2)(1-\lambda_N \beta_{k,N}) - 1  \!+\! C_3(N)\alpha_k + C_4(N)\alpha_k^2\leq -\lambda_N \alpha_k.
\end{align}
Note that \eqref{eq:step size aa} always admits solution for small enough $\alpha_1$. Given $\beta_{k,N}$, applying Lemma \ref{lemma:step sizes} for $n=N,\ldots,1$ to \eqref{eq:step size cc} and \eqref{eq:step size bb} tells that there exist solutions for $\beta_{k,n}(\forall n)$.

With \eqref{eq:step size aa}--\eqref{eq:step size cc}, it follows from \eqref{eq:lk+1-lk} that
\begin{align}\label{eq:lk+1-lkfinal}
    &\E_k[\mathcal{L}_{k+1}]-\mathcal{L}_k \nonumber\\
    &\!\leq\!  -\frac{\alpha_k}{8}\|v(x_k)\|^2 \!-\! \sum_{n=1}^N\lambda_n \alpha_k \|y_k^n-y_k^{n,*}\|^2 \!+\! \Theta(\alpha_k^2) \!+\! \Theta\Big(\sum_{n=1}^N (1\!+\!\beta_{k,n-1}\!+\!\beta_{k,n-1}^2)\beta_{k,n}^2\Big).
\end{align}
Furthermore, taking expectation on both sides of \eqref{eq:lk+1-lkfinal} then summing over $k=1,\ldots,K$ yields
\begin{align}\label{eq:rate}
    &\sum_{k=1}^K\alpha_k\E\Big[\frac{1}{8}\|v(x_k)\|^2 + \lambda_n \|y_k^n-y_k^{n,*}\|^2 \Big]\nonumber\\
    &\leq  \mathcal{L}_1-\E[\mathcal{L}_{K+1}] + \Theta\Big(\sum_{k=1}^K\alpha_k^2 \Big) + \Theta\Big(\sum_{k=1}^K\sum_{n=1}^N (1+\beta_{k,n-1}+\beta_{k,n-1}^2)\beta_{k,n}^2\Big)\nonumber\\
    &\leq \mathcal{L}_1 + C_F+  \Theta\Big(\sum_{k=1}^K\alpha_k^2 \Big) + \Theta\Big(\sum_{k=1}^K\sum_{n=1}^N (1+\beta_{k,n-1}+\beta_{k,n-1}^2)\beta_{k,n}^2\Big).
\end{align}
The inequality \eqref{eq:rate} implies a convergence rate of $\mathcal{O}(\frac{1}{\sqrt{K}})$ with step sizes $\alpha_k = \Theta(\frac{1}{\sqrt{K}})$ and $\beta_k=\Theta(\frac{1}{\sqrt{K}})$. This completes the proof.

\section{Proof of Lemma \ref{lemma:bilevel} ~and~ Corollary \ref{corollary:bilevel}}

To prove the corollary, it suffices to prove Lemma \ref{lemma:bilevel} and then directly apply Theorem \ref{theorem:strongly monotone} and \ref{theorem:non-strongly-monotone-2}. We direct the readers interested in why we can relax the assumptions in \citep{chen2021tighter} to the proof of Theorem \ref{theorem:strongly monotone} and \ref{theorem:non-strongly-monotone-2}. In particular, we provide a refined technique on bounding the drifting optimality gap in \eqref{eq:yk+1-yk+1*} and \eqref{eq:yk+1-yk+1*c}, which is crucial in alleviating the assumption. 

\begin{proof}
We start to verify the Assumptions by order.

\noindent\textbf{(1) Conditions \ref{item:strong convex}~and~\ref{item:g lip 2} $\xrightarrow[]{}$ Assumption \ref{assumption:y*}.} Since $g(x,y)$ is strongly-convex w.r.t. $y$, there exists a unique $y^*(x)$ such that $h(x,y^*(x))=-\nabla_y g(x,y^*(x))=0$. 

By \cite[Lemma 2.2]{ghadimi2018approximation}, we have
\begin{align}
    \|y^*(x) - y^*(x')\| \leq L_y\|x-x'\|,~L_y = \frac{L_{xy}}{\lambda_1}.
\end{align}
By \cite[Lemma 2]{chen2021tighter}, we have
\begin{align}
    \|\nabla y^*(x) - \nabla y^*(x')\| \leq L_{y'}\|x-x'\|,~L_{y'} = \frac{l_{xy}+l_{xy}L_y}{\lambda_1} + \frac{L_{xy}(l_{yy}+l_{yy}L_y)}{\lambda_1^2}.
\end{align}
\textbf{(2) Conditions \ref{item:strong convex}--\ref{item:f lip} $\xrightarrow[]{}$ Assumption \ref{assumption:vglip}.} 
By \cite[Lemma 2.2]{ghadimi2018approximation}, we have
\begin{subequations}
\begin{align}
    \|v(x,y)-v(x,y')\| &\leq L_{v,y} \|y-y'\|,~L_{v,y}=l_{fx} + \frac{l_{fy}L_{xy}}{\lambda_1}+l_y(\frac{l_{xy}}{\lambda_1}+ \frac{L_h L_{xy}}{\lambda_1^2}) \\
    \|v(x)-v(x')\| &\leq L_v \|y-y'\|,~L_v = \frac{L_{xy}(L_{v,y}+l_{fy}')}{\lambda_1} + l_{fx} + l_y \big( \frac{l_{xy}l_y}{\lambda_1} + \frac{l_{yy}L_{xy}}{\lambda_1^2} \big). 
\end{align}
\end{subequations}
Lastly, it follows from condition \ref{item:g lip 2} that $\|h(x,y)-h(x,y')\| \leq L_h \|y-y'\|$.

\noindent\textbf{(3) Condition \ref{item:noise bilevel} $\xrightarrow[]{}$ Assumption \ref{assumption:noise}; \ref{item:strong convex} $\xrightarrow[]{}$ Assumption \ref{assumption:strongly g}; \ref{item:F RSI} $\xrightarrow[]{}$ Assumption \ref{assumption:strongly v}; \ref{item:F} $\xrightarrow[]{}$ Assumption \ref{assumption:integral}.} These conditions directly imply their corresponding Assumption \ref{assumption:noise} when $N=1$.
\end{proof}

\section{Proof of Theorem \ref{theorem:ac}}\label{section:ac appendix}
Throughout this section, we omit all the index $n$ since $N=1$. We also write $y^*(x_k)$ in short as $y_k^*$. In this section, we first give a proof of Lemma \ref{lemma:ac} and then a proof of Theorem \ref{theorem:ac}.
\subsection{Proof of Lemma \ref{lemma:ac}}

\begin{proof}
 We will check the assumptions by order.

\noindent\textbf{(1) Condition \ref{item:A ac}--\ref{item:ergodic} $\xrightarrow[]{}$Assumption \ref{assumption:y*}.} Under condition \ref{item:A ac}, we have $y^*(x) = -A_x^{-1}b_x$. With Lemma \ref{lemma:A-A'} and $\|A_x^{-1}\|\leq \sigma^{-1}, \|b_x\|\leq 1$, applying Lemma \ref{lemma:product lip} to $y^*(x)$ implies that it is Lipschitz continuous with modulus $L_y \coloneqq (\sigma^{-1}+2\sigma^{-2})L_\mu'$.

We next verify the Lipschitz continuity of $\nabla y^*(x)$. 
For a vector $x\in \mathbb{R}^{d}$ and a differentiable mapping $f:\mathbb{R}^d\mapsto \mathbb{R}^{d_1 \times d_2}$, we denote $[x]_i$ as the $i$th element of $x$ and we use $\nabla_i f(x) \coloneqq \frac{\partial f(x)}{\partial [x]_i}$. Then we have
\begin{align}\label{eq:ny*}
    \nabla_i y^*(x) = A_x^{-1}\nabla_i A_x A_x^{-1} b_x - A_x^{-1} \nabla_i b_x = -A_x^{-1}\nabla_i A_x y^*(x) - A_x^{-1} \nabla_i b_x.
\end{align}

By Lemma \ref{lemma:product lip}, to prove $\nabla_i y^*(x)$ is Lipschitz continuous w.r.t. $x$, it suffices to prove $\nabla_i A_x$, $y^*(x)$, $\nabla_i b_x$ and $A_x^{-1}$ are bounded (in norm) and Lipschitz continuous. First we have
\begin{align}\label{eq:AbAbound}
     \|A_x^{-1}\| \leq \sigma^{-1},\|y^*(x)\| \leq \|A_x^{-1}\|\|b_x\|\leq\sigma^{-1}.
\end{align}
And by Lemma \ref{lemma:A-A'}, we have
\begin{align}\label{eq:Ax-Ax'}
    \|A_x^{-1} - A_{x'}^{-1}\| \leq  2\sigma^{-2}  L_\mu' \|x-x'\|.
\end{align}
Thus it suffices to prove $\nabla_i A_x$ and $\nabla_i b_x$ are bounded in norm and Lipschitz continuous.

We start by
\begin{align}
    \nabla_i b_x = \E_{s\sim \mu_{\pi_x}, a\sim\pi_x(\cdot|s)}[\nabla_i \log \pi_x(a|s) G_x (s,a)],
\end{align}
where $G_x(s,a) \coloneqq \E_{\pi_x}[\sum_{t=0}^\infty \big(r(s_t,a_t)\phi(s_t)-b_x\big)|s_0=s,a_0=a]$. By letting $\hat{r}(s,a,s')=r(s,a)\phi(s)$ in Lemma \ref{lemma:Gx}, we have
\begin{align}\label{eq:gx w}
    \|G_x(s,a)\| \leq C_G,~ \|G_x(s,a)-G_{x'}(s,a)\| \leq L_G\|x-x'\|,
\end{align}
where $C_G \coloneqq 2 + \frac{\rho\kappa}{1-\rho}$ and $L_G \coloneqq L_\mu'+\frac{\rho\kappa L_\pi|\mathcal{A}|}{1-\rho}+\big(\frac{\kappa}{1-\rho}+1\big)^2(L_\pi |\mathcal{A}| + L_\mu ) + L_\mu$.
Then we have $\|\nabla_i b_x\|$ is bounded as
\begin{align}\label{eq:nibx upper}
    \|\nabla_i b_x\| \leq C_\pi\E_{s\sim \mu_x, a\sim\pi_x(\cdot|s)}[\|G_x (s,a)\|]
    \leq  C_\pi C_G,
\end{align}
Now we start to prove the Lipschitz continuity of $\nabla_i b_x$. First we have
\begin{align}\label{eq:nib-nib'}
    &\|\nabla_i b_x - \nabla_i b_{x'}\| \nonumber\\
    &\leq \big\|\E_{s\sim \mu_{\pi_x}, a\sim\pi_x}[\nabla_i \log \pi_x(a|s) G_x (s,a)]- \E_{s\sim \mu_{\pi_{x'}}, a\sim\pi_{x'}}[\nabla_i \log \pi_x(a|s) G_x (s,a)]\big\| \nonumber\\
    &\quad+\E_{s\sim \mu_{\pi_{x'}}, a\sim\pi_{x'}}\big\|\nabla_i \log \pi_x(a|s) G_x (s,a)-\nabla_i \log \pi_{x'}(a|s) G_{x'} (s,a)\big\| \nonumber\\
    &\leq  \|\mu_{\pi_x}\cdot \pi_x - \mu_{\pi_{x'}}\cdot\pi_{x'}\|_{TV} \sup \|\nabla \log \pi_x(a|s) G_x (s,a)\|\nonumber\\
    &\quad+\E_{s\sim \mu_{\pi_{x'}}, a\sim\pi_{x'}}\|\nabla_i \log \pi_x(a|s) G_x (s,a)-\nabla_i \log \pi_{x'}(a|s) G_{x'} (s,a)\| \nonumber\\
    &\leq C_G L_\mu'\|x-x'\| + (C_\pi L_G + L_\pi C_G)\|x-x'\| \coloneqq L_b'\|x-x'\|,
\end{align}
where the $\mu_{\pi_x}\cdot \pi_x$ denotes the probability measure specified by the probability function $(\mu_{\pi_x}\cdot \pi_x) (s,a)=\mu_{\pi_x}(s)\pi_x(a|s)$.
In the second inequality, we apply Lemma \ref{lemma:mu-mu'} to the first term; and for the second term, we apply Lemma \ref{lemma:product lip} along with \eqref{eq:gx w} and condition \ref{item:policy ac}.

For $\nabla_i A_x$, we have
\begin{align}
    \nabla_i A_x = \E_{s\sim \mu_{\pi_x}, a\sim\pi_x}[\nabla_i \log \pi_x(a|s) G_x (s,a)],
\end{align}
where we slightly abuse the notation and define $G_x(s,a) \coloneqq \E_{\pi_x}\big[\sum_{t=0}^\infty \big( \phi(s_t)\big(\gamma \phi(s_{t+1})-\phi(s_t)\big)^\top-A_x \big)|s_0=s,a_0=a\big]$. Observing that $\nabla_i A_x$ has similar structure as that of $\nabla_i b_x$, we can apply the same technique and obtain
\begin{align}\label{eq:nA nA'}
    \|\nabla_i A_x\|
    &\leq  C_\pi C_G', \nonumber\\
    \|\nabla_i A_x - \nabla_i A_{x'}\|
    &\leq C_G' L_\mu'\|x-x'\| + (C_\pi L_G' + L_\pi C_G')\|x-x'\|\coloneqq L_A' \|x-x'\|,
\end{align}
where $C_G' \coloneqq 4 + \frac{2\rho\kappa}{1-\rho}$ and $L_G' \coloneqq 2 L_\mu'+\frac{\rho\kappa L_\pi|\mathcal{A}|}{1-\rho}+\big(\frac{\kappa}{1-\rho}+1\big)^2(L_\pi |\mathcal{A}| + L_\mu ) + L_\mu$.

Finally, applying Lemma \ref{lemma:product lip} to \eqref{eq:ny*} with \eqref{eq:AbAbound}, \eqref{eq:Ax-Ax'}, \eqref{eq:nibx upper}, \eqref{eq:nib-nib'} and \eqref{eq:nA nA'} yields
\begin{align}
    \|\nabla_i y^*(x) - \nabla_i y^*(x')\| \leq L_{y'}\|x-x'\|,
\end{align}
where $L_{y'}\coloneqq 2\sigma^{-3}L_\mu' C_\pi C_G'+L_A'\sigma^{-2}+L_y C_\pi C_G' \sigma^{-1}+2\sigma^{-2}L_\mu'C_\pi C_G + \sigma^{-1}L_b'$.

\noindent\textbf{(2) Condition \ref{item:A ac}~and~\ref{item:policy ac} $\xrightarrow[]{}$Assumption \ref{assumption:vglip}~and~\ref{assumption:strongly g}.} We first check Assumption \ref{assumption:vglip}. In AC, we have $v(x)=v(x,y^*(x))=\nabla F(x)$. By \citep[Lemma 3.2]{zhang2019global}, there exists a constant $L_v\coloneqq \frac{L_\pi'}{(1-\gamma)^2}+\frac{(1+\gamma)C_\pi}{(1-\gamma)^2}$ such that
\begin{align}
    \|\nabla F(x) - \nabla F(x')\| \leq L_v \|x-x'\|.
\end{align}
Then we have
\begin{align}
    \|v(x,y) - v(x,y')\| &= \|\E[(\gamma \phi(s') - \phi(s))^\top (y-y') \nabla \log \pi_x(a|s)]\| \leq 2 C_\pi \|y-y'\|, \nonumber\\
    \|h(x,y)-h(x,y')\| &= \|A_x (y-y')\|\leq 2\|y-y'\|.
\end{align}
This completes the verification of Assumption \ref{assumption:vglip}. Lastly, Assumption \ref{assumption:strongly g} is directly implied by the inequality $\langle y-y', A_x (y-y') \rangle \leq -\lambda_1 \|y-y'\|^2$ in condition \ref{item:A ac}. 

\noindent\textbf{(3) Assumption \ref{assumption:integral} holds.} It is clear that $|F(x)| \leq \frac{1}{1-\gamma}$.

\noindent\textbf{(4) Proving \eqref{eq:asm 3 ac}.} It is easy to check that $\E[\xi_k|\mathcal{F}_k^{1}]=0,~\E[\psi_k|\mathcal{F}_k]=0$. Next we have
\begin{align}
\|\xi_k\|^2
&\leq 2\E\|(r(s,a)+(\gamma \phi(s')-\phi(s))^\top y_k) \nabla \log \pi_{x_k}(a|s)]\|^2 \nonumber\\
&\quad+  2\|(r(\bar{s}_k,\bar{a}_k)+(\gamma \phi(\bar{s}_k)-\phi(\bar{s}_k))^\top y_k) \nabla \log \pi_{x_k}(\bar{a}_k|\bar{s}_k)\|^2\nonumber\\
&\leq 8 C_\pi^2 + 16 C_\pi^2 \|y_k\|^2 \nonumber\\
&\leq 8 C_\pi^2 + 32 C_\pi^2 \|y_k^*\|^2 + 32 C_\pi^2 \|y_k-y_k^*\|^2  \nonumber\\
&\leq 8 C_\pi^2(1+4\sigma^{-2}) + 32 C_\pi^2 \|y_k-y_k^*\|^2 \coloneqq \sigma_0^2 + \bar{\sigma}_0^2 \|y_k-y_k^*\|^2,
\end{align}
where to get the last inequality we have used $\|y_k^*\| = \|A_{x_k}^{-1}b_{x_k}\| \leq \sigma^{-1}$. Similarly we have
\begin{align}
    \|\psi_k\|^2 
    &\leq 2\E\|\phi(s)(\gamma \phi(s') - \phi(s))^\top y_k + r(s,a)\phi(s)\|^2 \nonumber\\
    &\quad+ 2\|\phi(s_k)(\gamma \phi(s_k') - \phi(s_k))^\top y_k+ r(s_k,a_k)\phi(s_k)\|^2 \nonumber\\
    &\leq 16 \|y_k\|^2 + 8 \nonumber\\
    &\leq 32 \|y_k-y_k^*\|^2 + 32 \sigma^{-2} + 8 \coloneqq \sigma_1^2 + \bar{\sigma}_1^2 \|y_k-y_k^*\|^2.
\end{align}
This completes the proof.
\end{proof}

\subsection{Proof of Theorem \ref{theorem:ac}}
In this section, we will provide a proof of theorem \ref{theorem:ac}. 

\begin{proof}
 The proof will be similar to that of Theorem \ref{theorem:non-strongly-monotone-2}, and only the steps that are different due to the adaptation of Assumption \ref{assumption:noise} to AC will be shown here.

In the following proof, we omit the index $n$ since $N=1$. We also write $\E[\cdot|\mathcal{F}_k]$ in short as $\E_k[\cdot]$ and $y^*(x_k)$ as $y_k^*$.

\noindent\textbf{Contraction of the critic optimality gap.} First we have
\begin{align}\label{eq:yk+1-yk*0 ac}
    \E_k\|y_{k+1}-y_k^*\|^2 
    &= \|y_k-y_k^*\|^2 + 2\beta_k \E_k\langle y_k-y_k^*,h(x_k,y_k)+\psi_k\rangle + \E_k\|y_{k+1}-y_k\|^2,
\end{align}
The second term in \eqref{eq:yk+1-yk*0 ac} can be bounded as
\begin{align}\label{eq:yk+1-yk*0 ac I2}
    \E_k\langle y_k-y_k^*,h(x_k,y_k)+\psi_k\rangle
    &= \langle y_k-y_k^*,h(x_k,y_k)  \rangle+\langle y_k-y_k^*,\E_k[\psi_k]\rangle \nonumber\\
    &\leq -\lambda_1 \|y_k-y_k^*\|^2.
\end{align}
where the last inequality follows from the strong monotonicity of $h(x,y)$ and $\E_k[\psi_k]=0$ verified in Lemma \ref{lemma:ac}.

The third term in \eqref{eq:yk+1-yk*0 ac} can be bounded as
\begin{align}\label{eq:yk+1-yk*0 ac I3}
    \E_k\|y_{k+1}-y_k\|^2
    &= \beta_k^2 \E_k\|h(x_k,y_k)+\psi_k\|^2 \nonumber\\
    &= \beta_k^2 (\|h(x_k,y_k)\|^2 + \E_k\|\psi_k\|^2) \nonumber\\
    &\leq \beta_k^2 (\|h(x_k,y_k)\|^2 + \sigma_1^2 + \bar{\sigma}_1^2 \|y_k-y_k^*\|^2) \nonumber\\
    &\leq (L_h^2 + \bar{\sigma}_1^2)\beta_k^2 \|y_k-y_k^*\|^2 + \sigma_1^2 \beta_k^2,
\end{align}
where the second last inequality follows from \eqref{eq:asm 3 ac} and the last inequality follows from Assumption \ref{assumption:vglip} which gives
\begin{align}
   \|h(x,y)\|=\|h(x,y)-\underbrace{h(x,y^*(x))}_{=0}\|\leq L_h \|y-y^*(x)\|.
\end{align}
Collecting the upper bounds in \eqref{eq:yk+1-yk*0 ac I2} and \eqref{eq:yk+1-yk*0 ac I3} yields
\begin{align}\label{eq:yk+1-yk* ac}
\E_k\|y_{k+1}-y_k^*\|^2
    &\leq (1-2\lambda_1 \beta_k + (L_h^2 + \bar{\sigma}_1^2) \beta_k^2)\|y_k-y_k^*\|^2 + \sigma_1^2 \beta_k^2 \nonumber\\
    &\leq (1-\lambda_1 \beta_k )\|y_k-y_k^*\|^2 + \sigma_1^2 \beta_k^2,
\end{align}
where the last inequality is due to the choice of step size that satisfies $(L_h^2 + \bar{\sigma}_1^2) \beta_k^2\leq \lambda_1 \beta_k$.

\noindent\textbf{Bounding the drifting optimality gap.} Next we start to bound the second term in \eqref{eq:driftc} as follows
\begin{align}\label{eq:driftc ac I2}
     &\E_k\langle y_k^*-y_{k+1},\alpha_k\nabla y^*(\hat{x}_{k+1})^\top \xi_k \rangle \nonumber\\
    &= \E_k\langle y_k^*-y_{k+1},\alpha_k\big(\nabla y^*(\hat{x}_{k+1})-\nabla y^*(x_k)\big)^\top \xi_k \rangle + \E_k\langle y_k^*-y_{k+1},\alpha_k\nabla y^*(x_k)^\top \E_k[\xi_k|\mathcal{F}_k^1] \rangle \nonumber\\
    &= \E_k\langle y_k^*-y_{k+1},\alpha_k\big(\nabla y^*(\hat{x}_{k+1})-\nabla y^*(x_k)\big)^\top \xi_k \rangle \nonumber\\
    &\leq \alpha_k\E_k\big[\| y_k^*-y_{k+1}\|\|\nabla y^*(\hat{x}_{k+1})-\nabla y^*(x_k)\|\| \xi_k\|\big] \nonumber\\
    &\leq \sigma_0\alpha_k\E_k\big[\| y_k^*\!-\!y_{k+1}\|\|\nabla y^*(\hat{x}_{k+1})\!-\!\nabla y^*(x_k)\|\big] \nonumber\\
    &\quad+ \bar{\sigma}_0\alpha_k\E_k\big[\| y_k^*\!-\!y_{k+1}\|\|\nabla y^*(\hat{x}_{k+1})\!-\!\nabla y^*(x_k)\|\| y_k-y_k^*\|\big].\!
\end{align}
where the second inequality follows from $\E_k[\xi_k|\mathcal{F}_k^1]=0$ shown in Lemma \ref{lemma:ac}, and the last inequality follows from \eqref{eq:asm 3 ac}.

The first term in \eqref{eq:driftc ac I2} can be bounded as
\begin{align}\label{eq:driftc ac I2 1}
    &\E_k\big[\| y_k^*\!-\!y_{k+1}\|\|\nabla y^*(\hat{x}_{k+1})\!-\!\nabla y^*(x_k)\|\big] \nonumber\\
    &\leq L_{y'}\E_k[\| y_k^*\!-\!y_{k+1}\|\|x_{k+1}-x_k\|]\nonumber\\
    &\leq L_{y'}\alpha_k\big(\E_k[\| y_k^*\!-\!y_{k+1}\|\|v(x_k,y_k)\|] + \E_k[\| y_k^*\!-\!y_{k+1}\|\|\xi_k\|]\big) \nonumber\\
    &\leq \frac{1}{2}L_{y'}\alpha_k\big(\E_k\| y_k^*\!-\!y_{k+1}\|^2+\|v(x_k,y_k)\|^2 + \E_k\| y_k^*\!-\!y_{k+1}\|^2 + \|\xi_k\|^2\big) \nonumber\\
    &\leq \frac{1}{2}L_{y'}\alpha_k\big(2\E_k\| y_k^*\!-\!y_{k+1}\|^2+\|v(x_k,y_k)\|^2 + \sigma_1^2 + \bar{\sigma}_1^2 \|y_k-y_k^*\|^2\big).
\end{align}
where the first inequality follows from Lemma \ref{lemma:ac} and the last inequality follows from \eqref{eq:asm 3 ac}.

The second term in \eqref{eq:driftc ac I2} can be bounded as
\begin{align}\label{eq:driftc ac I2 2}
    \E_k\big[\| y_k^*\!-\!y_{k+1}\|\|\nabla y^*(\hat{x}_{k+1})\!-\!\nabla y^*(x_k)\|\| y_k-y_k^*\|\big]
    &\leq 2L_y\E_k\big[\| y_k^*\!-\!y_{k+1}\|\| y_k-y_k^*\|\big] \nonumber\\
    &\leq L_y\E_k\| y_k^*\!-\!y_{k+1}\|^2 + L_y\|y_k-y_k^*\|^2.
\end{align}
Substituting \eqref{eq:driftc ac I2 1} and \eqref{eq:driftc ac I2 2} into \eqref{eq:driftc ac I2}, then substituting \eqref{eq:driftc ac I2} and \eqref{eq:driftcI1 1} into \eqref{eq:driftc} gives
\begin{align}\label{eq:drift ac I2}
    \E_k\langle y_k^*-y_{k+1},y_{k+1}^*-y_k^* \rangle 
      &\leq \big(L_y \big(\frac{L_{v,y}}{2} \!+\! 2 L_y\!+\!\bar{\sigma}_0 \big) \alpha_k + L_{y'}\sigma_0 \alpha_k^2\big)\E_k\|y_{k+1}-y_k^*\|^2 \nonumber\\
      &\quad+ \frac{1}{2}\big(L_y(L_{v,y}+\bar{\sigma}_0) \alpha_k + L_{y'}\sigma_0 \bar{\sigma}_1^2\alpha_k^2 \big)\|y_k - y_k^*\|^2 \nonumber\\
      &\quad+ \big(\frac{1}{8}\alpha_k + \frac{1}{2}L_{y'}\sigma_0\alpha_k^2\big)\| v(x_k)\|^2 + \frac{1}{2}L_{y'}\sigma_0 \sigma_1^2 \alpha_k^2.
\end{align}
The last term in \eqref{eq:yk+1-yk+1*c} is bounded as
\begin{align}\label{eq:drift ac I3}
    \E_k\|y_{k+1}^*-y_k^*\|^2 &\leq L_y^2 \alpha_k^2 \E_k\|v(x_k,y_k)+\xi_k\|^2 = L_y^2 \alpha_k^2 (\|v(x_k,y_k)\|^2 + \E_k\|\xi_k\|^2)\nonumber\\
    &\leq L_y^2\alpha_k^2 (\|v(x_k,y_k)\|^2 + \sigma_0^2 + \bar{\sigma}_0^2 \|y_k-y_k^*\|^2),
\end{align}
where the last inequality follows from \eqref{eq:asm 3 ac}. Substituting \eqref{eq:drift ac I2} and \eqref{eq:drift ac I3} into \eqref{eq:yk+1-yk+1*c} gives
\begin{align}\label{eq:yk+1-yk+1* ac}
    \E_k\|y_{k+1} - y_{k+1}^*\|^2 
    &\leq \big(1 + L_y \big(L_{v,y} \!+\! 4 L_y\!+\!2\bar{\sigma}_0 \big) \alpha_k + 2 L_{y'}\sigma_0 \alpha_k^2\big)\E_k\|y_{k+1}-y_k^*\|^2 \nonumber\\
      &\quad+ \big(L_y(L_{v,y}+\bar{\sigma}_0) \alpha_k + (L_{y'}\sigma_0 \bar{\sigma}_1^2 + L_y^2\bar{\sigma}_0^2 )\alpha_k^2 \big)\|y_k - y_k^*\|^2 \nonumber\\
      &\quad+ \big(\frac{1}{4}\alpha_k + (L_{y'}\sigma_0+L_y^2)\alpha_k^2\big)\| v(x_k,y_k*)\|^2 + (\sigma_0 \sigma_1^2+L_y^2\sigma_0^2) \alpha_k^2.
\end{align}
\textbf{Analysis of main sequence}. The second term in \eqref{eq:Fk+1-Fk} is instead bounded as
\begin{align}\label{eq:Fk+1-FkI2 ac}
    \E_k\langle v(x_k,y_k),\alpha_k \xi_k\rangle = \langle v(x_k,y_k),\alpha_k \E_k[\xi_k]\rangle = 0.
\end{align}
Then the last term in \eqref{eq:Fk+1-Fk} is instead bounded as
\begin{align}\label{eq:Fk+1-Fklast ac}
    \E_k\|x_{k+1}-x_k\|^2 = \alpha_k^2 (\|v(x_k,y_k)\|^2+\E_k\|\xi_k\|^2) \leq  \alpha_k^2 (\|v(x_k,y_k)\|^2+\sigma_0^2+\bar{\sigma}_0^2\|y_k-y_k^*\|^2).
\end{align}
Substituting the bounds in \eqref{eq:Fk+1-FkI1}, \eqref{eq:Fk+1-FkI2 ac} and \eqref{eq:Fk+1-Fklast ac} into \eqref{eq:Fk+1-Fk} yields
\begin{align}\label{eq:Fk+1-Fkfinal ac}
    \E_k[F(x_{k+1})] \!-\! F(x_k) &\!\geq\! (\frac{3\alpha_k}{4} \!-\! \frac{L_v}{2}\alpha_k^2) \|v(x_k,y_k^*)\|^2 \!-\! (L_{v,y}^2\alpha_k \!+\! \frac{L_{v,y}}{2}\bar{\sigma}_0^2 \alpha_k^2) \|y_k-y_k^*\|^2 \!-\! \frac{L_v\sigma_0^2}{2} \alpha_k^2.
\end{align}
\textbf{Establishing convergence.} Recall that the Lyapunov function $\mathcal{L}_k = - F(x_k) + \|y_k-y_k^*\|^2$. With the bounds in \eqref{eq:yk+1-yk* ac}, \eqref{eq:yk+1-yk+1* ac} and \eqref{eq:Fk+1-Fkfinal ac}, we have
\begin{align}\label{eq:ac final}
    \E_k[\mathcal{L}_{k+1}]-\mathcal{L}_k 
    &\leq \big(-\frac{1}{2}\alpha_k + (\frac{L_v}{2}+L_{y'}\sigma_0 L_y^2)\alpha_k^2 \big)\|v(x_k,y_k^*)\|^2 \nonumber\\
    &\quad+ \big((1+C_0' \alpha_k + C_1' \alpha_k^2)(1-\lambda_1 \beta_k) -1 + C_2' \alpha_k + C_3' \alpha_k^2 \big)\|y_k-y_k^*\|^2 \nonumber\\
    &\quad+ \Theta(\alpha_k^2 + (1+\alpha_k+\alpha_k^2)\beta_k^2),
\end{align}
where $C_0' \coloneqq L_y \big(L_{v,y} \!+\! 4 L_y\!+\!2\bar{\sigma}_0 \big)$, $C_1' \coloneqq 2 L_{y'}\sigma_0$, $C_2' \coloneqq L_y(L_{v,y}+\bar{\sigma}_0) + L_{v,y}^2$, $C_3' \coloneqq L_{y'}\sigma_0 \bar{\sigma}_1^2 + \frac{L_{v,y}}{2}\bar{\sigma}_0^2$.
Notice that \eqref{eq:ac final} takes a similar form to that of \eqref{eq:lk+1-lk} ($N=1$). 

If the step sizes are chosen such that
\begin{align}
    -\frac{1}{2}\alpha_k + (\frac{L_v}{2}+L_{y'}\sigma_0 L_y^2)\alpha_k^2 \leq -\frac{1}{4}\alpha_k, \nonumber\\
    (1+C_0' \alpha_k + C_1' \alpha_k^2)(1-\lambda_1 \beta_k) -1 + C_2' \alpha_k + C_3' \alpha_k^2 \leq -\lambda_1 \alpha_k,
\end{align}
then it follows from the derivation after \eqref{eq:lk+1-lk} that Theorem \ref{theorem:non-strongly-monotone-2} holds for AC update.
\end{proof}

\subsection{Supporting lemma}

\begin{Lemma}\citep[Lemma 3]{zou2019sarsa}\label{lemma:mu-mu'}
Define $(\mu_{\pi_x} \cdot \pi_x) (s,a) \coloneqq \mu_{\pi_x}(s)\pi_x(a|s)$. Under conditions \ref{item:ergodic} and \ref{item:policy ac}, it holds that
\begin{align}
    \|\mu_{\pi_x}-\mu_{\pi_{x'}}\|_{TV} \leq L_\mu \|x-x'\|,~\|\mu_{\pi_x} \cdot \pi_x-\mu_{\pi_{x'}} \cdot \pi_{x'}\|_{TV} \leq L_\mu' \|x-x'\|.
\end{align}
where $L_\mu \coloneqq  2 L_\pi |\mathcal{A}|(\log_{\rho}\kappa^{-1}+\frac{1}{1-\rho})$ and $L_\mu'\coloneqq L_\mu + 2 L_\pi |\mathcal{A}|$.
\end{Lemma}

\begin{Lemma}\label{lemma:A-A'}
Define $\mu_{\pi_x} \cdot \pi_x (s,a) \coloneqq \mu_{\pi_x}(s)\pi_x(a|s)$. Under conditions \ref{item:ergodic} and \ref{item:policy ac}, the following inequalities hold
\begin{align}
    \|A_x - A_{x'}\| \leq 2  L_\mu' \|x-x'\|,~\|A_x^{-1} - A_{x'}^{-1}\| \leq  2\sigma^{-2}  L_\mu' \|x-x'\|,~\|b_x -b_{x'}\| \leq L_\mu' \|x-x'\|.
\end{align}
where $L_\mu' =  2 L_\pi |\mathcal{A}|(1+\log_{\rho}\kappa^{-1}+\frac{1}{1-\rho})$.
\end{Lemma}
\begin{proof}
 First we have
\begin{align}
    \|b_x - b_{x'}\| \leq \|\mu_{\pi_x} \cdot \pi_x-\mu_{\pi_{x'}} \cdot \pi_{x'}\|_{TV} \sup_{s,a}\|r(s,a)\phi(s)\| \leq  L_\mu' \|x-x'\|,
\end{align}
where the last inequality follows from Lemma \ref{lemma:mu-mu'}.
And similarly
\begin{align}
    \|A_x - A_{x'}\| \leq 2  L_\mu' \|x-x'\|.
\end{align}
Finally, we have
\begin{align}
    &\|A_x^{-1} - A_{x'}^{-1}\| 
    = \|A_{x'}^{-1}(A_x - A_{x'})A_x^{-1}\|\leq \sigma^{-2}\|A_x - A_{x'}\| \nonumber\\
    &\!\leq\! \sigma^{-2} \|\mu_{\pi_x} \cdot \pi_x-\mu_{\pi_{x'}} \cdot \pi_{x'}\|_{TV} \sup_{s,s'}\|\phi(s)(\gamma \phi(s')-\phi(s))\|\!\leq\! 2\sigma^{-2}  L_\mu' \|x-x'\|,
\end{align}
where the last inequality follows from Lemma \ref{lemma:mu-mu'}.
This completes the proof.
\end{proof}

\begin{Lemma}\label{lemma:Gx}
Suppose conditions \ref{item:A ac}--\ref{item:ergodic} hold. With mapping $\hat{r}:\mathcal{S}\times \mathcal{A}\times\mathcal{S}\mapsto \mathbb{R}^{d\times d'}$ such that $\|r(s,a,s')\|\leq C_r$ for any $(s,a,s')$, define 
\begin{align}
    &G_x(s,a)\coloneqq\E_{\substack{a_t\sim\pi_x(\cdot|s_t)\\s_{t+1}\sim\mathcal{P}(\cdot|s_t,a_t)}}\big[\sum_{t=0}^\infty \big(\hat{r}(s_t,a_t,s_{t+1})-\bar{r}_x\big)\big|s_0=s,a_0=a\big],\nonumber\\
    &{\rm with}~ \bar{r}_x \coloneqq \E_{\substack{s\sim\mu_{\pi_x},a\sim\pi_x(\cdot|s)\\s'\sim\mathcal{P}(\cdot|s,a)}}[\hat{r}(s,a,s')],
\end{align}
Then there exists a constant $L_G$ such that for any $(s,a)\in\mathcal{S}\times\mathcal{A}$ and $x,x'\in\mathbb{R}^d$, the following inequalities hold
\begin{align}
    \|G_x(s,a)-G_{x'}(s,a)\| &\leq L_G\|x-x'\|, \nonumber\\
    \|G_x(s,a)\|&\leq 2 C_r + \frac{C_r\rho\kappa}{1-\rho}.
\end{align}
\end{Lemma}
\begin{proof}
 We write $G_x(s_0,a_0)$ as:
\begin{align}\label{eq:Gx 0}
    G_x(s_0,a_0) 
    &= \E_{s_1\sim\mathcal{P}}[\hat{r}(s_0,a_0,s_1)]\!-\!\bar{r}_x\!+\!\sum_{t=1}^\infty \Big(\!\!\!\sum_{(s,a)\in\mathcal{S}\times\mathcal{A}}\!\!\!\!\!\mathbf{Pr}_{\pi_x}(s_t\!=\!s|s_0,a_0)\pi_x(a|s)\E_{s'\sim\mathcal{P}}[\hat{r}(s,a,s')] \nonumber\\
    &\quad- \!\!\!\sum_{(s,a)\in\mathcal{S}\times\mathcal{A}}\!\!\!\mu_{\pi_x}(s)\pi_x(a|s)\E_{s'\sim\mathcal{P}}[\hat{r}(s,a,s')] \Big).
\end{align}
Given $(s_0,a_0)$, define the vector $p_1 \coloneqq [\mathcal{P}(s^{(0)}|s_0,a_0), \mathcal{P}(s^{(1)}|s_0,a_0),\ldots,\mathcal{P}(s^{(|\mathcal{S}|)}|s_0,a_0)]$ where $s^{(0)},\ldots,s^{(|\mathcal{S}|)}$ are states in $\mathcal{S}$. 
Given $\pi_x$, define the following state transition matrix
\begin{align}
    P_{\pi_x}\coloneqq \begin{bmatrix} \mathcal{P}_{\pi_x}(s^{(0)}|s^{(0)}) & \mathcal{P}_{\pi_x}(s^{(1)}|s^{(0)}) &\dots & \mathcal{P}_{\pi_x}(s^{(|\mathcal{S}|)}|s^{(0)}) \\ \vdots \\ \mathcal{P}_{\pi_x}(s^{(0)}|s^{(|\mathcal{S}|)}) & \mathcal{P}_{\pi_x}(s^{(1)}|s^{(|\mathcal{S}|)}) &\dots & \mathcal{P}_{\pi_x}(s^{(|\mathcal{S}|)}|s^{(|\mathcal{S}|)})\end{bmatrix},
\end{align}
where $\mathcal{P}_{\pi_x}(s'|s) = \sum_{a\in\mathcal{A}}\mathcal{P}(s'|s,a)\pi_x(a|s)$. Then it is clear that we can write the probability function $\mathbf{Pr}_{\pi_x}(s_t=\cdot|s_0,a_0)$ as its vector form $p_1 P_{\pi_x}^{t-1}$.
We slightly abuse the notation and use $[p_1 P_{\pi_x}^t]_s=\mathbf{Pr}_{\pi_x}(s_t=s|s_0,a_0)$. Then \eqref{eq:Gx 0} can be rewritten as
\begin{align}
    G_x(s_0,a_0) 
    &= \E_{s_1\sim\mathcal{P}}[\hat{r}(s_0,a_0,s_1)]-\bar{r}_x+\sum_{t=0}^\infty \Big(\sum_{(s,a)\in\mathcal{S}\times\mathcal{A}}[p_1 P_{\pi_x}^t]_s\pi_x(a|s)\E_{s'\sim\mathcal{P}}[\hat{r}(s,a,s')] \nonumber\\
    &\quad- \sum_{s,a}[p_1 P_{\pi_x}^\infty]_s\pi_x(a|s)\E_{s'\sim\mathcal{P}}[\hat{r}(s,a,s')] \Big) \nonumber\\
    &= \E_{s_1\sim\mathcal{P}}[\hat{r}(s_0,a_0,s_1)]-\bar{r}_x+\sum_{t=0}^\infty \sum_{(s,a)}([p_1 P_{\pi_x}^t]_s-[p_1 P_{\pi_x}^\infty]_s)\pi_x(a|s)\E_{s'\sim\mathcal{P}}[\hat{r}(s,a,s')] \nonumber\\
    &= \E_{s_1\sim\mathcal{P}}[\hat{r}(s_0,a_0,s_1)]-\bar{r}_x+ \sum_{(s,a)}[p_1 Y_x]_s\pi_x(a|s)\E_{s'\sim\mathcal{P}}[\hat{r}(s,a,s')],
\end{align}
where $Y_x \coloneqq \sum_{t=0}^\infty (P_{\pi_x}^t-P_{\pi_x}^\infty)$.
Then $\|G_x(s,a)\|$ can be bounded as follows
\begin{align}
    \|G_x (s,a)\| 
    &\leq 2 C_r+C_r\sum_{s,a} |[p_1 Y_x]_s| \pi_x(a|s) \nonumber\\
    &\leq  2 C_r+C_r\sum_{s} |[p_1 Y_x]_s| \nonumber\\
    &\leq 2 C_r + \frac{C_r\rho\kappa}{1-\rho}\coloneqq C_G,
\end{align}
where the last inequality follows from condition \ref{item:ergodic} and
\begin{align}\label{eq:p1Yx bound}
    \sum_{s}|[p_1 Y_x]_s|
    &\leq \sum_{t=1}^\infty \sum_{s}|\mathbf{Pr}_{\pi_x}(s_t=s|s_0,a_0)-\mu_{\pi_x}(s)|\nonumber\\ &=\sum_{t=1}^\infty \|\mathbb{P}_{\pi_x}(s_t\in\cdot|s_0,a_0)-\mu_{\pi_x}(\cdot)\|_{TV} \leq \frac{\rho\kappa}{1-\rho}.
\end{align}
Then we have
\begin{align}\label{eq:gx-gx' 0}
    &\|G_x(s,a) - G_{x'}(s,a)\| \nonumber\\
    &\leq \|\bar{r}_x - \bar{r}_{x'}\| +  C_r\sum_{s,a} \big|[p_1 Y_x]_s \big| \big\|\pi_x(a|s)-\pi_{x'}(a|s)\big\| + C_r\sum_{s,a}\big| [p_1 (Y_x - Y_{x'})]_s \big|\pi_{x'}(a|s) \nonumber\\
    &\leq \|\bar{r}_x - \bar{r}_{x'}\|+\sum_{s}\big|[p_1 Y_x]_s\big| L_\pi |\mathcal{A}| \|x-x'\| +  \|p_1 (Y_x - Y_{x'})\|_1 \nonumber\\
    &\leq\|\bar{r}_x - \bar{r}_{x'}\| +\frac{\rho\kappa L_\pi|\mathcal{A}|}{1-\rho}\|x-x'\| + \|Y_x - Y_{x'}\|_\infty.
\end{align}
where the last inequality follows from \eqref{eq:p1Yx bound}. The first term in \eqref{eq:gx-gx' 0} can be bounded as
\begin{align}\label{eq:rx-rx'}
    \|\bar{r}_x - \bar{r}_{x'}\| \leq  \|\mu_x \cdot \pi_x-\mu_{x'}\cdot\pi_{x'}\|_{TV}\sup_{s,a,s'}\|r(s,a,s')\| \leq C_r L_\mu',
\end{align}
where the last inequality follows from Lemma \ref{lemma:mu-mu'}.
By \citep[Theorem 2.5]{moshe1984perturbation}, we have $Y_x+P_{\pi_x}^\infty=(I-P_{\pi_x}+P_{\pi_x}^\infty)^{-1}$. First note that
\begin{align}\label{eq:I-P+P}
    \|(I-P_{\pi_x}+P_{\pi_x}^\infty)^{-1}\|_{\infty}
    &\leq \|Y_x\|_\infty+\|P_{\pi_x}^\infty\|_\infty \nonumber\\
    &\leq \sum_{t=0}^\infty\|P_{\pi_x}^t-P_{\pi_x}^\infty\|_\infty + 1 \nonumber\\
    &= \sum_{t=0}^\infty \max_{s_0\in\mathcal{S}}\sum_{s}|\mathbf{Pr}_{\pi_x}(s_t=s|s_0)-\mu_{\pi_x}(s)| + 1 \nonumber\\
    &\leq \frac{\kappa}{1-\rho}+1,
\end{align}
where the last inequality follows from condition \ref{item:ergodic}.
We also have
\begin{align}\label{eq:I-P+P'}
    &\|(I-P_{\pi_x}+P_{\pi_x}^\infty)^{-1} - (I-P_{\pi_{x'}}+P_{\pi_{x'}}^\infty)^{-1}\|_\infty \nonumber\\
    &\leq \|(I-P_{\pi_x}+P_{\pi_x}^\infty)^{-1}\|_\infty\|P_{\pi_x}-P_{\pi_{x'}}+P_{\pi_{x'}}^\infty-P_{\pi_x}^\infty\|_\infty\|(I-P_{\pi_{x'}}+P_{\pi_{x'}}^\infty)^{-1}\|_\infty \nonumber\\
    &\stackrel{\eqref{eq:I-P+P}}{\leq} \big(\frac{\kappa}{1-\rho}+1\big)^2\big(\|P_{\pi_x}-P_{\pi_{x'}}\|_\infty+\|P_{\pi_{x'}}^\infty-P_{\pi_x}^\infty\|_\infty\big) \nonumber\\
    &\leq \big(\frac{\kappa}{1-\rho}+1\big)^2(L_\pi |\mathcal{A}| + L_\mu )\|x-x'\|.
\end{align}
where in the last inequality we have used
\begin{align}
    \|P_{\pi_x}-P_{\pi_{x'}}\|_\infty &= \max_{s}\sum_{s'}|\sum_{a}\pi_x(a|s)\mathcal{P}(s'|s,a)-\sum_{a}\pi_{x'}(a|s)\mathcal{P}(s'|s,a)| \nonumber\\
    &= \max_{s}|\sum_{a}\pi_x(a|s)-\sum_{a}\pi_{x'}(a|s)|\sum_{s'}\mathcal{P}(s'|s,a) \nonumber\\
    &\quad\leq  \max_{s}\sum_{a}|\pi_x(a|s)-\pi_{x'}(a|s)| \leq L_\pi |\mathcal{A}|\|x-x'\|, \nonumber\\
    \label{eq:pinf-pinf'}
    \|P_{\pi_{x'}}^\infty-P_{\pi_{x'}}^\infty\|_\infty &= \|\mu_{\pi_x}-\mu_{\pi_{x'}}\|_{TV} \leq L_\mu \|x-x'\|~~\text{(Lemma \ref{lemma:mu-mu'})}.
\end{align}
With \eqref{eq:I-P+P'} and \eqref{eq:pinf-pinf'}, we can write
\begin{align}\label{eq:Yx-Yx' bound}
    \|Y_x - Y_{x'}\|_\infty 
    &\leq \|P_{\pi_x}^\infty-P_{\pi_{x'}}^\infty\|_\infty+\|(I-P_{\pi_x}+P_{\pi_x}^\infty)^{-1} - (I-P_{\pi_{x'}}+P_{\pi_{x'}}^\infty)^{-1}\|_\infty \nonumber\\
    &\leq \big(\big(\frac{\kappa}{1-\rho}+1\big)^2(L_\pi |\mathcal{A}| + L_\mu ) + L_\mu\big)\|x-x'\|.
\end{align}
Substituting \eqref{eq:rx-rx'} and \eqref{eq:Yx-Yx' bound} into \eqref{eq:gx-gx' 0} gives
\begin{align}
    \|G_x(s,a) \!-\! G_{x'}(s,a)\| \!\leq\! \Big(C_rL_\mu' \!+ \!\frac{\rho\kappa L_\pi|\mathcal{A}|}{1-\rho}\!+\!\big(\frac{\kappa}{1-\rho}\!+\!1\big)^2(L_\pi |\mathcal{A}| \!+\! L_\mu ) \!+\! L_\mu\Big)\|x-x'\|.\!
\end{align}
This completes the proof.
\end{proof}

\section{Proof of Lemma \ref{lemma:sc} and Corollary \ref{corollary:sc}}\label{section:sco appendix}
Here we prove Lemma \ref{lemma:sc} which along with the generic Theorem \ref{theorem:strongly monotone}~and~\ref{theorem:non-strongly-monotone-2} implies Corollary \ref{corollary:sc}.
\begin{proof}
 We will verify the assumptions by order.

\noindent\textbf{(1) Condition \ref{item:f sc} $\xrightarrow[]{}$ Assumption \ref{assumption:y*}.} Note that $y^{n,*}(y^{n-1}) = f^{n-1}(y^{n-1})$, then \ref{item:f sc} directly implies Assumption \ref{assumption:y*} holds.

\noindent\textbf{(2) Condition \ref{item:f sc} $\xrightarrow[]{}$ Assumption \ref{assumption:vglip}.} First note
\begin{align}
    v(x) 
    &= v\big(x,y^{1,*}(x), y^{2,*}(y^{1,*}(x)),\dots,y^{N,*}(\dots y^{2,*}(y^{1,*}(x))\dots)\big) \nonumber\\
    &=v\big(x,f^0(x), f^1(f^0(x)),\dots,f^{N-1}(\dots f^1(f^0(x))\dots)\big) \nonumber\\
    &= \nabla f^0 (x) \nabla f^1 (f^0(x))\cdots \nabla f^N(f^{N-1}(\dots f^1(f^0(x))\dots)).
\end{align}
By Lemma \ref{lemma:product lip}, in order for $v(x)$ to be Lipschitz continuous, it suffices to let $\nabla f^n(x)$ be bounded and Lipschitz continuous for every $n=0,1,\ldots,N$. This is satisfied under condition \ref{item:f sc}. 

Now in order for $v(x,y^1,y^2,\ldots,y^N)$ be Lipschitz continuous w.r.t. $y^1,y^2,\ldots,y^N$, it again suffices to let $\nabla f^n(x)$ be bounded and Lipschitz continuous for every $n=0,1,\ldots,N$, which is satisfied under condition \ref{item:f sc}. 

Finally, the Lipschitz continuity of $h^n(y^{n-1},y^n)$ w.r.t. $y^n$ is directly implied by condition \ref{item:f sc}.

\noindent\textbf{(3) Condition \ref{item:noise sc} ~and~ \ref{item:ind}$\xrightarrow[]{}$ Assumption \ref{assumption:noise}.} First we have
\begin{align}
    \E[\xi_k | \mathcal{F}_k^1] 
    &= -v(x_k,y_k^1,\dots,y_k^N) \!+\! \E[\nabla f^0 (x_k;\hat{\zeta}_k^0) \cdots \nabla f^N(y_k^N;\zeta_k^N)| \mathcal{F}_k^1]  \nonumber\\
    & = -v(x_k,y_k^1,\dots,y_k^N) \!+\! \E[\nabla f^0 (x_k;\hat{\zeta}_k^0) \cdots \nabla f^N(y_k^N;\zeta_k^N)| \mathcal{F}_k]\nonumber\\
    &= -v(x_k,y_k^1,\dots,y_k^N) \!+\! \nabla f^0 (x_k) \cdots \nabla f^N(y_k^N) = 0,
\end{align}
where we have used the condition that $\hat{\zeta}_k^0,\zeta_k^0,\zeta_k^1,\dots,\zeta_k^N$ are conditionally independent of each other given $\mathcal{F}_k$. The same goes for $\psi_k^n$ that
\begin{align}
    \E[\psi_k^n | \mathcal{F}_k^{n+1}] 
    &= -h^n(y_k^{n-1},y_k^n) + \E[f^{n-1}(y_k^{n-1},\zeta_k^{n-1})| \mathcal{F}_k] - y_k^n = 0.
\end{align}
The bounded variance condition directly implies that $\E[\|\psi_k^n\|^2 | \mathcal{F}_k^{n+1}] < \infty$. Now for $\xi_k$ we have
\begin{align}
   &\E[ \|\xi_k\|^2 | \mathcal{F}_k^1] \nonumber\\
    &= \E[\|\xi_k\|^2 | \mathcal{F}_k] \nonumber\\
    &= \E_k\|\nabla f^0 (x) \nabla f^1 (y^1)\cdots \nabla f^N(y^N)\!-\! \nabla f^0 (x_k;\hat{\zeta}_k^0) \cdots \nabla f^N(y_k^N;\zeta_k^N)\|^2  \nonumber\\
    &=  \E_k\|\nabla f^0 (x_k;\hat{\zeta}_k^0)\|^2 ... \E_k\|\nabla f^N(y_k^N;\zeta_k^N)\|^2 \!-\! \|\nabla f^0 (x)\|^2\|\nabla f^1 (y^1)\|^2...\|\nabla f^N(y^N)\|^2 ,
\end{align}
which is bounded by a constant since under contion \ref{item:noise sc}, we have $ \E_k\|\nabla f^n (x_k;\zeta_k^n)\|^2 < \infty$ for any $n$.

\noindent\textbf{(4) Condition \ref{item:v monotone sc} $\xrightarrow[]{}$ Assumption \textit{\ref{assumption:strongly v}}, \ref{item:F sc} $\xrightarrow[]{}$ Assumption \ref{assumption:integral}.} These assumptions are directly implied by the conditions.

\noindent\textbf{(5) Verifying Assumption \ref{assumption:strongly g}.} By plugging in $y^{n,*}(y^{n-1}) = f^{n-1}(y^{n-1})$, it is immediate that Assumption \ref{assumption:strongly g} holds with $\lambda_n=1$ for any $n\in[N]$.
\end{proof}

\section{Proof of Lemma \ref{lemma:mampg} and Theorem \ref{theorem:mampg}}\label{section:mampg appendix}
Here we prove Lemma \ref{lemma:mampg} which along with the generic Theorem \ref{theorem:non-strongly-monotone-2} implies Theorem \ref{theorem:mampg}.
\begin{proof}
 We now check the assumptions by order.

\noindent\textbf{(1) Condition \ref{item:policy}$\xrightarrow[]{}$ Assumption \ref{assumption:y*}.} First we have $y^{n,*}(y^{n-1})=f^{n-1}(y^{n-1})$. 
In order for the concatenation $f^{n-1}(y^{n-1})$ to be Lipschitz continuous and smooth, we only need each block $f_i^{n-1}(y_i^{n-1})$ to be Lipschitz continuous and smooth. 
Recall that $f_i^{n-1}(y_i^{n-1})=y_i^{n-1} + \eta \nabla F_i (y_i^{n-1})$. The Lipschitz continuity of $f_i^{n-1}(y_i^{n-1})$ is guaranteed by the Lipschitz smoothness of $F_i (\cdot)$, which is well established in the literature \citep{zhang2019global}. Thus we only need to check the Lipschitz smoothness of $f_i^{n-1}(y_i^{n-1})$, that is, the Lipschitz continuity of $\nabla^2 F_i (y_i^{n-1})$. By \citep{fallah2021convergence}, the policy hessian is given by
\begin{align}
    \nabla^2 F (x) = \E_{\zeta \sim p(\cdot|x)}\Big[ \underbrace{g(x;\zeta)\sum_{t=0}^H \nabla \log\pi_x (a_t |s_t)^\top + \nabla g(x;\zeta)}_{H(x;\zeta)}\Big],
\end{align}
where $\zeta = (s_0,a_0,...,s_H,a_H)$ and $p(\zeta|x)=\rho(s_0)\pi_x (a_0|s_0)\Pi_{t=0}^{H-1} \mathcal{P}(s_{t+1}|a_t,s_t) \pi_x(a_{t+1}|s_{t+1})$; $g(x;\zeta) \coloneqq \sum_{h=0}^H \nabla\log\pi_x(a_h|s_h)\sum_{t=h}^H \gamma^t r(s_t,a_t)$, and we omit $i$ since the result holds for general case. 

For any $x,x'\in\mathbb{R}^{d_0}$, we have
\begin{align}\label{eq:phessian lip}
    &\|\nabla^2 F (x) - \nabla^2 F (x')\| \nonumber\\
    &\leq \big\|\E_{\zeta \sim p(\cdot|x)}[H(x;\zeta)] - \E_{\zeta \sim p(\cdot|x')}[H(x;\zeta)] \big\|+ \big\|\E_{\zeta \sim p(\cdot|x')}[H(x;\zeta)] - \E_{\zeta \sim p(\cdot|x')}[H(x';\zeta)]\big\| \nonumber\\
    &\leq \big\|\E_{\zeta \sim p(\cdot|x)}[H(x;\zeta)] - \E_{\zeta \sim p(\cdot|x')}[H(x;\zeta)] \big\|
    + \E_{\zeta \sim p(\cdot|x')}\big\|H(x;\zeta) - H(x';\zeta)\big\|.
\end{align}
We consider the second term first. By Lemma \ref{lemma:product lip}, in order for $H(x;\zeta)$ to be Lipschitz continuous w.r.t. $x$, it suffices to prove: i) $\nabla \log \pi_x (a|s)$ is bounded and Lipschitz continuous; ii) $g(x;\zeta)$ is bounded and Lipschitz continuous; iii) $\nabla g(x;\zeta)$ is Lipschitz continuous. First, i) is directly implied by condition \ref{item:policy}. We then prove ii) as follows
\begin{align}
    \|g(x;\zeta)\| &\leq \sum_{h=0}^H \big\|\nabla\log\pi_x(a_h|s_h) \big\|\big|\sum_{t=h}^H \gamma^t r(s_t,a_t) \big| \leq \frac{C_\pi}{(1-\gamma)^2} \\
     \|g(x;\zeta)-g(x';\zeta)\| &= \sum_{h=0}^H \big\|\nabla\log\pi_x(a_h|s_h)- \nabla\log\pi_{x'}(a_h|s_h)\big\|\sum_{t=h}^H |\gamma^t r(s_t,a_t)| \nonumber\\
     &\leq \frac{L_\pi'}{(1-\gamma)^2}\|x-x'\|.
\end{align}
Next we prove iii) as follows
\begin{align}
    \|\nabla g(x;\zeta) - \nabla g(x';\zeta)\|
    &\leq \sum_{h=0}^H \big\|\nabla^2\log\pi_x(a_h|s_h)- \nabla^2\log\pi_{x'}(a_h|s_h)\big\|\sum_{t=h}^H |\gamma^t r(s_t,a_t)| \nonumber\\
     &\leq \frac{L_\pi''}{(1-\gamma)^2}\|x-x'\|.
\end{align}
By Lemma \ref{lemma:product lip}, we know i), ii) and iii) imply the Lipschitz continuity of $H(x;\zeta)$, i.e. it holds that
\begin{align}\label{eq:phessian lip 2}
    \|H(x;\zeta)-H(x';\zeta)\| \leq \frac{L_\pi''+2 H C_\pi L_\pi'}{(1-\gamma)^2}\|x-x'\|.
\end{align}

The first term in \eqref{eq:phessian lip} can be bounded as
\begin{align}\label{eq:TV}
    \big\|\E_{\zeta \sim p(\cdot|x)}[H(x;\zeta)] - \E_{\zeta \sim p(\cdot|x')}[H(x;\zeta)]\big\|
    &\leq \sup_{\zeta} \big\|H(x;\zeta)\big\| \sum_{\zeta}|p(\zeta|x)-p(\zeta|x')| \nonumber\\
    &\stackrel{\eqref{eq:H upper}}{\leq}\frac{H C_\pi^2 + L_\pi'}{(1-\gamma)^2} \sum_{\zeta}|p(\zeta|x)-p(\zeta|x')| \nonumber\\
    &\leq \frac{H C_\pi^2 + L_\pi'}{(1-\gamma)^2} (H+1) |\mathcal{A}| L_\pi \|x-x'\|.
\end{align}
where the second inequality follows from
\begin{align}\label{eq:H upper}
    \big\|H(x;\zeta)\big\| \leq \|g(x;\zeta)\|\sum_{t=0}^H \|\nabla \log\pi_x (a_t |s_t)\| + \|\nabla g(x;\zeta)\| \leq \frac{H C_\pi^2}{(1-\gamma)^2} + \frac{L_\pi'}{(1-\gamma)^2}.
\end{align}
Substituting \eqref{eq:phessian lip 2} and \eqref{eq:TV} into \eqref{eq:phessian lip} yields
\begin{align}
    \|\nabla^2 F (x) - \nabla^2 F (x')\| \leq \Big(\frac{L_\pi''+2 H C_\pi L_\pi'}{(1-\gamma)^2} + \frac{H C_\pi^2 + L_\pi'}{(1-\gamma)^2} (H+1) |\mathcal{A}| L_\pi \Big)\|x-x'\|.
\end{align}
This implies that $f_i^{n-1}(\cdot)$ is $\eta\big(\frac{L_\pi''+2 H C_\pi L_\pi'}{(1-\gamma)^2} + \frac{H C_\pi^2 + L_\pi'}{(1-\gamma)^2} (H+1) |\mathcal{A}| L_\pi \big)$-Lipschitz smooth for $n\in[N]$.

\noindent\textbf{(2) Conditions \ref{item:noise mampg}~and~\ref{item:traj}$\xrightarrow[]{}$ Assumption \ref{assumption:noise}}. It is clear that conditions \ref{item:noise mampg}~and~\ref{item:traj} imply condition \ref{item:noise sc}~and~\ref{item:ind}. Thus by Lemma \ref{lemma:sc}, Assumption \ref{assumption:noise} is satisfied. 

We will specify the estimators that satisfy condition \ref{item:noise mampg} as follow. 
To estimate $f_i^n(y)$ ($n=0,1,...,N-1$), one can use:
\begin{align}
  f_i^n(y;\zeta_i^n) \coloneqq y + \eta  \sum_{h=0}^H \nabla \log \pi_y (a_h|s_h)\sum_{t=h}^H \gamma^t r_i(s_t,a_t),~ n=0,1,...,N-1,
\end{align}
where $\zeta_i^n = (s_0,a_0,...,s_H,a_H)$ is generated under policy $\pi_y$, transition distribution $\mathcal{P}_i$ and initial distribution $\rho_i$. The estimator satisfies condition \ref{item:noise mampg}:
\begin{align}\label{eq:fin estimate}
    \E_{\zeta_i^n}[f_i^n(y;\zeta_i^n)] &= y + \eta  \nabla F_i(y) = f_i^n(y), \nonumber\\
    \E_{\zeta_i^n}[\|f_i^n(y;\zeta_i^n)-f_i^n(x)\|^2] &\leq \E_{\zeta_i^n}\|\sum_{h=0}^H \nabla \log \pi_y (a_h|s_h)\sum_{t=h}^H \gamma^t r_i(s_t,a_t)\|^2 \leq \frac{C_\pi^2}{(1-\gamma)^4}.
\end{align}
To estimate $\nabla f_i^n(y)$ ($n=0,1,...,N-1$), one can use:
\begin{align}
    \nabla f_i^n(y;\zeta_i^n) \coloneqq I + \eta H(y;\zeta_i^n),~ n=0,1,...,N-1,
\end{align}
where $\zeta_i^n = (s_0,a_0,...,s_H,a_H)$ is generated under policy $\pi_y$, transition distribution $\mathcal{P}_i$ and initial distribution $\rho_i$. The estimator satisfies condition \ref{item:noise mampg}:
\begin{align}
     \E_{\zeta_i^n}[\nabla f_i^n(y;\zeta_i^n)] &= I + \eta\nabla^2 F_i(y) = \nabla f_i^n(y), \nonumber\\
     \E_{\zeta_i^n}\|\nabla f_i^n(y;\zeta_i^n) -  \nabla f_i^n(y)\|^2 &\leq \E_{\zeta_i^n}\|\nabla f_i^n(y;\zeta_i^n)\|^2 \stackrel{\eqref{eq:H upper}}{\leq} 2 + 2 \eta^2 \frac{(HC_\pi^2 + L_\pi')^2}{(1-\gamma)^4}.
\end{align}
To estimate $\nabla f_i^N(x)$, one can use
\begin{align}
    \nabla f_i^N(x;\zeta_i^N) \coloneqq \sum_{h=0}^H \nabla \log \pi_x (a_h|s_h)\sum_{t=h}^H \gamma^t r_i(s_t,a_t),
\end{align}
where $\zeta_i^n = (s_0,a_0,...,s_H,a_H)$ is generated under policy $\pi_y$, transition kernel $\mathcal{P}_i$ and initial distribution $\rho_i$. This estimator satisfies Conditions \ref{item:noise mampg}, following the similar lines in \eqref{eq:fin estimate}.

\noindent\textbf{(3) Verifying Assumption \ref{assumption:strongly g} and \ref{assumption:integral}}. Assumption \ref{assumption:strongly g} is satisfied with $\lambda_n=1$ by directly plugging in $y^{n,*}(y^{n-1})=f^{n-1}(y^{n-1})$. Assumption \ref{assumption:integral} is satisfied by observing that
\begin{align}
    F(x) =  \frac{1}{M}\sum_{i=1}^M F_i (\Tilde{x}_i^N(x)) \leq \frac{1}{1-\gamma},
\end{align}
where we have used the fact that $F_i(x) \leq \frac{1}{1-\gamma}$ for any $x$.
\end{proof}

\section{Technical lemmas}
\begin{Lemma}\label{lemma:yk+1n-yn*dots}
Suppose Assumption \ref{assumption:y*} ~and~ \ref{assumption:vglip} hold. Recall we defined $L_y (n) =  \sum_{i=n}^{N} L_{y,i-1} L_{y,i-2}\dots L_{y,n} $ with $L_{y,n-1} L_{y,n-2}\dots L_{y,n} = 1$ for any $n \in [N]$. Then it holds that
\begin{align}
    \big\|v(x_k,y_k^{1:N}) \!-\! v(x_k)\big\| \!\leq\! L_{v,y} \sum_{n=1}^N L_y(n) \|y_k^n \!-\! y^{n,*}( y_k^{n\!-\!1})\|.
\end{align}
\end{Lemma}
 \begin{proof}
 By the Lipschitz continuity of $v(x,y^1,\dots,y^N)$ w.r.t. $y^1,\dots,y^N$, we have
 \begin{align}\label{eq:v-vk}
      \|v(x_k,y_k^{1:N}) - v(x_k)\|
      & \leq L_{v,y} \sum_{n=1}^N \|y_k^n - y^{n,*}(\dots y^{2,*}(y^{1,*}(x_k))\dots) \|.
 \end{align}
 For any $n \geq 2$, we have
 \begin{align}
      &\|y_k^n - y^{n,*}(\dots y^{2,*}(y^{1,*}(x_k))\dots) \| \nonumber\\
      &\leq \|y_k^n - y^{n,*}( y_k^{n-1})\| + \| y^{n,*}( y_k^{n-1})- y^{n,*}(\dots y^{2,*}(y^{1,*}(x_k))\dots) \| \nonumber\\
      &\leq\|y_k^n - y^{n,*}( y_k^{n-1})\| + L_{y,n-1}\|y_k^{n-1} -y^{n-1,*}(\dots y^{2,*}(y^{1,*}(x_k))\dots) \|.
 \end{align}
 Keep unraveling yields
 \begin{align}\label{eq:yk+1n-yn*dots}
     \|y_k^n - y^{n,*}(\dots y^{2,*}(y^{1,*}(x_k))\dots) \| 
     &\leq \sum_{j=1}^{n}  L_{y,n-1} L_{y,n-2}\dots L_{y,j}\|y_k^j - y^{j,*}( y_k^{j-1})\|,
 \end{align}
where $ L_{y,n-1} L_{y,n-2}\dots L_{y,n} \coloneqq 1$.
Substituting \eqref{eq:yk+1n-yn*dots} into the first inequality completes the proof.
 \end{proof}
 
 \begin{Lemma}\label{lemma:step sizes}
With any positive $\lambda_1$ and non-negative constants $\lambda_0$, $\lambda_2<\lambda_1$ and $C_1,...,C_4$, consider the following inequality about the step size $\beta_{k,n-1}$:
\begin{align}\label{eq:step size 0}
    (1+C_1 \beta_{k,n-1} + C_2 \beta_{k,n-1}^2)(1-\lambda_1 \beta_{k,n}) -1 + \lambda_2 \beta_{k,n} + C_3\alpha_k + C_4\alpha_k^2 \leq -\lambda_0 \alpha_k.
\end{align}
Suppose all step sizes are in the same time-scale. Then given any $\beta_{k,n}$, if $\alpha_k\leq\beta_{k,n-1}\leq 1$, the above inequality always admits solutions for $\beta_{k,n-1}$.
\end{Lemma}
\begin{proof}
First we have
\begin{align}
    C_2 \beta_{k,n-1}^2 \leq C_2\beta_{k,n-1},~~C_4 \alpha_k^2 \leq C_4\alpha_k.
\end{align}
With the above inequality, we can simplify \eqref{eq:step size 0} to
  \begin{equation}\label{eq:step size 1}
    (1+(C_1+C_2)\beta_{k,n-1})(1-\lambda_1 \beta_{k,n}) + \lambda_2\beta_{k,n} \leq 1-(\lambda_0+C_3+C_4)\alpha_k.
\end{equation}
By $\lambda_2\beta_{k,n} \!\leq\! (1\!+\!(C_1\!+\!C_2)\beta_{k,n\!-\!1})\lambda_2\beta_{k,n}$, the sufficient condition of \eqref{eq:step size 0} is
\begin{equation}\label{eq:step size 2}
(1+(C_1+C_2)\beta_{k,n-1})(1-\lambda' \beta_{k,n}) \!\leq\! 1-(\lambda_0+C_3+C_4)\alpha_k.
\end{equation}
where $\lambda'=\lambda_1-\lambda_2 > 0$.
Next we show that \eqref{eq:step size 2} holds. 
With $\alpha_k\leq\beta_{k,n-1}$, rearranging and simplifying \eqref{eq:step size 2} gives
\begin{equation}
    \beta_{k,n-1} \leq \lambda'\frac{\beta_{k,n}}{\lambda_0+C_1+C_2+C_3+C_4},
\end{equation}
which can be satisfied if $\beta_{k,n-1},\beta_{k,n}$ are in the same scale, and $\beta_{1,n-1}$ is small relative to $\beta_{1,n}$.
\end{proof}

\begin{Lemma}[Robbins-Siegmund {\citep[Theorem 2.3.5]{gadet2017stochastic}}]\label{lemma:robbins-siegmund}
Consider a sequence of $\sigma$-algebras $\{\mathcal{F}_k\}_{k\geq1}$ and four integrable non-negative sequences $\{U_k\}, \{V_k\}, \{\tau_k\}, \{\delta_k\}$ that satisfy
\begin{enumerate}
    \item [i)] $U_k, V_k, \tau_k, \delta_k$ are $\mathcal{F}_k$-measurable.
    \item [ii)] $\Pi_{k\geq 1} (1+\tau_k) < \infty$ and $\sum_{k \geq 1}\E[\beta_k]<\infty$.
    \item [iii)] For $k\geq1$, $\E[V_{k+1}|\mathcal{F}_k] \leq V_k(1+\tau_k) + \delta_k - U_{k+1}$.
\end{enumerate}
Then it holds that
\begin{enumerate}
    \item [1)] $V_k \xrightarrow[]{k\xrightarrow[]{}\infty}V_{\infty}<\infty$ and $\sup_{k\geq 1}\E[V_k]< \infty$.
    \item [2)] $\sum_{k\geq 1}\E[U_k]<\infty$ and $\sum_{k\geq 1}U_k < \infty$ a.s.
\end{enumerate}
\end{Lemma}
 
 \begin{Lemma}\label{lemma:error metic equiv}
 Suppose Assumption \ref{assumption:y*} holds. Then there exists a positive constant $C_N$ such that
 \begin{align}
 \|x_k-x^*\|^2 + \sum_{n=1}^N \|y_k^n-y^{n,*}\|^2 \leq C_N\big(\|x_k-x^*\|^2 + \sum_{n=1}^N \|y_k^n-y^{n,*}(y_k^{n-1})\|^2\big).
 \end{align}
\end{Lemma}
\begin{proof}
First note that under Assumption \ref{assumption:y*}, we have
\begin{align}\label{eqn.195}
  \sum_{n=1}^N \|y_k^n-y^{n,*}\|
  &= \sum_{n=1}^N \|y_k^n-y^{n,*}(\dots y^{2,*}(y^{1,*}(x^*))\|.
\end{align}
To bound the RHS of \eqref{eqn.195}, we can follow the derivation of \eqref{eq:v-vk}--\eqref{eq:yk+1n-yn*dots} with $x_k=x^*$ and obtain
 \begin{align}
 \sum_{n=1}^N \|y_k^n-y^{n,*}\|
  &=\sum_{n=1}^N \|y_k^n-y^{n,*}(\dots y^{2,*}(y^{1,*}(x^*))\| \nonumber\\
  & \leq L_y(1) \|y_k^1 - y^{1,*}(x^*)\| + \sum_{n=2}^N L_y(n) \|y_k^n \!-\! y^{n,*}( y_k^{n\!-\!1})\|,
 \end{align}
where $\{L_y(n)\}_{n=1}^N$ is a series of constants specified in Lemma \ref{lemma:yk+1n-yn*dots}. From the last inequality, we have
\begin{align}\label{eq:ykn-ykn* equiv}
    \sum_{n=1}^N \|y_k^n-y^{n,*}\| 
     &\leq L_y(1)\|y^{1,*}(x_k) \!-\! y^{1,*}(x^*)\| \!+\! L_y(1) \|y_k^1 - y^{1,*}(x_k)\| \!+\! \sum_{n=2}^N L_y(n) \|y_k^n \!-\! y^{n,*}( y_k^{n\!-\!1})\| \nonumber\\
     &\leq L_y(1) L_{y,1}\|x_k-x^*\| \!+\! \sum_{n=1}^N L_y(n) \|y_k^n \!-\! y^{n,*}( y_k^{n\!-\!1})\|.
\end{align}
Then we have
\begin{align}
   \|x_k-x^*\|^2 + \sum_{n=1}^N \|y_k^n-y^{n,*}\|^2 
    &\leq \Big(\|x_k-x^*\| + \sum_{n=1}^N \|y_k^n-y^{n,*}\| \Big)^2 \\
    &\stackrel{\eqref{eq:ykn-ykn* equiv}}{\leq}2(1+L_y(1) L_{y,1})^2\|x_k-x^*\|^2 \!+\! 2N\sum_{n=1}^N L_y^2(n) \|y_k^n \!-\! y^{n,*}( y_k^{n\!-\!1})\|^2.\nonumber
\end{align}
Choosing $C_N=2\max\{(1+L_y(1) L_{y,1})^2,N L^2_y(1),...,N L^2_y(n)\}$ completes the proof.
\end{proof}

\begin{Lemma}[Lipschitz continuity of a product.]\label{lemma:product lip}
Define $f_i:\mathbb{R}^{d}\mapsto\mathbb{R}^{d_i \times d_{i+1}}$. If there exist positive constants $L_1,L_2,...,L_n$ and $C_1,C_2,...,C_n$ such that for any $x,x' \in \mathbb{R}^d$ it holds that
\begin{align}
    \|f_i(x)-f_i(x')\| \leq L_i \|x-x'\|, ~\|f_i(x)\|\leq C_i,~\forall i\in[n].
\end{align}
Then it holds that
\begin{align}
    \|f_1(x)f_2(x)...f_n(x)-f_1(x')f_2(x')...f_n(x')\| \leq \sum_{j=1}^n C_1 C_2 ... L_j ... C_n \|x-x'\|.
\end{align}
\end{Lemma}
\begin{proof}
 We can decompose the product as
\begin{align}
    &\|f_1(x)f_2(x)...f_n(x)-f_1(x')f_2(x')...f_n(x')\|  \nonumber\\
    &= \|f_1(x)f_2(x)...f_n(x)-f_1(x')f_2(x)...f_n(x) + f_1(x')f_2(x)...f_n(x)-f_1(x')f_2(x')...f_n(x) \nonumber\\
    &\quad+\cdots+f_1(x')f_2(x')...f_n(x)-f_1(x')f_2(x')...f_n(x')\| \nonumber\\
    &\leq C_2...C_n\|f_1(x)-f_1(x')\| \!+\! C_1C_3..C_n\|f_2(x)-f_2(x')\| \!+\! \cdots \!+\! C_1C_2..C_{n-1}\|f_n(x)-f_n(x')\| \nonumber\\
    &\leq \sum_{j=1}^n C_1 C_2 ... L_j ... C_n \|x-x'\|.
\end{align}
\vspace{-0.2cm}
This completes the proof.
\end{proof}
\vspace{-1cm}
\end{document}